\documentclass[a4paper,11pt]{article}

\usepackage{xr}
\makeatletter

\newcommand*{\addFileDependency}[1]{
\typeout{(#1)}
\@addtofilelist{#1}
\IfFileExists{#1}{}{\typeout{No file #1.}}
}\makeatother

\usepackage{xr-hyper}
\usepackage{mattsstyle}
\usepackage{tcolorbox}
\usepackage{soul}
\usepackage{bm}
\usepackage{comment}
\usepackage{xfrac}

\usepackage{booktabs,tabularx,ragged2e}
\newcolumntype{Y}{>{\RaggedRight\arraybackslash}X}
\usepackage{adjustbox} 

\def\Lp#1{\mathrm{L}^{#1}}
\def\Hk#1{\mathrm{H}^{#1}}

\def\Hk#1{\mathrm{H}^{#1}}
\def\spaceBar{\, | \,}

\newcommand{\ub}{\mathbf{u}}

\DeclareMathOperator{\nn}{NN}

\def\awmath#1{\mathbf{\textcolor{darkgreen}{#1}}}

\def\commentOut#1{}

\allowdisplaybreaks

\usepackage{pifont}
\title{Generalization Bounds and Statistical Guarantees for Multi-Task and Multiple Operator Learning with MNO Networks}
\author[1]{Adrien Weihs}
\author[1]{Hayden Schaeffer}

\affil[1]{Department of Mathematics,\protect\\ University of California Los Angeles,\protect\\ Los Angeles, CA 90095, USA. \vspace{\baselineskip}}

\date{}

\newcommand{\ScalingNetwork}{\mathrm{MNO}}

\usepackage{tikz}
\usetikzlibrary{positioning, arrows.meta, calc}

\newcommand{\clippedClass}{\mathrm{Cl}}
\newcommand{\coverName}{\mathrm{Cvr}}
\newcommand{\empiricalEvaluation}{\mathrm{Emp}}
\newcommand{\shiftedClippedClass}{\mathrm{SCl}}
\begin{document}

\maketitle

\begin{abstract}
\noindent Multiple operator learning concerns learning operator families
\(\{G[\alpha]:U\to V\}_{\alpha\in W}\) indexed by an operator descriptor \(\alpha\).
Training data are collected hierarchically by sampling operator instances \(\alpha\), then input functions \(u\) per instance, and finally evaluation points \(x\) per input, yielding noisy observations of \(G[\alpha][u](x)\).
While recent work has developed expressive multi-task and multiple operator learning architectures and approximation-theoretic scaling laws, quantitative statistical generalization guarantees remain limited.
We provide a covering-number-based generalization analysis for separable models, focusing on the Multiple Neural Operator (MNO) architecture: we first derive explicit metric-entropy bounds for hypothesis classes given by linear combinations of products of deep ReLU subnetworks, and then combine these complexity bounds with approximation guarantees for MNO to obtain an explicit approximation-estimation tradeoff for the expected test error on new (unseen) triples \((\alpha,u,x)\).
The resulting bound makes the dependence on the hierarchical sampling budgets \((n_\alpha,n_u,n_x)\) transparent and yields an explicit learning-rate statement in the operator-sampling budget \(n_\alpha\), providing a sample-complexity characterization for generalization across operator instances. The structure and architecture can also be viewed as a general purpose solver or an example of a ``small'' PDE foundation model, where the triples are one form of multi-modality. 
\end{abstract}



\section{Introduction}

Operator learning seeks to approximate maps between function spaces, typically of the form
\(u \mapsto G[u]\), where the input \(u\) is a function and the output is another
function (see \cite{KOVACHKI2024419,lu2022comprehensive} and references therein). In many applications, however, the object of interest is not a single operator but a
family of related operators indexed by a parameter. This motivates the recent \emph{multiple operator learning} problem (see \cite{sun2025foundation, sun2025lemonlearninglearnmultioperator,weihs2025MOL} and references therein),
which we formalize as learning a map
\[
G: W \;\longrightarrow\; \{G[\alpha] : U \to V\}_{\alpha \in W},
\]
where \(W,U,V\) are function spaces, \(\alpha\in W\) encodes the operator instance, and for each \(\alpha\),
the corresponding operator \(G[\alpha]\) maps inputs \(u\in U\) to outputs in \(V\).
We highlight three prototypical settings in which multiple operator learning is either intrinsic to the problem formulation or would benefits from this formulation.  We refer to Section \ref{sec:examples} for a detailed description of each example class.  

\begin{enumerate}
    \item \textit{Parameterized integral operators.} A basic example is a family of integral operators
\begin{equation} \label{eq:kernelOperator}
   G[\alpha][u](x) = \int K_\alpha(x,y) u(y)\, \dd y 
\end{equation}
where the kernel \(K_\alpha\) depends on a parametric function \(\alpha\). The problem is inherently a multi-operator learning task, as each parameter $\alpha$ induces a different (related) operator.
\item \textit{Solution operators of parameterized PDEs.} Many simulations or forward problems in the physical sciences are naturally expressed through PDE solution operators: for each parameter
\(\alpha\) encoding, for instance, coefficients, boundary data, geometry, or even the governing equation, the map \(G[\alpha]\) takes an input \(u\) (e.g.\ a forcing term, initial condition, or source) to the corresponding solution \(G[\alpha][u]\).
\item \textit{Operator families indexed by symbolic or textual descriptions.}
More broadly, the ``operator index'' \(\alpha\) can encode a task specification: a symbolic form,
a natural-language prompt, or a discrete task label. This viewpoint connects multiple operator learning to PDE foundation models \cite{liu2024prosefd, sun2025foundation}, where
a shared representation is trained across many operator instances and queried on new tasks by conditioning on an explicit
operator description. In such regimes, the ability to generalize in \(\alpha\) is important since one aims to predict accurate operators that are not seen during training.
\end{enumerate}

Neural networks are a particularly well-suited approximation class for multiple operator learning because they combine high expressive power with architectural flexibility. Deep networks can capture nonlinear dependence on both the operator index \(\alpha\) and the input function \(u\), and they can be
instantiated in forms that encode relevant inductive biases. At the same time, choosing an effective architecture is subtle: the network must allocate capacity between
encoding how the operator varies with \(\alpha\) and representing the action \(u\mapsto G[\alpha][u]\), while also coping with
discretization and the effective dimensionality of \(W\times U\times \Omega_V\). A number of principled architectures have
been proposed recently to address these challenges. Among them, the Multiple Neural Operator (MNO) architecture introduced in~\cite{weihs2025MOL}
achieves strong empirical performance across diverse operator families and is accompanied by expressivity guarantees and
explicit scaling laws, providing both practical effectiveness and theoretical guidance for design.

Beyond approximation capabilities, it is important to understand how such architectures generalize when trained from finite data. Generalization bounds quantify how accurately a learned predictor will perform on unseen inputs drawn from the same data-generating process. They typically decompose into an approximation term, capturing the best achievable error within the chosen hypothesis class $\cF$, and an estimation term, which depends on the statistical complexity of that class together with the available sample size. In particular, when complexity is controlled via covering numbers (we write \(\cN(\eta,\mathcal{F},\|\cdot\|)\) for the \(\eta\)-covering number of a function class \(\mathcal{F}\) with respect to a norm \(\|\cdot\|\)), the resulting estimation term involves the metric entropy \(\log \cN(\eta,\mathcal{F},\|\cdot\|)\) at scale \(\eta\), combined with sample-size factors (e.g.\ \(1/\sqrt{n}\) or \(1/n\)), reflecting how generalization improves as more data are collected.
To the best of our knowledge, this work provides the first generalization bound of this kind for multiple operator learning or multi-task neural operators. 

In our setting, the generalization error concerns how a learned model achieves small expected error on new, unseen triples \((\alpha,u,x)\sim \mu_\alpha\times \mu_u\times \mu_x\), thereby controlling transfer across operator instances, input functions, and evaluation points simultaneously. We obtain such a bound for the MNO architecture by combining an approximation term (controlling the best-in-class approximation error $\eps$ through the results in \cite{weihs2025MOL}) with covering-number complexity control of the induced $\eps$-dependent hypothesis class, yielding a transparent approximation--estimation tradeoff that links target accuracy, architectural complexity, and the sampling budgets \((n_\alpha,n_u,n_x)\) (numbers of sampled operators, inputs per operator, and evaluation points, respectively). In addition, the bound formalizes two practical benefits of the multiple operator learning viewpoint: (i) \emph{amortization across operator instances}, since a single conditional model can share representations across an operator family rather than training one model per \(\alpha\); and (ii) \emph{hierarchical sampling guidance}, since the explicit dependence on \((n_\alpha,n_u,n_x)\) clarifies how accuracy improves when increasing operator variability, inputs per operator, or evaluation resolution.

\subsection{Contributions}

Our main contributions are as follows: \begin{enumerate}
    \item  We derive \textbf{metric entropy bounds} for function classes given by linear combinations of products of three deep ReLU subnetworks, which is precisely the separable structure underlying MNO. In particular, Proposition~\ref{prop:back:coveringClippledMultipleOperator} provides an estimate as a function of the architectural parameters of each subnetwork class: depth \(L_i\), width \(p_i\), sparsity budget \(K_i\), and parameter magnitude \(\kappa_i\), as well as the product-structure multiplicities \(P,H,N\).
    
    \item  Under Lipschitz regularity of the map $G$, we prove an \textbf{explicit scaling law for the expected generalization error} of the MNO architecture; see Theorem~\ref{thm:scalingLawsGeneralizationError}. The bound is stated for test triples \((\alpha,u,x)\sim\mu_\alpha\times\mu_u\times\mu_x\) and makes explicit the dependence on the hierarchical sampling budgets \((n_\alpha,n_u,n_x)\), on a prescribed target accuracy \(\varepsilon>0\) (achieved by an explicit \(\varepsilon\)-dependent choice of the MNO hypothesis class), and on a covering scale \(\eta>0\) through the metric entropy \(\log \cN(\eta)\). Specifically, it takes the schematic form: \begin{align*}
        \mathbb{E}_{\alpha,u,x} \ls \text{test error} \rs \lesssim \eps^2 + \eta + \frac{\eta}{\sqrt{n_\alpha n_u n_x}} \sqrt{\log(\cN(\eta))} + \frac{1}{n_\alpha n_u n_x} \log(\cN(\eta)) + \frac{1}{n_\alpha} \log(\cN(\eta))
    \end{align*}

    \item As a consequence of the scaling law, we derive an \textbf{explicit sample-complexity rate} in the operator-sampling budget \(n_\alpha\); see Corollary~\ref{cor:boundEps}. In particular, by selecting \(\varepsilon=\varepsilon(n_\alpha)\) and \(\eta=\eta(n_\alpha)\) to balance the approximation and estimation terms in Theorem~\ref{thm:scalingLawsGeneralizationError}, we obtain the rate
\[
\bbE[\text{test error}]
=
\mathcal{O}\!\left(
\left(
\frac{\log\log(n_\alpha)}{\log\log\log(n_\alpha)}
\right)^{-2/d_W}
\right),
\]
with constants independent of \(n_\alpha,n_u,n_x\).

\end{enumerate}

This perspective is closely connected to recent PDE foundation models, where a single conditional architecture is trained across a broad family of PDEs and queried via an explicit operator descriptor. Our results provide indirect theoretical support for such approaches by establishing generalization guarantees and sample-complexity bounds for separable deep architectures that can themselves be approximated by PDE foundation models \cite{sun2025foundation, liu2024prosefd}.

\subsection{Informal Statement of the Main Results}

For ease of presentation, we state an informal version of the main results (Theorem~\ref{thm:scalingLawsGeneralizationError} and Corollary \ref{cor:boundEps}), highlighting the dependence of the expected test error on the operator sampling budget and on the complexity of the induced hypothesis class; all assumptions and precise results are specified in the formal statements in Section \ref{sec:main}.

\begin{theorem}[Generalization error for MNO]
    Let $G:W \mapsto \{G[\alpha]:U \mapsto W\}_{\alpha \in W}$ be a Lipschitz multiple operator map from the function space $W$ into Lipschitz operators from $U$ to $V$. Assume that we observe sampled noisy data:
\[
y_{\ell i j}
:= G[\alpha_\ell][u_{\ell i}](x_{\ell i j})+\zeta_{\ell i j},
\qquad
1 \leq \ell \leq n_\alpha,\ 1 \leq i \leq n_u,\ 1 \leq j\leq n_x,
\]
where $\zeta_{\ell i j}$ denotes observation noise. For every $\eps > 0$, there exists a MNO \[
    \sum_{p=1}^{P}\sum_{k=1}^{H}\sum_{\ell=1}^{N}
    \theta_{pk\ell}\, l_p(\alpha)\, b_k(u)\, \tau_\ell(x),
    \] 
    trained on $\{y_{\ell i j}\}$, whose expected test error on unseen triples $(\alpha,u,x)$ satisfies: \[
    \mathcal{O}\left(\eps^2 + \frac{1}{n_\alpha} \eps^{-\delta_1 \eps^{-\delta_2\eps^{-d_W}}} \log(n_\alpha)
\right).
    \]  
Here $\theta_{\ell i j} \in \bbR$, $\delta_i >0$, 
$d_W$ denotes the dimension of the domain of functions in $W$ and the ReLU subnetworks $l_p$, $b_k$, and $\tau_\ell$ can be chosen with the following architectural
scalings:
\newlength{\colA} 
\newlength{\colB} 
\newlength{\colC} 
\newlength{\colD} 
\newlength{\colE} 
\newlength{\colF} 

\setlength{\colA}{1.2cm}
\setlength{\colB}{2.3cm}
\setlength{\colC}{2.3cm}
\setlength{\colD}{2.3cm}
\setlength{\colE}{2.3cm}
\setlength{\colF}{3.2cm}

\begin{table}[H]
\centering
\small
\renewcommand{\arraystretch}{1.5}
\setlength{\tabcolsep}{6pt}
\begin{tabularx}{\linewidth}{
    >{\bfseries}p{\colA}
    p{\colB}
    >{\centering\arraybackslash}p{\colC}
    p{\colD}
    p{\colE}
    p{\colF}
}
& \# networks & width & depth & sparsity & parameter magnitude \\
\midrule
$l_p$
& $P \lesssim \eps^{-\eps^{-d_W}}$
& $\mathcal{O}(1)$
& $\lesssim \eps^{-d_W}$
& $\lesssim \eps^{-d_W}$
& $\lesssim \eps^{-\eps^{-d_W}}$
\\
$b_k$
& $H \lesssim \eps^{-\eps^{-\eps^{-d_W}}}$
& $\mathcal{O}(1)$
& $\lesssim \eps^{-\eps^{-d_W}}$
& $\lesssim \eps^{-\eps^{-d_W}}$
& $\lesssim \eps^{-\eps^{-\eps^{-d_W}}}$
\\
$\tau_\ell$
& $N \lesssim \eps^{-\eps^{-d_W}}$
& $\mathcal{O}(1)$
& $\lesssim \eps^{-d_W}$
& $\lesssim \eps^{-d_W}$
& $\lesssim \eps^{-\eps^{-d_W}}$
\\
\end{tabularx}
\end{table}
Moreover, choosing $\varepsilon \asymp 
    \left(
        \frac{\log\log(n_\alpha)}{\log\log\log(n_\alpha)}
    \right)^{-\frac{1}{d_W}}
$
yields the generalization bound: \[
\mathcal{O}\!\left(
\left(
\frac{\log\log(n_\alpha)}{\log\log\log(n_\alpha)}
\right)^{-2/d_W}
\right).
\]
    
\end{theorem}

\subsection{Related Works and Literature Review}

\paragraph{Neural operator architectures}

Operator learning aim to approximate maps that take functions as inputs and return functions as outputs \cite{pathak2022fourcastnet,zhu2023fourierdeeponet,jiang2023fouriermionet,li2023fnoseismic,moya2023operatorgrid,chen2023neuraloperator,Bhattacharya,li2021fourier}. Neural networks have proven effective for learning such mappings in a range of scientific and engineering settings \cite{deepLearningImages,Graves2013SpeechRW,KHOO_LU_YING_2021,Jentzen,zhangBelnet}. A common design pattern in neural operator models is to split the processing of functional and spatial inputs across interacting subnetworks and recombine them via additive or bilinear/tensor-style contractions \cite{ChenChen1993,ChenChen1995}; DeepONet \cite{deepOnet} is a canonical example, with a branch network that encodes the input function and a trunk network that produces a coordinate-dependent learned output basis. This ``separable expansion'' viewpoint connects operator networks to low-rank approximation ideas, where complex mappings are expressed as sums of simpler, lower-dimensional factors \cite{markovsky2012lowrank}; neural operators can be viewed as nonlinear, data-adaptive analogues of such decompositions. In the multiple-operator learning setting, MNO \cite{weihs2025MOL} adopts an analogous separable structure by separating operator identity (task/parameter) from input-function dependence through interacting subnetworks, enabling shared representations across a family of operators. Many other designs have emerged, including Fourier Neural Operators \cite{li2021fourier} (motivated by spectral representations), Green's function-based approaches such as Deep Green Networks \cite{deepGreen,BoulleGreen}, and graph-based variants (including multipole constructions) that exploit sparsity and multiscale structure to improve efficiency \cite{anandkumar2019neural,multipole}. For additional architectures and broader perspectives, see the surveys and references in \cite{KOVACHKI2024419,Goswami2023}.

\paragraph{Multi-task and multiple operator learning}

Motivations for learning families of operators are twofold. In some settings, the underlying problem is naturally specified as a collection of related operators (e.g., indexed by physical parameters, geometries, or boundary conditions). In others, jointly learning multiple operators is a strategy for improving data efficiency and generalization by sharing structure across tasks. A growing body of recent work proposes multi-operator learning frameworks along these lines \cite{sun2025foundation,liu2024prose,mccabe2023multiple,yang2023incontext,yang2023prompting,cao2024vicon,zhang2024modno,zhang2024d2no,liu2025bcat,ye2025pdeformer,zhang2025probabilistic,Jollie_2025,herde2024poseidon,bacho2025operatorlearningmachineprecision,weihs2025MOL}. Notably, \cite{sun2025foundation,liu2024prose} demonstrate that multi-operator models can transfer to tasks beyond those encountered during training.

At a high level, there are two common formulations. One may (i) train separate operator models independently, one per task/operator instance, or (ii) treat the target as a parameterized operator family $\{G[\alpha]\}$, where a discrete or continuous descriptor $\alpha$ encodes the operator identity. The first approach does not condition on any explicit operator descriptor and therefore can struggle when the operator family varies substantially; in particular, it offers limited leverage for generalization to unseen operators. The second approach augments operator learning with an explicit operator encoding \cite{sun2025foundation,liu2024prose,yang2023prompting,negrini2025multimodal,liu2024prosefd,weihs2025MOL}, incorporating side information such as the governing equation, symbolic representation, textual description, or task label alongside the input functions. In this way, the second approach focuses on multi-task learning and general solvers. Providing this additional context typically strengthens transfer and has emerged as a key ingredient in recent PDE foundation model works. Conditioning on operator information enables zero-shot generalization to new PDE tasks, as demonstrated in \cite{sun2025foundation}, and such approaches have shown promising performance on out-of-distribution problems without expensive retraining.

\paragraph{Theoretical analyses of approximation and statistical generalization}

A central theoretical requirement in operator learning is expressivity, where universal approximation results ensure that a given architecture can approximate broad classes of operators to arbitrary accuracy. Early foundational work developing operator network constructions and proving universal approximation for mappings between spaces of scalar-valued functions was established in \cite{ChenChen1993,ChenChen1995}. Subsequent analyses extended these guarantees to widely used architectures, including DeepONet \cite{Lanthaler2022,liu2024neuralscalinglawsdeep}, the Fourier Neural Operator \cite{Kovachki2021}, and PCA-Net \cite{Bhattacharya}, among others. Further developments related to operator expressivity, discretization effects, and architectural refinements include \cite{mionet,CASTRO2023127413,castro2022,Huang2025,Kovachki2023,zhangBelnet,zhang2025discretization}.

Beyond approximation-theoretic guarantees, scaling laws aim to quantify how error depends on data size, model capacity, and computational budget. Establishing a theoretical foundation for such laws provides a route to principled generalization estimates and predicts how performance should improve as resources increase \cite{kaplan2020scalinglawsneurallanguage}. Empirically, \cite{dehoop2022costaccuracytradeoffoperatorlearning} studies cost--accuracy trade-offs across neural operator architectures, highlighting how network size and sampling budgets affect approximation error. On the theoretical side, \cite{liu2024neuralscalinglawsdeep} derives scaling laws and complexity estimates for deep ReLU networks and DeepONet. Related analyses for DeepONet and variants appear in \cite{Lanthaler2022,lanthalerPCAnet,marcati2023,herrman,lanthalerStuart,furuya2023globally,MarcatiSchwabPolytopes,SchwabZech,SchwabStein}. Generalization error bounds for DeepONet and related models are developed in \cite{liu2024,liu2024neuralscalinglawsdeep,benitez}, while sample-complexity results are established in \cite{kovachki2024datacomplexityestimatesoperator,grohs2025theorytopracticegapneuralnetworks,adcock2025samplecomplexitylearninglipschitz}. For multi-task and multiple operator learning, empirical evidence can be found in \cite{sun2025lemonlearninglearnmultioperator,Jollie_2025}. Universal approximation results and expressivity scaling laws for MNO are derived in \cite{weihs2025MOL}; in this work, we establish generalization error estimates for MNO.

The remainder of the paper is structured as follows. In Section~\ref{sec:background}, we extensively formalize the multi-task and multiple operator learning setting and the MNO architecture,
and collect the mathematical background needed for our analysis. In Section~\ref{sec:main}, we present our main theoretical results. In Section~\ref{sec:proofs}, we provide detailed proofs. Finally, in Section~\ref{sec:discussion}, we conclude with a summary of our contributions and discuss directions for future work.

\section{Background} \label{sec:background}

This section is organized into three parts. We begin with a collection of illustrative examples that motivate the multiple operator learning viewpoint and clarify the distinct roles of the parametric function (operator descriptor), the input function, and the evaluation variable. In this context, we also recall the MNO architecture, which makes this separation explicit by modeling the dependencies on \(\alpha\), \(u\), and \(x\) through distinct components. 
Next, we summarize the scaling-law results for the MNO architecture that underpin the approximation component of our generalization analysis. Finally, we recall the covering-number estimates for the neural network classes used in our construction, which provide the complexity bounds needed for the estimation part of the proof.

\subsection{Multi-Task Problems, Multiple Operator Learning, and the MNO Architecture} \label{sec:examples}

We start by introducing a general and flexible network class used in all of our subsequent constructions.  

\begin{mydef}[Feedforward ReLU network class] \label{def:networkClass}
Let \( q : \mathbb{R}^{d_1} \to \mathbb{R} \) be a feedforward ReLU network defined as
\[
q(x) = W_L \cdot \mathrm{ReLU}\left(W_{L-1} \cdots \mathrm{ReLU}(W_1 x + b_1) + \cdots + b_{L-1} \right) + b_L,
\]
where \( W_\ell \) are weight matrices, \( b_\ell \) are bias vectors, and \( \mathrm{ReLU}(a) = \max\{a, 0\} \) is applied element-wise.

We define the class of such feedforward networks with ReLU activations:
\[
\cF_{\rm NN}(d_1, d_2, L, p, K, \kappa, R) = \left\{ [q_1, q_2, \dots, q_{d_2}]^\top \in \mathbb{R}^{d_2} \; \middle| \;
\begin{array}{l}
\text{each } q_k : \mathbb{R}^{d_1} \to \mathbb{R} \text{ has the above form with} \\
L \text{ layers, width bounded by } p, \\
\|q_k\|_{\Lp{\infty}} \leq R, \quad \|W_\ell\|_{\infty,\infty} \leq \kappa, \quad \|b_\ell\|_\infty \leq \kappa, \\
\sum_{\ell=1}^L \left( \|W_\ell\|_0 + \|b_\ell\|_0 \right) \leq K
\end{array}
\right\},
\]
where
\begin{itemize}
    \item \( \|q\|_{\Lp{\infty}} = \sup_{x \in \Omega} |q(x)| \),
    \item \( \|W_\ell\|_{\infty,\infty} = \max_{i,j} |[W_\ell]_{ij}| \),
    \item \( \|b_\ell\|_\infty = \max_i |[b_\ell]_i| \),
    \item \( \|\cdot\|_0 \) denotes the number of nonzero elements.
\end{itemize}
This network class consists of vector-valued functions with input dimension \( d_1 \), output dimension \( d_2 \), depth \( L \), width at most \( p \), at most \( K \) nonzero parameters, all bounded in magnitude by \( \kappa \), and uniformly bounded output norm by \( R \).
\end{mydef}

We recall our goal of approximating a multiple-operator map
\(
G: W \to \{G[\alpha] : U \to V\}_{\alpha \in W},
\)
where \(W,U,V\) are function spaces (with underlying domains \(\Omega_W,\Omega_U,\Omega_V\), respectively). The MNO architecture introduced in~\cite{weihs2025MOL} provides an effective and structurally aligned model class for this problem.

\begin{mydef}[$\ScalingNetwork$ Architecture] \label{def:scalingNetwork}
For fixed  positive integers \(P, H^{(p)}\), $1 \leq p \leq P$, we define a $\ScalingNetwork$ as \[
  \mathrm{MNO}[\alpha][u](x) = \sum_{p=1}^{P} \sum_{k=1}^{H^{(p)}} l_p(\bm{\alpha}) b_{pk}(\bm{u}) \tau_{pk}(x)
     \]
     for \(\alpha\in W\), \(u\in U\), and \(x\in\Omega_V\), where \(l_p\), \(b_{pk}\), and \(\tau_{pk}\) are neural networks in suitable classes \(\cF_{\rm NN}\), and \(\bm{\alpha}\), \(\bm{u}\) denote discretizations of \(\alpha\) and \(u\), respectively.
\end{mydef}
It is shown in \cite[Remark 3.20]{weihs2025MOL} that MNO is a special case of the more general fully separable architecture
\[
\sum_{p=1}^{P}\sum_{k=1}^{H}\sum_{\ell=1}^{N}
\theta_{pk\ell}\, l_p(\bm{\alpha})\, b_k(\bm{u})\, \tau_\ell(x),
\qquad \theta_{pk\ell}\in\mathbb R.
\]
Our analysis is formulated for this fully separable class, since it is more convenient for approximation and covering-number estimates. The resulting bounds transfer to MNO by a standard re-indexing (equivalently, a rearrangement of subnetworks).

We next revisit the three representative classes of examples from the introduction, now in a form that makes their connection to the MNO architecture explicit. Collectively, these examples illustrate the following recurring themes:
\begin{itemize}
    \item \textit{Distinct roles of the inputs.} The operator descriptor \(\alpha\), the input function \(u\), and the evaluation variable \(x\) play fundamentally different roles: \(\alpha\) specifies the operator instance (task), \(u\) is the operand acted upon by that operator, and \(x\) is the query location at which the output is evaluated.
    \item \textit{Shared structure and computational efficiency.} Many operator families exhibit reusable structure across \(\alpha\). Learning these families jointly allows one to represent common components once and reuse them across parameter regimes, amortizing training and reducing redundant approximation effort compared to training separate models per operator.
    \item \textit{Generalization across operator instances.} The objective in multiple operator learning is not merely to fit finitely many operators, but to learn a map \(\alpha \mapsto G[\alpha]\) that generalizes to new (possibly unseen) \(\alpha \in W\), enabling transfer to new coefficients, boundary conditions, or related tasks without retraining.
    \item \textit{Breadth of applications.} The same multiple operator learning viewpoint arises across diverse settings, including parameterized kernel operators, PDE solution operators, and task-conditioned operator families.
\end{itemize}
These principles directly motivate the MNO architecture. MNO is built to respect the separation of roles of \(\alpha\), \(u\), and \(x\) through a separable structure, in which shared subnetworks \(l\), \(b\), and \(\tau\) encode the dependence on the operator descriptor, the input function, and the evaluation variable, respectively. In particular, generalization across \(\alpha\in W\) is a natural requirement in this formulation, since \(\alpha\) enters the model as an explicit input and the learned predictor is defined for any admissible \(\alpha\) within the specified domain.

\paragraph{Parametrized integral operators}

\begin{example}[Homogeneous kernels with parameter-dependent interaction radius]
Assume that $\Omega_W = \Omega_V$ and let
\(\alpha:\Omega_V\to(0,\infty)\) represent a spatially varying interaction length scale. Consider kernels of the form
\begin{equation}\notag
K_\alpha(x,y)
=
\frac{1}{\alpha(x)^d}\,\rho\!\left(\frac{|x-y|}{\alpha(x)}\right),
\end{equation}
where \(\rho:[0,\infty)\to\mathbb R\) is a prescribed radial profile. The associated operator is
\begin{equation}\notag
G[\alpha][u](x)
=
\int_{\Omega_U}
\frac{1}{\alpha(x)^d}\,\rho\!\left(\frac{|x-y|}{\alpha(x)}\right) u(y)\,\dd y.
\end{equation}
Such parameterized kernel operators are a standard building block in nonlocal models (e.g. \cite{weihs2023discreteToContinuum}), and, under appropriate scalings of the kernel profile \(\rho\), they can be used to approximate local differential operators \cite{Bourgain01anotherlook}.

\end{example}

\begin{example}[Variable-order fractional kernel operators]
Another important class of examples is provided by variable-order fractional kernels arising in nonlocal models and fractional calculus \cite{diNezza2021Hitchhiker}. Assume \(\Omega_W=\Omega_V\), and let \(\alpha:\Omega_V\to(0,1)\) be a spatially varying order function. Define
\begin{equation}\notag
K_\alpha(x,y)
=
\frac{c_{d,\alpha(x)}}{|x-y|^{d+2\alpha(x)}},
\end{equation}
where \(c_{d,\alpha(x)}\) denotes a normalization constant depending on the spatial dimension \(d\) and the local fractional order \(\alpha(x)\). The associated operator is
\begin{equation}\notag
G[\alpha][u](x)
=
\int_{\Omega_U}
\frac{c_{d,\alpha(x)}}{|x-y|^{d+2\alpha(x)}} u(y)\,\dd y.
\end{equation}
This example makes the role separation in multiple operator learning particularly transparent. Specifically, the functions $\alpha$ and $u$ play fundamentally different roles, and therefore representing the operator as $G[\alpha,u]$, thereby treating \((\alpha,u)\) as a single concatenated input in a classical operator-learning formulation may obscure their structural distinction. The function $\alpha$ determines the kernel $K_\alpha$ and thereby defines the integration rule itself (namely how points interact, the weighting structure, and the relevant length scales). Changing $\alpha$ modifies the action under consideration and thus characterizes the task. In contrast, the function $u$ represents the data being processed; it is the input integrated against the kernel, analogous to a signal or state on which the operator acts. If $u$ varies while $\alpha$ remains fixed, the rule and therefore the task remains unchanged, and only the input within that task is varied. MNO and related multi-operator learning architectures directly mimic this hierarchical structure by explicitly separating these inputs, which models their respective roles in a mathematically consistent fashion.
\end{example}

\paragraph{Solution operator of parametrized PDEs} 

\begin{example}[Green-kernel representation of a parameterized PDE solution operator]
For broad classes of well-posed linear boundary-value problems, the solution operator admits the integral-kernel representation \eqref{eq:kernelOperator}
where \(K_\alpha\) is the Green's kernel associated with the parameterized differential operator. 

For example, consider the boundary value problem:
\[
-\,v''(x) = u(x), \qquad 0 < x < a, 
\qquad v(0)=0, \qquad v(a)=0,
\]
where $a>0$ and $u : (0,a) \to \mathbb{R}$ is a given source term.
For every $u \in \Lp{2}(0,a)$, classical elliptic regularity theory
implies the existence of a unique weak solution:
\[
v \in \Hk{1}_0(0,a) \cap \Hk{2}(0,a). 
\]
We denote the corresponding solution operator by:
\[
G[a] : \Lp{2} \to \Hk{1}_0(0,a) \cap \Hk{2}(0,a)  
\]
which is the inverse of $-\,\dfrac{d^2}{dx^2}$ on $\Hk{1}_0(0,a)$. 
The solution admits the Green representation
\[
v(x) = G[a][u](x)
= \int_0^a K_a(x,y)\, u(y)\, dy,
\qquad 0 < x < a,
\]
where the Dirichlet Green's kernel is given by
\[
K_a(x,y)
=
\frac{1}{a}
\begin{cases}
x\,(a-y), & 0 \le x \le y \le a, \\
y\,(a-x), & 0 \le y < x \le a.
\end{cases}
\]
Equivalently, for $0 \le x,y \le a$, this can be written through ReLUs:
\[
K_a(x,y)
=
\frac{x+y}{2}
- \frac{\operatorname{ReLU}(x-y) + \operatorname{ReLU}(y-x)}{2}
- \frac{1}{a} xy,
\]
where $\operatorname{ReLU}(z)=\max\{0,z\}$. Extending the integral to the interval $[0,1]$ using the
Heaviside function $H$, we obtain:
\begin{align*}
G[a][u](x)
&= \int_0^1 H(a-y)\, K_a(x,y)\, u(y)\, dy \\
&= \int_0^1 \ell^{(1)}(a,y)\, \tau^{(1)}(x,y)\, u(y)\, dy
\;+\;
\int_0^1 \ell^{(2)}(a,y)\, \tau^{(2)}(x,y)\, u(y)\, dy,
\end{align*}
where the functions are separated into:
\[
\ell^{(1)}(a,y) = H(a-y),
\qquad
\ell^{(2)}(a,y) = -\frac{H(a-y)}{a},
\]
and
\[
\tau^{(1)}(x,y)
= \frac{x+y}{2}
- \frac{\operatorname{ReLU}(x-y) + \operatorname{ReLU}(y-x)}{2},
\qquad
\tau^{(2)}(x,y) = xy.
\]
This representation for $G[a][u]$ shows that the operator is in fact a finite sum of separable kernel components.
Letting $\alpha:\Omega_V \to \bbR$ be the constant function $a$, a Monte Carlo
quadrature with sampling nodes $Y_i \sim \mathrm{Unif}(0,1)$ yields:
\begin{equation} \notag
G[\alpha][u](x)
\approx
\sum_{p=1}^{2}
\sum_{i=1}^{N}
\ell^{(p)}(\alpha,Y_i)\,
b_i(u)\,
\tau^{(p)}(x,Y_i),
\end{equation}
where
\begin{align*}
\ell^{(1)}(\alpha,Y_i) &= \frac{H(\alpha-Y_i)}{N}, \qquad
\ell^{(2)}(\alpha,Y_i) = -\frac{H(\alpha-Y_i)}{N\alpha}, \qquad b_i(u) = u(Y_i), \\
\tau^{(1)}(x,Y_i)
&=
\frac{x+Y_i}{2}
-\frac{\operatorname{ReLU}(x-Y_i)+\operatorname{ReLU}(Y_i-x)}{2}, \qquad
\tau^{(2)}(x,Y_i) = xY_i.
\end{align*}
The Monte Carlo approximation therefore yields the same structural
form as the MNO ansatz, i.e., a finite sum of separable (low-rank)
components where the nodes $Y_i$ can be can be absorbed into the network parameters. 
The functions $\ell^{(p)}$ encode the dependence on the
operator parameter $\alpha$, the coefficients $b_i$ encode the dependence on the input function $u$ through point evaluations, and the functions $\tau^{(p)}$ encode the dependence on the output variable $x$.
In this sense, the Green's solution provides an explicit,
kernel-based realization of the separable operator structure. 
Similar decompositions arise for other linear PDE through their Green's formulations.
\end{example}

\begin{example}[Nonlinear PDE solution operator with a shared semigroup structure]
For nonlinear PDEs, the solution map 
$u \mapsto G[\alpha][u]$ is typically nonlinear and therefore cannot, in general, be represented by a single kernel acting linearly on $u$. 
We illustrate the idea of multi-operator approximations using a parameterized family of equations. Let $\alpha=(\sigma,\nu)$ be the model parameters, where
\[
\sigma \in \{0,1\}
\quad \text{and} \quad
\nu>0,
\]
and consider the PDE
\[
\partial_t z
+ \sigma\, z\,\partial_x z
=
\nu\,\partial_{xx} z,
\qquad 
z(0,x)=u(x),
\qquad x\in\mathbb{R}.
\]
When $\sigma=0$, this reduces to the linear heat equation; 
when $\sigma=1$, it is the viscous Burgers equation.
We denote the corresponding solution operator by
\[
u \rightarrow G[\alpha][u].
\]
The linear heat equation setting occurs when $\alpha=(0,\nu)$, and the solution is given by the heat semigroup:
\[
S_t^\nu[u](x)
:=
\int_{\mathbb{R}}
\Gamma_\nu(t,x-y)\,u(y)\,dy,
\]
where
\[
\Gamma_\nu(t,z)
=
\frac{1}{\sqrt{4\pi \nu t}}
\exp\!\left(-\frac{z^2}{4\nu t}\right).
\]
Thus, the heat operator is given by:
$G[(0,\nu)][u]
=
S_t^\nu[u]$.
For the Burgers' case ($\sigma=1$), if $\alpha=(1,\nu)$, the Cole-Hopf transformation
\[
\phi(t,x)
=
\exp\left(
-\frac{1}{2\nu}
\int_{0}^{x} z(t,\xi)\, d\xi
\right)
\]
reduces the equation to the heat equation
\[
\partial_t \phi = \nu\,\partial_{xx} \phi,
\]
and hence,
\[
\phi(t,x)
=
\int_{\mathbb{R}}
\Gamma_\nu(t,x-y)\,
\exp\!\left(
-\frac{1}{2\nu}
\int_{0}^{y} u(\xi)\, d\xi
\right)
dy.
\]
Transforming back yields
\[
G[(1,\nu)][u](x)
=
-2\nu\,\partial_x \log \phi(t,x),
\]
or equivalently,
\[
G[(1,\nu)][u](x)
=
\frac{1}{t}
\frac{
\int_{\mathbb{R}} (x-y)\,\Gamma_\nu(t,x-y)\,E_u^\nu(y)\,dy
}{
\int_{\mathbb{R}} \Gamma_\nu(t,x-y)\,E_u^\nu(y)\,dy
},
\]
where
\[
E_u^\nu(y)
=
\exp\!\left(
-\frac{1}{2\nu}
\int_{0}^{y} u(\xi)\, d\xi
\right).
\]

In both cases, the same linear heat semigroup $S_t^\nu[u](x)$ appears as a common sub-operator. The main difference between the two PDEs' solution operator lies only in nonlinear input/output transformations.
Define
\[
\mathcal{P}_{\sigma,\nu}[u]
=
\begin{cases}
u, & \sigma=0, \\
E_u^\nu, & \sigma=1,
\end{cases}
\qquad
\mathcal{T}_{\sigma,\nu}[\phi]
=
\begin{cases}
\phi, & \sigma=0, \\
-2\nu\,\partial_x \log \phi, & \sigma=1.
\end{cases}
\]
Then both solution operators admit the unified representation
\[
G[\alpha]
=
\mathcal{T}_{\sigma,\nu}
\circ
S_t^\nu
\circ
\mathcal{P}_{\sigma,\nu},
\qquad
\alpha=(\sigma,\nu).
\]
If generating the solutions to the two PDEs are modeled by independent by neural networks, as in the single operator learning case, one constructs
\[
\nn_{(0,\nu)} \approx S_t^\nu,
\qquad
\nn_{(1,\nu)} \approx
\mathcal{T}_{1,\nu} \circ S_t^\nu \circ \mathcal{P}_{1,\nu}.
\]
In this case, the heat semigroup $S_t^\nu$ must effectively be learned
twice. If $\mathcal{C}(S_t^\nu)$ denotes the approximation complexity
(e.g., rank, width, or parameter count) required to represent
$S_t^\nu$, then the total complexity scales like:
\[
2\,\mathcal{C}(S_t^\nu)
+
\mathcal{C}_{\mathrm{nonlinear}}.
\]
By contrast, in a simultaneous (multi-operator) setting, the family:
\[
u \to G[\alpha][u],
\qquad
\alpha=(\sigma,\nu),
\]
is learned jointly. The common operator $S_t^\nu$ is approximated
once, while only the lower-complexity transformations
$\mathcal{P}_{\sigma,\nu}$ and $\mathcal{T}_{\sigma,\nu}$ depend on
the equation type. The resulting complexity scales as
\[
\mathcal{C}(S_t^\nu)
+
\mathcal{C}_{\mathrm{nonlinear}},
\]
which is significantly smaller whenever the representation of the heat semigroup dominates the approximation cost. Thus, simultaneous operator learning leverages the shared linear propagator indexed by $\nu$, whereas disjoint learning redundantly approximates the same semigroup structure for each equation type.

\end{example}

\paragraph{Operator families indexed by symbolic or textual descriptions}

\begin{example}[PROSE architecture]
The family of networks defined by the PROSE architecture \cite{liu2024prosefd,liu2024prose} learns nonlinear operators $G[\alpha][u](x)$ using a multimodal Transformer framework.  The objective is to approximate an operator where \(\alpha\) denotes auxiliary (e.g., symbolic, text, or parametric) information and \(u\) denotes a function represented through sampled observations. 

In PROSE, the inputs \(\alpha\) and \(u\) are first mapped into a shared latent space via separate MLP encoders, $\tilde{\alpha} = \Phi_\alpha(\alpha)$, $\tilde{u} = \Phi_u(u)$, where \(\Phi_\alpha\) and \(\Phi_u\) are learned nonlinear embeddings. The encoded tokens are then concatenated to form
$S = [\tilde{\alpha}, \tilde{u}] \in \mathbb{R}^{2 \times d}$. A self-attention layer integrates information across modalities:
\[
\begin{aligned}
Y 
= \mathrm{SelfAttention}(S) = \operatorname{softmax}\left(
\frac{(S W_Q^{(s)})(S W_K^{(s)})^{\top}}{\sqrt{d_k}}
\right)
(S W_V^{(s)}),
\end{aligned}
\]
where \(W_Q^{(s)}, W_K^{(s)}, W_V^{(s)}\) are learned projection matrices and the softmax is applied row-wise. This step produces fused latent representations  $Y \in \mathbb{R}^{2 \times d_v}$. The PROSE framework resembles an operator since the output function can be evaluate at query location \(x\) as follows. The query embedding $\tilde{x} = \Phi_x(x)$,
is used in a cross-attention mechanism with keys and values derived from the processed inputs $Y$:
\[
\mathrm{CrossAttention}(\tilde{x}, Y)
=
\operatorname{softmax}\left(
\frac{(\tilde{x} W_Q^{(c)})(Y W_K^{(c)})^{\top}}{\sqrt{d_k}}
\right)
(Y W_V^{(c)}).
\]
The resulting output is mapped through decoder
\(\Psi\) (e.g., a simple MLP) to produce the scalar (or vector-valued) operator evaluation,
\[
G[\alpha][u](x)
=
\Psi\left(
\mathrm{CrossAttention}(\tilde{x}, Y)
\right).
\]
While the above description uses a single-head, in practice \cite{liu2024prosefd,liu2024prose} a multi-head formulation is used. This generalizes the single-head formulation by introducing head-indexed projection matrices 
\(
\{W_{Q,h}, W_{K,h}, W_{V,h}\}_{h=1}^H
\),
computing attention independently across heads, and aggregating the resulting  representations via concatenation followed by a learned output projection. Mathematically,
\[
\mathrm{MultiHead}(S)
=
\mathrm{Concat}\big(
\mathrm{head}_1(S), \dots, \mathrm{head}_H(S)
\big) W_O,
\]
where for each head \(h\),
\[
\mathrm{head}_h(S)
=
\operatorname{softmax}\left(
\frac{(S W_{Q,h})(S W_{K,h})^{\top}}{\sqrt{d_k}}
\right)
(S W_{V,h}),
\]
and \(W_O\) denotes the output projection matrix. These architectures illustrate that generalization across operator instances \(\alpha\) is naturally incorporated by treating \(\alpha\) as an explicit input modality. 
\end{example}

\subsection{Scaling Laws for Multiple Operator Learning}

In this section, we review the approximation-theoretic results underlying the MNO architecture, which provide the approximation term in Theorem~\ref{thm:scalingLawsGeneralizationError} and motivate the \(\varepsilon\)-dependent instantiation of the network hypothesis class used throughout the proof of the latter. Specifically, we begin by recalling \cite[Theorem 3.16]{weihs2025MOL}, which establishes scaling laws for the expressivity of a general multiple operator architecture; the corresponding scaling laws for MNO follow as a special case. 

\begin{theorem}[Multiple Operator Scaling Laws]\label{thm:main:multipleOperatorApproximation}
Let $d_W,d_U,d_V>0$ be integers, \[
\gamma_W, \gamma_U, \gamma_V, \beta_W, \beta_U,\beta_V,L_W,L_U,L_V,L_G,L_{\mathcal{G}} > 0 \qquad \text{and} \qquad r_G, r_{\mathcal{G}} \geq 1\]
and assume that $W(d_W,\gamma_W,L_W,\beta_W)$,
 $U(d_U,\gamma_U,L_U,\beta_U)$ and $V(d_V,\gamma_V,L_V,\beta_V)$ satisfy Assumption \ref{assumption:Main:assumptions:S4}. 
 Let $G$ be a map such that \begin{align*}
 	&G:\{ \alpha:\Omega_W \mapsto \bbR \spaceBar \Vert \alpha \Vert_{\Lp{\infty}} \leq \beta_W \} \mapsto \mathcal{G} \qquad \text{where }\\
 	&\mathcal{G} = \Big\{ G[\alpha]\spaceBar G[\alpha]:\{ u:\Omega_U \mapsto \bbR \spaceBar \Vert u \Vert_{\Lp{\infty}} \leq \beta_U \} \mapsto V \text{ and } \\
 	&\quad \Vert G[\alpha][u_1] - G[\alpha][u_2]\Vert_{\Lp{\infty}(\Omega_V)} \leq L_{\mathcal{G}}\Vert u_1 - u_2 \Vert_{\Lp{r_\mathcal{G}}(\Omega_U)} \Big\} 
 \end{align*}
Furthermore, assume that $G$ satisfies \begin{equation*} 
   \Vert G(\alpha_1) - G(\alpha_2) \Vert_{\Lp{\infty}(\{ u:\Omega_U \mapsto \bbR \spaceBar \Vert u \Vert_{\Lp{\infty}} \leq \beta_U \} \times \Omega_V)} \leq L_G \Vert \alpha_1 - \alpha_2 \Vert_{\Lp{r_G}(\Omega_W)}  
\end{equation*}
for $\alpha_1,\alpha_2 \in \{ \alpha:\Omega_W \mapsto \bbR \spaceBar \Vert \alpha \Vert_{\Lp{\infty}} \leq \beta_W \}$.

There exists constants $C$ depending on $\gamma_V,L_V$, $C_{\delta}$ depending on $L_{\mathcal{G}},d_U,\gamma_U,r_{\mathcal{G}},L_U$, $C'$ depending on $\beta_U, L_{\mathcal{G}}, d_U, \gamma_U, r_{\mathcal{G}}$, $C_\zeta$ depending on $L_G, d_W, \gamma_W, r_G,L_W$ and $C''$ depending on $\beta_W, L_G, d_W, \gamma_W, r_G$ such that the following holds. For any $\varepsilon>0$, \begin{itemize}
     
     \item let $N=  2^{n_{c_W} + 2} C \sqrt{d_V} (C'' \sqrt{n_{c_W}})^{n_{c_W}} \varepsilon^{-(n_{c_W}+1)}$ and consider the network class \newline $\cF_1=\cF_{\rm NN}(d_V,1,L_1,p_1,K_1,\kappa_1,R_1)$  with parameters scaling as
    \begin{align*}
    &L_1 = \mathcal{O}\left(d_V^2\log d_V+d_V^2(n_{c_W}+1)\log(\varepsilon^{-1}) +d_V^2\log(2^{n_{c_W}+1}(C''\sqrt{n_{c_W}})^{n_{c_W}}) + d_V^2\log(2)\right),\\
    &\quad  p_1 = \mathcal{O}(1),\\
    &K_1 = \mathcal{O}\left(d_V^2\log d_V+d_V^2(n_{c_W}+1)\log(\varepsilon^{-1}) +d_V^2\log(2^{n_{c_W}+1}(C''\sqrt{n_{c_W}})^{n_{c_W}}) +d_V^2\log(2)\right)\\
&\kappa_1=\mathcal{O}(d_V^{d_V/2+1}\varepsilon^{-(d_V+1)(n_{c_W}+1)} \ls 2^{n_{c_W}+2}(C''\sqrt{n_{c_W}})^{n_{c_W}}\rs^{(d_V+1)}),\qquad R_1=1
    \end{align*}
    where the constants hidden in $\mathcal{O}$ depend on $\gamma_V$ and $L_V$;

    \item let $\{v_\ell\}_{\ell=1}^{N^{d_V}} \subset \Omega_V$ be a uniform grid with spacing $2\gamma_V/N$ along each dimension;
    
    \item let $\delta=\frac{C_{\delta}\varepsilon^{(1+d_V)(1+n_{c_W})}}{2^{d_V+n_{c_W}+2}(C\sqrt{d_V})^{d_V}(C''\sqrt{n_{c_W}})^{n_{c_W}}}$ and let $\{c_m\}_{m=1}^{n_{c_U}}\subset \Omega_U$ be points so that $\{\mathcal{B}_{\delta}(c_m) \}_{ m  = 1}^{n_{c_U}}$ is a cover of $\Omega_U$ for some $n_{c_U}$;
    
    \item let $H = 2^{(d_V+1)(n_{c_W}+2)}C' \sqrt{n_{c_U}} (C \sqrt{d_V})^{d_V} (C'' \sqrt{n_{c_W}})^{n_{c_W}(d_V +1)} \eps^{-(d_V+1)(1+n_{c_W})}$ and consider the network class $\cF_2=\cF_{\rm NN}(n_{c_U},1,L_2,p_2,K_2,\kappa_2,R_2)$ with parameters scaling as
\begin{align*}
        &L_2 = \mathcal{O}\big(n_{c_U}^2\log n_{c_U}+n_{c_U}^2(d_V +1)(n_{c_W}+1)\log(\varepsilon^{-1}) + n_{c_U}^2 \log(2^{d_V +1} (C \sqrt{d_V})^{d_V} )\\
        &\qquad + n_{c_U}^2(d_V +1)\log(2^{n_{c_W}+1}(C''\sqrt{n_{c_W}})^{n_{c_W}}) + n_{c_U}^2\log(2)\big),\quad p_2=\mathcal{O}(1), \\
&K_2 = \mathcal{O}\big(n_{c_U}^2\log n_{c_U}+n_{c_U}^2(d_V +1)(n_{c_W}+1)\log(\varepsilon^{-1}) + n_{c_U}^2 \log(2^{d_V +1} (C \sqrt{d_V})^{d_V} )\\
        &\qquad + n_{c_U}^2(d_V +1)\log(2^{n_{c_W}+1}(C''\sqrt{n_{c_W}})^{n_{c_W}}) + n_{c_U}^2\log(2)\big), \\ 
&\kappa_2=\mathcal{O}(n_{c_U}^{n_{c_U}/2+1}\varepsilon^{-(d_V+1)(n_{c_U}+1)(n_{c_{W}}+1)}[ 2^{d_V +2} (C \sqrt{d_V})^{d_V} ]^{n_{c_U}+1} \ls 2^{d_V +1} (C \sqrt{d_V})^{d_V} \rs^{(d_V+1)(n_{c_U}+1)}), \\
&R_2=1
    \end{align*}
    where the constants hidden in $\mathcal{O}$ depend on $\beta_U, L_{\mathcal{G}}, d_U, \gamma_U,r_{\mathcal{G}}$;

\item let $\zeta=C_{\zeta}\varepsilon$ and let $\{y_m\}_{m=1}^{n_{c_W}}\subset \Omega_W$ be points so that $\{\mathcal{B}_{\zeta}(y_m) \}_{ m = 1}^{n_{c_W}}$ is a cover of $\Omega_W$ for some $n_{c_W}$;

\item let $P = 2C'' \sqrt{n_{c_W}} \eps^{-1}$ and consider the network class $\cF_3=\cF_{\rm NN}(n_{c_W},1,L_3,p_3,K_3,\kappa_3,R_3)$ with parameters scaling as
\begin{align*}
       &L_3=\mathcal{O}\left(n_{c_W}^2\log(n_{c_W})+n_{c_W}^2\log(\varepsilon^{-1}) + n_{c_W}^2 \log(2)\right),\quad  p_3 = \mathcal{O}(1),\\
       &K_3 = \mathcal{O}\left(n_{c_W}^2\log n_{c_W}+n_{c_W}^2\log(\varepsilon^{-1}) + n_{c_W}^2 \log(2)\right), \\ &\kappa_3=\mathcal{O}(n_{c_W}^{n_{c_W}/2+1}2^{n_{c_W+1}}\varepsilon^{-n_{c_W}-1}),\qquad \, R_3=1
    \end{align*}
    where the constants hidden in $\mathcal{O}$ depend on $\beta_W, L_G, d_W, \gamma_W,r_G$.
    
 \end{itemize} 
 Then, there exists networks $\{\tau_\ell\}_{\ell=1}^{N^{d_V}} \subset \cF_1$, networks $\{b_k\}_{k=1}^{H^{n_{c_U}}} \subset \mathcal{F}_2$, networks $\{l_p\}_{p=1}^{P} \subset \mathcal{F}_3$, functions $\{u_k\}_{k=1}^{H^{n_{c_U}}} \subset \{ u:\Omega_U \mapsto \bbR \spaceBar \Vert u \Vert_{\Lp{\infty}} \leq \beta_U \}$ and functions $\{\alpha_p\}_{p=1}^P \subset \{ \alpha:\Omega_W \mapsto \bbR \spaceBar \Vert \alpha \Vert_{\Lp{\infty}} \leq \beta_W \}$ such that 
 \begin{align}
        \sup_{\alpha\in W}\sup_{u\in U}\sup_{x \in \Omega_V}\left|G[\alpha][u](x)-\sum_{p=1}^{P^{n_{c_W}}}\sum_{k=1}^{H^{n_{c_U}}} \sum_{\ell=1}^{N^{d_V}} G[\alpha_p][u_k](v_{\ell}) l_p(\bm{\alpha}) b_k(\ub) \tau_{\ell}(x)\right|\leq \varepsilon, \label{eq:main:multipleOperatorApproximation}
    \end{align}
where $\bm{\alpha}=(\alpha(y_1), \alpha(y_2),...,\alpha(y_{n_{c_W}}))^\top$ is a discretization of $\alpha$ and $\ub=(u(c_1), u(c_2),...,u(c_{n_{c_U}}))^\top$ is a discretization of $u$.

\end{theorem}

Next, for a function $f : \mathbb{R} \to \mathbb{R}$ and a constant $a \ge 0$, we define the
\emph{clipping operator} that constrains the range of $f$ to the interval $[-a,a]$:
\[
    \operatorname{Clip}_a(f) \;=\; \min\{\max(f,-a),\, a\}.
\]
This truncation can be implemented exactly by a two–layer ReLU network. One explicit
realization is
\begin{equation*} 
    \operatorname{Clip}_a(f) =
    -\mathrm{ReLU}\left(-\mathrm{ReLU}(f+a) + 2a\right) + a,
\end{equation*}
which expresses the clipping operation using only affine maps and ReLU activations. In particular, such a network is in the class $\cF_{\rm NN}(1,1,2,1,6,2a,a)$.

The next result incorporates the clipping operation into the general multiple operator architecture and shows that the resulting clipped network class admits analogous expressivity guarantees and scaling laws.

\begin{corollary}[Clipped network scaling laws] \label{cor:back:clippedScalingLaws}
Assume the same setting as in Theorem \ref{thm:main:multipleOperatorApproximation}. Let $\nn[\bm{\alpha}][\bm{u}](x)$ be the network such that \eqref{eq:main:multipleOperatorApproximation} holds for $\eps/2$. Then, \begin{equation} \label{eq:cor:back:clippedScalingLaws}
    \sup_{\alpha \in W} \sup_{u \in U} \sup_{x \in \Omega_V} \left\vert G[\alpha][u](x) -  \operatorname{Clip}_{\beta_V} \l \nn [\bm{\alpha}][\bm{u}](x) \r \right\vert \leq \eps.
\end{equation}
\end{corollary}

In view of Corollary~\ref{cor:back:clippedScalingLaws}, clipping can be incorporated without loss of approximation power (up to adjustment of the target accuracy). Accordingly, we henceforth take our hypothesis class to consist of clipped multiple operator networks. This boundedness property is a technical ingredient in the generalization analysis used in the proof of Theorem~\ref{thm:scalingLawsGeneralizationError}. We formalize the resulting model class in the following definition.

\begin{mydef}[Clipped multiple operator network class] \label{def:clippedClass}
    Let $\cF_i$ for $1 \leq i \leq 3$ be network classes defined in Definition \ref{def:networkClass}. For $a,I > 0$, $P,H,N \in \bbN$ and fixed sampling points $\{y_s\}_{s=1}^{n_{c_W}} \subset \Omega_W$ and $\{c_s\}_{s=1}^{n_{c_U}}$, we define $ \mathrm{Cl}_a(I,\cF_1,\cF_2,\cF_3,\{y_s\},\{c_s\},P,H,N)$, the set of $a$-clipped multiple operator networks, as \begin{align}
    &\mathrm{Cl}_a(I,\cF_1,\cF_2,\cF_3,\{y_s\},\{c_s\},P,H,N) \notag \\
    &= \left\{ \operatorname{Clip}_a\l \sum_{p=1}^{P} \sum_{k=1}^H \sum_{\ell=1}^{N} \theta_{pk\ell} l_p(\bm{\alpha}) b_k(\mathbf{u}) \tau_\ell(x)\r \mid \theta_{pk\ell} \in [-I,I],\, \tau_\ell \in \cF_1,\, b_k \in \cF_2, \, l_p \in \cF_3  \right\} \notag
    \end{align}
    where $\bm{\alpha}=(\alpha(y_1), \alpha(y_2),...,\alpha(y_{n_{c_W}}))^\top$ is a discretization of $\alpha$ and $\ub=(u(c_1), u(c_2),...,u(c_{n_{c_U}}))^\top$ is a discretization of $u$.
\end{mydef}

\begin{remark}[Scaling laws for clipped multiple operator network classes] \label{rem:eps2}
For $1 \leq i \leq 3$, let $\cF_i(\eps)$ denote the network classes in Theorem \ref{thm:main:multipleOperatorApproximation} such that \eqref{eq:main:multipleOperatorApproximation} holds for $\eps$. 
Corollary \ref{cor:back:clippedScalingLaws} implies that there exists a network in $$\mathrm{Cl}_{\beta_V}(\beta_V,\cF_1(\eps/2),\cF_2(\eps/2),\cF_3(\eps/2),\{y_s\},\{c_s\},P^{n_{c_W}},H^{n_{c_U}},N^{d_V})$$ such that \eqref{eq:cor:back:clippedScalingLaws} holds for $\eps$. Indeed, the network in Corollary \ref{cor:back:clippedScalingLaws} is clipped at $\beta_V$ and has coefficients $\theta_{pk\ell}  = G[\alpha_p][u_k](v_\ell)$ which satisfy $\vert G[\alpha_p][u_k](v_\ell) \vert \leq \beta_V$ by assumption on $G$.
\end{remark}

\subsection{Covering Number of Neural Networks}
In this section, we discuss covering numbers of the neural network classes defined in Definition \ref{def:networkClass}.
\begin{mydef}[Covering Number]
Let \( (X, d) \) be a metric space and let \( \eta > 0 \). A finite subset \( \mathcal{C} \subset X \) is called a \( \theta \)-cover of \( X \) if for every \( x \in X \), there exists \( c \in \mathcal{C} \) such that
\[
d(x, c) \leq \eta.
\]
The covering number of \( X \) at scale \( \theta \) with respect to the metric \( d \) is defined as
\[
\mathcal{N}(\eta, X, d) := \min \left\{ |\mathcal{C}| : \mathcal{C} \subset X \text{ is a } \eta\text{-cover of } X \right\}.
\]
\end{mydef}
The next result is similar to 
\cite[Lemma 7]{chenZhao2022}, \cite[Lemma 3.2]{ou2024quantizationregimesrelunetworks} or extensions in \cite[Theorem 2.1]{ou2024coveringnumbersdeeprelu}.
\begin{proposition}[Covering number of feedforward ReLU network class] \label{prop:back:covering}
    Let $\cF_{\nn}(d_1,1,L,p,K,\kappa,R)$ be the network class defined in Definition \ref{def:networkClass} and suppose that $\kappa \geq 1$. Let $d(q_1,q_2)$ denote the maximum parameter discrepancy for $q_1,q_2 \in \cF_{\nn}(d_1,1,L,p,K,\kappa,R)$, that is \[
d(q_1,q_2) = \max_{1 \leq \ell \leq L} \max\{\Vert W_\ell^{(1)} - W_{\ell}^{(2)} \Vert_{\infty,\infty}, \Vert b_\ell^{(1)} - b_\ell^{(2)} \Vert_\infty\}.
\]
Then, the following identities hold: \begin{enumerate}
    \item For any $q \in \cF_{\nn}(d_1,1,L,p,K,\kappa,R)$, the $\Lp{\infty}$-norm of the output is bounded as follows: \begin{equation} \label{eq:prop:covering:boundOutput}
    \Vert q \Vert_{\Lp{\infty}} \leq \kappa^L(p+1)^{L-1}(p\Vert x \Vert_{\Lp{\infty}} + 1);
\end{equation}
    \item For any $q_1,q_2 \in \cF_{\nn}(d_1,1,L,p,K,\kappa,R)$, the $\Lp{\infty}$-norm of the difference between the outputs is bounded as follows: \begin{equation} \label{eq:prop:covering:boundDifferenceOutput}
   \Vert q_1 - q_2 \Vert_{\Lp{\infty}} \leq L \kappa^{L-1}(p+1)^{L-1}(p\Vert x \Vert_{\Lp{\infty}}+1)d(q_1,q_2);
\end{equation}
    \item The covering number of $\cF_{\nn}(d_1,1,L,p,K,\kappa,R)$ is bounded as follows: \begin{equation*} 
        \cN\l\eta,\cF_{\nn}(d_1,1,L,p,K,\kappa,R),\Vert \cdot \Vert_{\Lp{\infty}}\r \leq \binom{L(p^2 + p)}{K} \left( \left\lfloor \frac{L \kappa^{L}(p+1)^{L-1} (p\Vert x \Vert_{\Lp{\infty}} + 1)}{\eta} \right\rfloor + 1 \right)^K.
    \end{equation*}
\end{enumerate}
    
\end{proposition}

\section{Main Results} \label{sec:main}

We start by introducing the standing assumptions on the underlying spaces utilized in the main results.

\begin{enumerate}[label=\textbf{S.}]
\item The space $U(d_U,\gamma_U,L_U,\beta_U)$ is a function set such that \begin{enumerate}
    \item any function $u \in U$ is defined on $\Omega_U := [-\gamma_U,\gamma_U]^{d_U}$;
    \item for all functions $u \in U$ and $x,y \in \Omega_U$, we have \[
    \vert u(x) - u(y) \vert \leq L_U \vert x - y \vert;
    \]
    \item for all functions $u \in U$, we have $\Vert u \Vert_{\Lp{\infty}} \leq \beta_U$.
\end{enumerate}\label{assumption:Main:assumptions:S4}
\end{enumerate}
In the remainder of this section, we state our main results and discuss their interpretation. We conclude with a brief proof sketch highlighting the key ideas and the role of the intermediate technical lemmas.

\subsection{Covering Numbers of Product Neural Network Classes}

Our first result concerns the covering number of the clipped multiple operator network class from Definition~\ref{def:clippedClass}. Since this class serves as the hypothesis class in Theorem~\ref{thm:scalingLawsGeneralizationError}, the estimation terms in the generalization bound involve \(\log \cN(\eta,\mathrm{Cl}_a,\|\cdot\|_{L^\infty(W\times U\times\Omega_V)})\) and therefore require an explicit covering-number estimate.

\begin{proposition}[Covering number of the clipped multiple operator network class] \label{prop:back:coveringClippledMultipleOperator}
Let $$\mathrm{Cl}_a(I,\cF_1,\cF_2,\cF_3,\{y_s\},\{c_s\},P,H,N)$$ be the clipped multiple operator class defined in Definition \ref{def:clippedClass} where, for each $\nn$ in that class, we assume the following input domains: $\nn: W \times U \times \Omega_V \mapsto \bbR$. Let \[
F(L,p,K,\kappa,h) =  \binom{L(p^2 + p)}{K} \left( \left\lfloor \frac{2\kappa}{h} \right\rfloor + 1 \right)^K.
\] 
Then, with \(
h = 2\eta/T 
\)
where \begin{align}
    T &= P \cdot H \cdot N \cdot  \Bigg[ I R_2 R_3 L_1\kappa_1^{L_1-1}(p_1+1)^{L_1-1}(p_1 \Vert x \Vert_{\Lp{\infty}} +1) \notag \\
        &+ I R_1 R_3 L_2\kappa_2^{L_2-1}(p_2+1)^{L_2-1}(p_2 \Vert \bm{u} \Vert_{\Lp{\infty}} +1) \notag \\
        &+ I R_1R_2 L_3\kappa_3^{L_3-1}(p_3+1)^{L_3-1}(p_3 \Vert \bm{\alpha} \Vert_{\Lp{\infty}} +1)  + R_1 R_2 R_3  \Bigg], \notag
\end{align}
we obtain the following upper bound on the covering number: \begin{align}
    &\cN\l\eta,\mathrm{Cl}_a(I,\cF_1,\cF_2,\cF_3,\{y_s\},\{c_s\},P,H,N),\Vert \cdot \Vert_{\Lp{\infty}(W \times U \times \Omega_V)}\r \notag \\
    &\leq \ls \l \lfloor 2I/h \rfloor + 1\r  F(L_3,p_3,K_3,\kappa_3,h) F(L_2,p_2,K_2,\kappa_2,h) F(L_1,p_1,K_1,\kappa_1,h) \rs^{ P \cdot H \cdot N}. \notag
\end{align} 

If, in addition, for $d_W,d_U,d_V>0$ integers and \(
\gamma_W, \gamma_U, \gamma_V, \beta_W, \beta_U,\beta_V,L_W,L_U,L_V > 0, 
\)
the spaces $W(d_W,\gamma_W,L_W,\beta_W)$,
 $U(d_U,\gamma_U,L_U,\beta_U)$ and $V(d_V,\gamma_V,L_V,\beta_V)$ satisfy Assumption \ref{assumption:Main:assumptions:S4}, then we can pick 
\begin{align}
    T &= P \cdot H \cdot N \cdot  \Bigg[ I R_2 R_3 L_1\kappa_1^{L_1-1}(p_1+1)^{L_1-1}(p_1 \gamma_V +1) \notag \\
        &+ I R_1 R_3 L_2\kappa_2^{L_2-1}(p_2+1)^{L_2-1}(p_2 \beta_U +1) \notag \\
        &+ I R_1R_2 L_3\kappa_3^{L_3-1}(p_3+1)^{L_3-1}(p_3 \beta_W +1)  + R_1 R_2 R_3  \Bigg]. \notag
\end{align}
In particular, the covering number is bounded uniformly in $\alpha \in W$, $u \in U$ and $x \in \Omega_V$.
\end{proposition}

We note that Proposition~\ref{prop:back:coveringClippledMultipleOperator} yields a covering-number bound that depends explicitly on all architectural parameters of the component network classes \(\cF_1,\cF_2,\cF_3\) constituting the hypothesis class, including their depths \(L_i\), widths \(p_i\), sparsity budgets \(K_i\), parameter magnitudes \(\kappa_i\), and output bounds \(R_i\), as well as the multiplicities \(P,H,N\) and coefficient magnitude $I$ of the separable expansion.

\subsection{Generalization Bounds}

We now formalize the data-generating and sampling procedure used throughout the paper. In the multiple operator setting the training data are naturally collected in a hierarchical manner: we first sample operator instances \(\alpha\), then sample input functions \(u\) conditional on each \(\alpha\), and finally sample evaluation points \(x\) at which noisy observations of \(G[\alpha][u]\) are recorded. The following definition makes this hierarchical training-set structure precise.

We recall that $X \sim \mathrm{subG(\sigma^2)}$, i.e. $X$ is a mean 0 sub-Gaussian random variable with variance proxy $\sigma^2$ if $\bbE(e^{\lambda X}) \leq e^{(\lambda^2 \sigma^2)/2}$. In particular, if $X \sim \mathrm{subG}(\sigma^2)$, then $aX \sim \mathrm{subG}(a^2\sigma^2)$.

\begin{mydef}[Training set] \label{def:trainingSet}
Let $G : W \mapsto \{G[\alpha] : U \to V\}$ be a map. Let $\mu_\alpha$ be a probability measure on $W$, $\mu_u$ a probability measure on $U$, and $\mu_x$ a probability measure on $\Omega_V$. Given fixed sampling points $\{y_s\}_{s=1}^{n_{c_W}} \subset \Omega_W$ and $\{c_s\}_{s=1}^{n_{c_U}} \subset \Omega_U$, we define the training set:
\[
S_{G, \{y_s\}, \{c_s\}} =
\left\{
\boldsymbol{\alpha}_\ell,
\left\{
\boldsymbol{u}_{\ell i},
\left\{ (x_{\ell ij}, w_{\ell ij}) \right\}_{j=1}^{n_x}
\right\}_{i=1}^{n_u}
\right\}_{\ell = 1}^{n_\alpha}
\]
where
\begin{itemize}
    \item \( \alpha_\ell \iid \mu_\alpha \) and \( \boldsymbol{\alpha}_\ell = \big(\alpha_\ell(y_1), \dots, \alpha_\ell(y_{n_{c_W}})\big) \in \mathbb{R}^{n_{c_W}} \);
    \item \( u_{\ell i} \iid \mu_u \) and \( \boldsymbol{u}_{\ell i} = \big(u_{\ell i}(c_1), \dots, u_{\ell i}(c_{n_{c_U}})\big) \in \mathbb{R}^{n_{c_U}} \);
    \item \( x_{\ell ij} \iid \mu_x \) drawn from \( \Omega_V \);
    \item \( w_{\ell ij} = G[\alpha_\ell][u_{\ell i}](x_{\ell ij}) + \zeta_{\ell ij} \), where \( \zeta_{\ell ij} \) are i.i.d. sub-Gaussian noise variables with mean 0 and variance proxy \( \sigma^2 \). 
\end{itemize}
All random variables are assumed independent across all indices. This structure is illustrated schematically in Figure~\ref{fig:training-set-diagram}.
\end{mydef}

\begin{figure}[ht]
\centering
\begin{tikzpicture}[
    node distance=1.5cm and 2.2cm,
    box/.style={draw, rounded corners, minimum width=3.2cm, minimum height=1cm, align=center},
    arr/.style={-{Latex}, thick},
]

\node[box] (alpha) {\(\alpha_\ell \iid \mu_\alpha\)};
\node[box, below=1cm of alpha] (u) {\(u_{\ell i} \iid \mu_u\)};
\node[box, below=2.5cm of u] (x) {\(x_{\ell ij} \iid \mu_x\)};
\node[box, below=1cm of x] (zeta) {\(\zeta_{\ell ij} \iid \text{subG}(\sigma^2)\)};

\node[box, right=1cm of u, yshift=-1.7cm] (mid) {\(w_{\ell i} = G[\alpha_\ell][u_{\ell i}]\)};

\node[box, right=5.2cm of alpha] (alpha_eval) {\(\bm{\alpha}_\ell = (\alpha_\ell(y_1), \dots, \alpha_\ell(y_{n_{c_W}}))\)};
\node[box, right=5.2cm of u] (u_eval) {\(\bm{u}_{\ell i} = (u_{\ell i}(c_1), \dots, u_{\ell i}(c_{n_{c_U}}))\)};
\node[box, right=5.2cm of zeta] (w) {\(w_{\ell ij} = G[\alpha_\ell][u_{\ell i}](x_{\ell ij}) + \zeta_{\ell ij}\)};

\draw[arr,draw = blue] (alpha) -- (alpha_eval);
\draw[arr,draw = blue] (u) -- (u_eval);
\draw[arr] (alpha.south east) -- ([yshift=2pt]mid.west);
\draw[arr] (u) -- (mid.west);
\draw[arr,draw = blue] (mid) -- ([yshift=5pt]w.west);
\draw[arr] (x) -- ([yshift=3pt]w.west);
\draw[arr] (zeta) -- (w.west);

\node[above=0.1cm of alpha] {\textbf{Parametric function}};
\node[above=0.1cm of u] {\textbf{Input function}};
\node[above=0.1cm of x] {\textbf{Evaluation points}};
\node[above=0.1cm of mid] {\textbf{Output function}};
\node[above=0.1cm of w] {\textbf{Noisy output samples}};
\node[above=0.1cm of alpha_eval] {\textbf{Sampled parametric function}};
\node[above=0.1cm of u_eval] {\textbf{Sampled input function}};
\node[above=0.1cm of zeta] {\textbf{Noise}};

\end{tikzpicture}
\caption{
Schematic structure of the training dataset $S_{G, \{y_s\}, \{c_s\}}$ defined in Definition~\ref{def:trainingSet} for multiple operator learning. supG($\sigma^2$) denotes a sub-Gaussian distribution with variance proxy $\sigma^2$.
Blue arrows indicate discretization steps; black arrows represent the flow of data. 
}
\label{fig:training-set-diagram}
\end{figure}
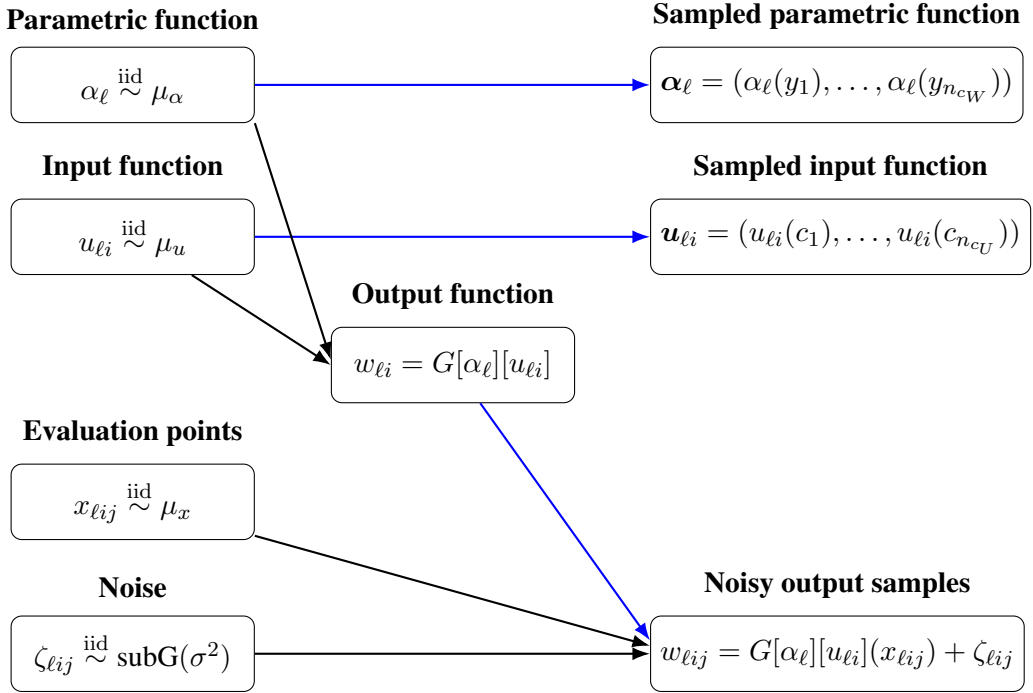

Next, we define the trained operator given a dataset $S_{G, \{y_s\}, \{c_s\}}$. In particular, the latter is a neural network chosen from the class \(\mathrm{Cl}_a(I,\cF_1,\cF_2,\cF_3,\{y_s\},\{c_s\})\) and obtained by minimizing a $\Lp{2}$ empirical loss over the training set.

\begin{mydef}[Trained operator] \label{def:trainedOperator}
    Let $\cF_i$ for $1 \leq i \leq 3$ be network classes defined in Definition \ref{def:networkClass}. Let $G : W \mapsto \{G[\alpha] : U \to V\}$ be a map. Let $\mu_\alpha$ be a probability measure on $W$, $\mu_u$ a probability measure on $U$, and $\mu_x$ a probability measure on $\Omega_V$. Given fixed sampling points $\{y_s\}_{s=1}^{n_{c_W}} \subset \Omega_W$ and $\{c_s\}_{s=1}^{n_{c_U}} \subset \Omega_U$, let $S_{G, \{y_s\}, \{c_s\}}$ be the training set defined in Definition \ref{def:trainingSet}. For $a,I >0$, the trained $a$-clipped operator $G_{a,I,\cF_1,\cF_2,\cF_3,S}$ is defined as \[
    G_{a,I,\cF_1,\cF_2,\cF_3,S} = \argmin_{\nn \in \mathrm{Cl}_a(I, \cF_1,\cF_2,\cF_3,\{y_s\}, \{c_s\})} \frac{1}{n_\alpha n_u n_x} \sum_{\ell=1}^{n_\alpha} \sum_{i=1}^{n_u} \sum_{j = 1}^{n_x} \l \nn[\bm{\alpha}_\ell][\bm{u}_{\ell i}](x_{\ell ij}) - w_{\ell i j}  \r^2
    \]
    where $\mathrm{Cl}_a(I, \cF_1,\cF_2,\cF_3,\{y_s\},\{c_s\})$ is defined in Definition \ref{def:clippedClass}.
\end{mydef}

Subsequently, we introduce the expected generalization error of the learned operator. The following quantity measures the expected performance of \(G_{a,I,\cF_1,\cF_2,\cF_3,S}\) by averaging over the randomness in the training set S and over unseen test inputs $(\alpha,u,x)$. It quantifies how well a model trained on one realization of the dataset generalizes to unseen data.

\begin{mydef}[Expected generalization error]
Let $\cF_i$ for $1 \leq i \leq 3$ be network classes defined in Definition \ref{def:networkClass}. Let $G : W \mapsto \{G[\alpha] : U \to V\}$ be a map. Let $\mu_\alpha$ be a probability measure on $W$, $\mu_u$ a probability measure on $U$, and $\mu_x$ a probability measure on $\Omega_V$. Let $\{y_s\}_{s=1}^{n_{c_W}} \subset \Omega_W$ and $\{c_s\}_{s=1}^{n_{c_U}} \subset \Omega_U$ be fixed sampling points. We define the expected generalization error as
\[
\mathbb{E}_{S_{G, \{y_s\}, \{c_s\}}} \underbrace{\mathbb{E}_{\alpha \sim \mu_\alpha} \, \mathbb{E}_{u \sim \mu_u} \, \mathbb{E}_{\{x_j\}_{j=1}^{n_x} \sim \mu_x^{\otimes n_x}}}_{\textnormal{test sampling}} \underbrace{\left[
\frac{1}{n_x} \sum_{j=1}^{n_x} \left( G_{a,I,\cF_1,\cF_2,\cF_3,S}[\bm{\alpha}][\ub](x_j) - G[\alpha][u](x_j) \right)^2
\right]}_{\textnormal{empirical approximation of the squared $\Lp{2}(\mu_x)$ error}},
\]
where \(\mathbb{E}_{S_{G, \{y_s\}, \{c_s\}}}\) denotes the expectation over the full training dataset \(S_{G, \{y_s\}, \{c_s\}}\) defined in Definition \ref{def:trainingSet} (i.e.\ over all i.i.d. draws), $G_{a,I,\cF_1,\cF_2,\cF_3,S}$ is the $a$-clipped trained operator defined in Definition \ref{def:trainedOperator}, $\bm{\alpha}=(\alpha(y_1), \alpha(y_2),...,\alpha(y_{n_{c_W}}))^\top$ is a discretization of $\alpha$ and $\ub=(u(c_1), u(c_2),...,u(c_{n_{c_U}}))^\top$ is a discretization of $u$.

\end{mydef}

We now state our main scaling law bound for the expected generalization error. Fix a target approximation accuracy \(\varepsilon>0\), and instantiate the multiple operator network hypothesis class \[\mathrm{Cl}_a(I,\cF_1,\cF_2,\cF_3,\{y_s\},\{c_s\},P^{n_{c_W}},H^{n_{c_U}},N^{d_V})\] using the \(\varepsilon\)-dependent architectural scalings prescribed by the approximation theory (so that the class contains an \(\varepsilon\)-accurate approximant to \(G\)). Let \(\eta>0\) denote the covering scale used to control the complexity of this class via \(\log \cN(\eta,\mathrm{Cl}_a,\|\cdot\|_{L^\infty(W\times U\times\Omega_V)})\). The following theorem combines these ingredients to yield an explicit approximation--estimation tradeoff linking target accuracy, architecture size, and the sampling budgets \((n_\alpha,n_u,n_x)\).

\begin{theorem}[Scaling laws for the expected generalization error] \label{thm:scalingLawsGeneralizationError}
Let $d_W,d_U,d_V>0$ be integers, \[
\gamma_W, \gamma_U, \gamma_V, \beta_W, \beta_U,\beta_V,L_W,L_U,L_V,L_G,L_{\mathcal{G}} > 0 \qquad \text{and} \qquad r_G, r_{\mathcal{G}} \geq 1\]
and assume that $W(d_W,\gamma_W,L_W,\beta_W)$,
 $U(d_U,\gamma_U,L_U,\beta_U)$ and $V(d_V,\gamma_V,L_V,\beta_V)$ satisfy Assumption \ref{assumption:Main:assumptions:S4}. 
 Let $G$ be a map such that \begin{align}
 	&G:\{ \alpha:\Omega_W \mapsto \bbR \spaceBar \Vert \alpha \Vert_{\Lp{\infty}} \leq \beta_W \} \mapsto \mathcal{G} \qquad \text{where } \notag \\
 	&\mathcal{G} = \Big\{ G[\alpha]\spaceBar G[\alpha]:\{ u:\Omega_U \mapsto \bbR \spaceBar \Vert u \Vert_{\Lp{\infty}} \leq \beta_U \} \mapsto V \text{ and } \notag \\
 	&\quad \Vert G[\alpha][u_1] - G[\alpha][u_2]\Vert_{\Lp{\infty}(\Omega_V)} \leq L_{\mathcal{G}}\Vert u_1 - u_2 \Vert_{\Lp{r_\mathcal{G}}(\Omega_U)} \Big\} \notag  
 \end{align}
Furthermore, assume that $G$ satisfies \begin{equation*} 
   \Vert G(\alpha_1) - G(\alpha_2) \Vert_{\Lp{\infty}(\{ u:\Omega_U \mapsto \bbR \spaceBar \Vert u \Vert_{\Lp{\infty}} \leq \beta_U \} \times \Omega_V)} \leq L_G \Vert \alpha_1 - \alpha_2 \Vert_{\Lp{r_G}(\Omega_W)}  
\end{equation*}
for $\alpha_1,\alpha_2 \in \{ \alpha:\Omega_W \mapsto \bbR \spaceBar \Vert \alpha \Vert_{\Lp{\infty}} \leq \beta_W \}$.
There exists constants $C$ depending on $\gamma_V,L_V$, $C_{\delta}$ depending on $L_{\mathcal{G}},d_U,\gamma_U,r_{\mathcal{G}},L_U$, $C'$ depending on $\beta_U, L_{\mathcal{G}}, d_U, \gamma_U, r_{\mathcal{G}}$, $C_\zeta$ depending on $L_G, d_W, \gamma_W, r_G,L_W$ and $C''$ depending on $\beta_W, L_G, d_W, \gamma_W, r_G$ such that the following holds. For any $\varepsilon>0$, 
\begin{itemize}
     
     \item let $N=  2^{2n_{c_W} + 3} C \sqrt{d_V} (C'' \sqrt{n_{c_W}})^{n_{c_W}} \varepsilon^{-(n_{c_W}+1)}$ and consider the network class \newline $\cF_1=\cF_{\rm NN}(d_V,1,L_1,p_1,K_1,\kappa_1,R_1)$  with parameters scaling as
    \begin{align*}
    &L_1 = \mathcal{O}\left(d_V^2\log d_V+d_V^2(n_{c_W}+1)\log(\varepsilon^{-1}) +d_V^2\log(2^{n_{c_W}+1}(C''\sqrt{n_{c_W}})^{n_{c_W}})\right) ,\quad  p_1 = \mathcal{O}(1),\\
    &K_1 = \mathcal{O}\left(d_V^2\log d_V+d_V^2(n_{c_W}+1)\log(\varepsilon^{-1}) +d_V^2\log(2^{n_{c_W}+1}(C''\sqrt{n_{c_W}})^{n_{c_W}})\right)\\
&\kappa_1=\mathcal{O}(2^{(d_V+1)(n_{c_W}+1)}d_V^{d_V/2+1}\varepsilon^{-(d_V+1)(n_{c_W}+1)} \ls 2^{n_{c_W}+\awmath{2}}(C''\sqrt{n_{c_W}})^{n_{c_W}}\rs^{(d_V+1)}),\qquad R_1=1
    \end{align*}
    where the constants hidden in $\mathcal{O}$ depend on $\gamma_V$ and $L_V$;

    \item let $\delta=\frac{C_{\delta}\varepsilon^{(1+d_V)(1+n_{c_W})}}{2^{2d_V+2n_{c_W}+3+d_Vn_{c_W}}(C\sqrt{d_V})^{d_V}(C''\sqrt{n_{c_W}})^{n_{c_W}}}$ and let $\{c_m\}_{m=1}^{n_{c_U}}\subset \Omega_U$ be points so that $\{\mathcal{B}_{\delta}(c_m) \}_{ m  = 1}^{n_{c_U}}$ is a cover of $\Omega_U$ for some $n_{c_U}$;
    
    \item let $H = 2^{3 + 2n_{c_W} + 3d_V + 2d_Vn_{c_W}}C' \sqrt{n_{c_U}} (C \sqrt{d_V})^{d_V} (C'' \sqrt{n_{c_W}})^{n_{c_W}(d_V +1)} \eps^{-(d_V+1)(1+n_{c_W})}$ and consider the network class $\cF_2=\cF_{\rm NN}(n_{c_U},1,L_2,p_2,K_2,\kappa_2,R_2)$ with parameters scaling as
\begin{align*}
        &L_2 = \mathcal{O}\big(n_{c_U}^2\log n_{c_U}+n_{c_U}^2(d_V +1)(n_{c_W}+1)\log(\varepsilon^{-1}) + n_{c_U}^2 \log(2^{d_V +1} (C \sqrt{d_V})^{d_V} )\\
        &\qquad + n_{c_U}^2(d_V +1)\log(2^{n_{c_W}+1}(C''\sqrt{n_{c_W}})^{n_{c_W}})\big),\quad p_2=\mathcal{O}(1), \\
&K_2 = \mathcal{O}\big(n_{c_U}^2\log n_{c_U}+n_{c_U}^2(d_V +1)(n_{c_W}+1)\log(\varepsilon^{-1}) + n_{c_U}^2 \log(2^{d_V +1} (C \sqrt{d_V})^{d_V} )\\
        &\qquad + n_{c_U}^2(d_V +1)\log(2^{n_{c_W}+1}(C''\sqrt{n_{c_W}})^{n_{c_W}})\big), \\ 
&\kappa_2=\mathcal{O}(n_{c_U}^{n_{c_U}/2+1}[\varepsilon/2]^{-(d_V+1)(n_{c_U}+1)(n_{c_{W}}+1)}[ 2^{d_V +2} (C \sqrt{d_V})^{d_V} ]^{n_{c_U}+1} \ls 2^{d_V +1} (C \sqrt{d_V})^{d_V} \rs^{(d_V+1)(n_{c_U}+1)}), \\
&R_2=1
    \end{align*}
    where the constants hidden in $\mathcal{O}$ depend on $\beta_U, L_{\mathcal{G}}, d_U, \gamma_U,r_{\mathcal{G}}$;

\item let $\zeta=C_{\zeta}\varepsilon/2$ and let $\{y_m\}_{m=1}^{n_{c_W}}\subset \Omega_W$ be points so that $\{\mathcal{B}_{\zeta}(y_m) \}_{ m = 1}^{n_{c_W}}$ is a cover of $\Omega_W$ for some $n_{c_W}$;

\item let $P = 4C'' \sqrt{n_{c_W}} \eps^{-1}$ and consider the network class $\cF_3=\cF_{\rm NN}(n_{c_W},1,L_3,p_3,K_3,\kappa_3,R_3)$ with parameters scaling as
\begin{align*}
       &L_3=\mathcal{O}\left(n_{c_W}^2\log(n_{c_W})+n_{c_W}^2\log(\varepsilon^{-1})\right),\quad  p_3 = \mathcal{O}(1),\quad K_3 = \mathcal{O}\left(n_{c_W}^2\log n_{c_W}+n_{c_W}^2\log(\varepsilon^{-1})\right), \\ &\kappa_3=\mathcal{O}(n_{c_W}^{n_{c_W}/2+1}2^{n_{c_W}+1}\varepsilon^{-n_{c_W}-1} 2^{n_{c_W}+1}),\qquad \, R_3=1
    \end{align*}
    where the constants hidden in $\mathcal{O}$ depend on $\beta_W, L_G, d_W, \gamma_W,r_G$.
\end{itemize}
Let $a = \beta_V$, $I \geq \beta_V$, $n_\alpha,n_u,n_x \in \bbN$,  $\mu_\alpha$ a probability measure on $W$, $\mu_u$ a probability measure on $U$, and $\mu_x$ a probability measure on $\Omega_V$. Consider the clipped network class \[
\mathrm{Cl}_a(I,\cF_1,\cF_2,\cF_3,\{y_s\},\{c_s\},P^{n_{c_W}},H^{n_{c_U}},N^{d_V}).
\] 

For $\eta > 0$, the expected generalization error is bounded as follows: \begin{align}
    &\mathbb{E}_{S_{G, \{y_s\}, \{c_s\}}} \mathbb{E}_{\alpha \sim \mu_\alpha} \, \mathbb{E}_{u \sim \mu_u} \, \mathbb{E}_{\{x_j\}_{j=1}^{n_x} \sim \mu_x^{\otimes n_x}} \left[
\frac{1}{n_x} \sum_{j=1}^{n_x} \left( G_{a,I,\cF_1,\cF_2,\cF_3,S}[\bm{\alpha}][\ub](x_j) - G[\alpha][u](x_j) \right)^2
\right] \notag  \\
        &\leq 4 \eps^2 + \eta (8 \sigma + 6) \notag \\
    &+ \frac{8\sigma \eta}{\sqrt{n_\alpha n_u n_x}} \sqrt{\log \l \cN\l\eta,\mathrm{Cl}_a(I,\cF_1,\cF_2,\cF_3,\{y_s\},\{c_s\},P^{n_{c_W}},H^{n_{c_U}},N^{d_V}),\Vert \cdot \Vert_{\Lp{\infty}(W \times U \times \Omega_V)}\r \r + \log(2)} \notag \\
    &+ \frac{16\sigma^2}{n_\alpha n_u n_x} \l \log \l \cN\l\eta,\mathrm{Cl}_a(I,\cF_1,\cF_2,\cF_3,\{y_s\},\{c_s\},P^{n_{c_W}},H^{n_{c_U}},N^{d_V}),\Vert \cdot \Vert_{\Lp{\infty}(W \times U \times \Omega_V)}\r \r + \log(2) \r \notag \\
    &+ \frac{112 \beta_V^2}{3 n_\alpha} \log\l   \cN\l\eta/(4\beta_V),\mathrm{Cl}_a(I,\cF_1,\cF_2,\cF_3,\{y_s\},\{c_s\},P^{n_{c_W}},H^{n_{c_U}},N^{d_V}),\Vert \cdot \Vert_{\Lp{\infty}(W \times U \times \Omega_V)}\r \r \notag 
    \end{align}
where $\bm{\alpha}=(\alpha(y_1), \alpha(y_2),...,\alpha(y_{n_{c_W}}))^\top$ is a discretization of $\alpha$ and $\ub=(u(c_1), u(c_2),...,u(c_{n_{c_U}}))^\top$ is a discretization of $u$.
\end{theorem}

In Theorem~\ref{thm:scalingLawsGeneralizationError}, the scaling laws depend on the best-in-class approximation error of the hypothesis class \(\mathrm{Cl}_a(I,\cF_1,\cF_2,\cF_3,\{y_s\},\{c_s\},P^{n_{c_W}},H^{n_{c_U}},N^{d_V})\), denoted by \(\varepsilon\), where the architectural parameters of the component network classes \(\cF_i\) are instantiated as \(\varepsilon\)-dependent choices. The next corollary controls the resulting metric entropy of this \(\varepsilon\)-dependent hypothesis class explicitly as a function of \(\varepsilon\).

\begin{corollary}[Metric entropy bound under $\varepsilon$-dependent MNO scaling] \label{cor:covering:eps}
    Assume the same setting as in Theorem \ref{thm:scalingLawsGeneralizationError}. Then, \begin{align}
    &\log \l \cN\l\eta,\mathrm{Cl}_a(I,\cF_1,\cF_2,\cF_3,\{y_s\},\{c_s\},P^{n_{c_W}},H^{n_{c_U}},N^{d_V}),\Vert \cdot \Vert_{\Lp{\infty}(W \times U \times \Omega_V)}\r \r \notag \\
    &\lesssim \eps^{-\delta_1 \eps^{-\delta_2\eps^{-d_W}}} \l 1 + \log(\eta^{-1}) \r \notag
\end{align}
where $\delta_1 =  d_U(1 + d_V)\l 1 + \frac{d_W}{2}\r + d_W\frac{(d_V+1)}{2} + (d_V+1) $ and $\delta_2 =d_U(1 + d_V)\l 1 + \frac{d_W}{2}\r $.
\end{corollary}

\begin{remark}[Effective parameter scaling induced by the \(\varepsilon\)-dependent MNO construction] \label{rem:effectiveScaling}
The \(\varepsilon\)-dependence of the metric entropy term in Corollary~\ref{cor:covering:eps} is driven by the growth in (i) the number of subnetworks in the separable expansion and (ii) the admissible parameter magnitudes. In particular, it depends on the product-structure multiplicities through
\[
\max\{N^{d_V},\,H^{n_{c_U}},\,P^{n_{c_W}}\}
\qquad\text{and on the parameter bounds through}\qquad
\max_{i\in\{1,2,3\}}\kappa_i.
\] As shown in \cite[Remark 3.17]{weihs2025MOL}, this growth is equivalent to the scaling of the total number of non-zero parameters \(N_\#\) required by the resulting MNO-type architecture: specifically, we have 
\[
N_\# \;\lesssim\; \varepsilon^{-\gamma_1\,\varepsilon^{-\gamma_2\,\varepsilon^{-d_W}}},
\]
for constants \(\gamma_1,\gamma_2>0\) depending only $d_W, d_U, d_V$. 
\end{remark}

Theorem \ref{thm:scalingLawsGeneralizationError} yields a family of bounds parameterized by the approximation accuracy \(\varepsilon\) and covering scale \(\eta\); using Corollary \ref{cor:covering:eps}, in the next result, we select \(\varepsilon=\varepsilon(n_\alpha)\) and \(\eta=\eta(n_\alpha)\) to balance the approximation and estimation terms and thereby extract a single explicit learning rate as a function of the number of operator samples $n_\alpha$.

\begin{corollary}[Generalization rate in the number of sampled operators $n_\alpha$] \label{cor:boundEps}
    Assume the same setting as in Theorem \ref{thm:scalingLawsGeneralizationError}. If we pick \[
    \varepsilon = 
    \left(
        \frac{d_W}{2\delta_2}
        \frac{\log\log(n_\alpha)}{\log\log\log(n_\alpha)}
    \right)^{-\frac{1}{d_W}} \qquad \text{and} \qquad \eta = 4\beta_V n_\alpha^{-1}
    \] 
    with $\delta_2 = d_U(1 + d_V)\l 1 + \frac{d_W}{2}\r$, then the expected generalization error scales as follows:
    \begin{align}
        &\mathbb{E}_{S_{G, \{y_s\}, \{c_s\}}} \mathbb{E}_{\alpha \sim \mu_\alpha} \, \mathbb{E}_{u \sim \mu_u} \, \mathbb{E}_{\{x_j\}_{j=1}^{n_x} \sim \mu_x^{\otimes n_x}} \left[
\frac{1}{n_x} \sum_{j=1}^{n_x} \left( G_{a,I,\cF_1,\cF_2,\cF_3,S}[\bm{\alpha}][\ub](x_j) - G[\alpha][u](x_j) \right)^2
\right] \notag \\
&= \mathcal{O} \l \left(
        \frac{\log\log( n_\alpha)}{\log\log\log( n_\alpha)}
    \right)^{-\frac{2}{d_W}} \r \notag
    \end{align}
where the constants hidden in $\mathcal{O}$ are independent of $n_\alpha, n_u, n_x$.
\end{corollary}

\subsection{Proof Sketch}\label{sec:proofSketch}

We briefly summarize the main ideas behind the proofs of Theorem~\ref{thm:scalingLawsGeneralizationError} and Corollaries \ref{cor:covering:eps}, \ref{cor:boundEps}.
The arguments follow the standard approximation--estimation paradigm, but the multiple operator setting (hierarchical sampling and data-dependent predictors) requires a careful organization of the stochastic terms.

\paragraph{Notation and risk decomposition.}
We recall that the learned operator network
\(
\widehat G_S := G_{a,I,\cF_1,\cF_2,\cF_3,S}
\)
depends on the (random) training sample \(S:=S_{G,\{y_s\},\{c_s\}}\).
For any candidate operator \(F\) (e.g.\ a network in our hypothesis class), define the population and empirical risks
\begin{align}
\cR(F)
&:= \bbE_{\alpha\sim\mu_\alpha}\bbE_{u\sim\mu_u}\bbE_{x\sim\mu_x}
\Bigl[\bigl(F[\alpha][u](x)-G[\alpha][u](x)\bigr)^2\Bigr],  \notag\\
\widehat{\cR}_S(F)
&:= \frac{1}{n_\alpha n_u n_x}\sum_{\ell=1}^{n_\alpha}\sum_{i=1}^{n_u}\sum_{j=1}^{n_x}
\Bigl(\,F[\alpha_\ell][u_{\ell i}](x_{\ell ij})-G[\alpha_\ell][u_{\ell i}](x_{\ell ij})\,\Bigr)^2. \notag
\end{align}
With this notation, the expected generalization error is
\begin{equation}
T_0 \;=\; \bbE_{S}\,\cR(\widehat G_S). \notag
\end{equation}
Adding and subtracting the expected empirical risk of the learned predictor yields the decomposition
\begin{align}
T_0
&= 2\,\bbE_S\widehat{\cR}_S(\widehat G_S)
\;+\;\Bigl(\bbE_S\cR(\widehat G_S)-2\,\bbE_S\widehat{\cR}_S(\widehat G_S)\Bigr)
\nonumber\\
&=: T_1 + T_2, \notag
\end{align}
which matches~\eqref{eq:scalingLawsGeneralizationError:eq1} (the factor \(2\) is a bookkeeping constant).

\paragraph{Step 1: Control of \(T_1\) (approximation + stochastic cross-term)}
The term \(T_1=2\bbE_S\widehat{\cR}_S(\widehat G_S)\) is an expected training-error quantity.
Using the observation model \(w=G[\alpha][u](x)+\zeta\) from Definition~\ref{def:trainingSet}, we expand the square and split \(T_1\) into
(i) an approximation component and (ii) a noise-driven cross-term; see \eqref{eq:scalingLawsGeneralizationError:eq2}--\eqref{eq:scalingLawsGeneralizationError:T4T5}.
The approximation component is controlled by the expressivity result for the clipped class:
by Corollary~\ref{cor:back:clippedScalingLaws} (together with Remark~\ref{rem:eps2}) there exists a network
\(\nn\in\mathrm{Cl}_a(I,\cF_1,\cF_2,\cF_3,\{y_s\},\{c_s\},P^{n_{c_W}},H^{n_{c_U}},N^{d_V})\)
such that \(\sup_{\alpha,u,x}|\nn[\alpha][u](x)-G[\alpha][u](x)|\le \varepsilon\), yielding a contribution of order \(\varepsilon^2\) in \eqref{eq:scalingLawsGeneralizationError:eq11}.
The remaining cross-term involves \((\widehat G_S-G)\zeta\) and cannot be averaged out directly because \(\widehat G_S\) depends on \(S\) (and hence on \(\zeta\)).
We control this term by discretizing the relevant function class by an \(\eta\)-net and applying a moment generating function estimate together with a union bound, producing an estimation contribution involving \(\log \cN(\eta)\) and the sample sizes \((n_\alpha,n_u,n_x)\); see Lemma~\ref{lem:main:T5Bound}.
This is the origin of the additive \(\eta\) term (net discretization) and the metric-entropy terms in the bound \eqref{eq:scalingLawsGeneralizationError:T1} for \(T_1\). 

\paragraph{Step 2: Control of \(T_2\) (generalization gap via symmetrization and discretization)}
The term \(T_2=\bbE_S\cR(\widehat G_S)-2\bbE_S\widehat{\cR}_S(\widehat G_S)\) is the generalization-gap contribution associated with the data-dependent predictor.
To control it, we introduce an independent ghost sample \(S'\) (symmetrization), which rewrites population quantities as expectations of independent empirical averages and reduces the problem to bounding a supremum of a symmetrized (shifted) process over a shifted clipped class; see \eqref{eq:T2bound:eq1}--\eqref{eq:T2bound:eq5}.
We then discretize this shifted class by an \(\eta\)-net (constructed in Lemma~\ref{lem:coveringShifted}), reducing the supremum to a maximum over finitely many functions plus an additive discretization error of order \(\eta\); see \eqref{eq:T2bound:eq6}--\eqref{eq:T2bound:eq10}.
Finally, we bound the resulting maximum by a moment generating function estimate and a union bound over the finite cover, producing a term proportional to \(\log \cN(\eta)\) with explicit sample-size factors; see \eqref{eq:T2bound:eq12}--\eqref{eq:T2bound:eq23}.

\paragraph{Step 3: Covering numbers and corollaries}
The proof of Theorem~\ref{thm:scalingLawsGeneralizationError} is complete once the bounds for \(T_1\) and \(T_2\) from Steps~1--2 are combined, yielding the stated scaling-law estimate in terms of the metric entropy \(\log \cN(\eta)\).
To obtain explicit consequences, we then invoke Proposition~\ref{prop:back:coveringClippledMultipleOperator} to bound \(\log \cN(\eta)\) for the clipped hypothesis class, and subsequently choose the free parameters \(\eta\) and \(\varepsilon\) as functions of the sampling budgets to optimize/simplify the bound. This yields Corollary~\ref{cor:covering:eps} (metric entropy as a function of \(\varepsilon\)) and Corollary~\ref{cor:boundEps} (an explicit generalization rate after balancing \(\varepsilon\) and \(\eta\) in terms of \(n_\alpha\)).

\section{Proofs} \label{sec:proofs}

In this section, we present detailed proofs of our results. 

\subsection{Proofs of the Background Results}

\subsubsection{Scaling laws for clipped networks}

\begin{proof}[Proof of Corollary \ref{cor:back:clippedScalingLaws}]
    We note that \begin{align}
       E &:= \sup_{\alpha \in W} \sup_{u \in U} \sup_{x \in \Omega_V}  \left\vert G[\alpha][u](x) - \operatorname{Clip}_{\beta_V} \l \nn[\alpha][u](x) \r \right\vert \notag \\
       &\leq \sup_{\alpha \in W} \sup_{u \in U} \sup_{x \in \Omega_V} \left\vert G[\alpha][u](x) - \nn[\alpha][u](x) \right\vert \notag \\
       &+ \sup_{\alpha \in W} \sup_{u \in U} \sup_{x \in \Omega_V} \left\vert \nn[\alpha][u](x) - \operatorname{Clip}_{\beta_V} \l \nn[\alpha][u](x) \r \right\vert \notag \\
       &\leq \eps + \sup_{\alpha \in W} \sup_{u \in U} \sup_{x \in \Omega_V} \left\vert \nn[\alpha][u](x) - \operatorname{Clip}_{\beta_V} \l \nn[\alpha][u](x) \r \right\vert \label{eq:clippledScalingLaws:eq1} \\
       &=: \frac{\eps}{2} + T_1 \notag
    \end{align}
    where we used \eqref{eq:main:multipleOperatorApproximation} for \eqref{eq:clippledScalingLaws:eq1}.
    
For a fixed $\alpha \in W$, $u \in U$ and $x \in \Omega_V$, if $\vert \nn[\alpha][u](x) \vert \leq \beta_V$, then $T_1 = 0$ by definition and $E \leq \eps/2$. Else, we first suppose that $\nn[\alpha][u](x) > \beta_V$. This implies that $\operatorname{Clip}_{\beta_V} \l \nn[\alpha][u](x) \r = \beta_V$ and therefore, \begin{align}
     T_1 &= \nn[\alpha][u](x) - \beta_V \notag \\
     &\leq \nn[\alpha][u](x) - G[\alpha][u](x) \label{eq:clippledScalingLaws:eq2} \\
     &\leq \vert \nn[\alpha][u](x) - G[\alpha][u](x) \vert \notag \\
     &\leq \frac{\eps}{2} \label{eq:clippledScalingLaws:eq3}
\end{align}
where we used the fact that $G[\alpha][u](x) \leq \beta_V$ by Assumption \ref{assumption:Main:assumptions:S4} for \eqref{eq:clippledScalingLaws:eq2} and \eqref{eq:main:multipleOperatorApproximation} for \eqref{eq:clippledScalingLaws:eq3}. Similarly, if $\nn[\alpha][u](x) < -\beta_V$, then $\operatorname{Clip}_{\beta_V} \l \nn[\alpha][u](x) \r = -\beta_V$ and 
\begin{align}
     T_1 &= - \beta_V - \nn[\alpha][u](x)\notag \\
     &\leq G[\alpha][u](x) - \nn[\alpha][u](x) \label{eq:clippledScalingLaws:eq4} \\
     &\leq \vert \nn[\alpha][u](x) - G[\alpha][u](x) \vert \notag \\
     &\leq \frac{\eps}{2} \label{eq:clippledScalingLaws:eq5}
\end{align}
where we used the fact that $G[\alpha][u](x) \geq -\beta_V$ by Assumption \ref{assumption:Main:assumptions:S4} for \eqref{eq:clippledScalingLaws:eq4} and \eqref{eq:main:multipleOperatorApproximation} for \eqref{eq:clippledScalingLaws:eq5}. Combining \eqref{eq:clippledScalingLaws:eq3}, \eqref{eq:clippledScalingLaws:eq5} and \eqref{eq:clippledScalingLaws:eq1}, we conclude that $E \leq \eps$. 

\end{proof}

\subsubsection{Covering numbers for neural network classes}

The following argument follows a standard discretization approach: we first establish uniform \(L^\infty\) bounds on network outputs and on output differences in terms of the architectural parameters, and then quantize the admissible parameter set to construct an explicit \(\eta\)-net and count its cardinality.

\begin{proof}[Proof of Proposition \ref{prop:back:covering}] \begin{enumerate}
    \item We start by establishing a bound for the $\Lp{\infty}$-norm of the output of $$q \in \cF_{\nn}(d_1,1,L,p,K,\kappa,R),$$ i.e. we provide a value of $R$ as a function of all the other parameters. Specifically, we proceed by induction on the network depth to show that \begin{equation} \notag 
    \Vert q \Vert_{\Lp{\infty}} \leq \kappa^k(p+1)^{k-1}(p\Vert x \Vert_{\Lp{\infty}} + 1)
\end{equation}
when $L = k$.

We recall the following bound for $A \in \bbR^{m \times n}$ and $ x \in \bbR^n$:\begin{equation} \label{eq:prop:covering:matrixBound}
    \Vert A x \Vert_{\Lp{\infty}} = \max_{1 \leq i \leq m} \left\vert \sum_{j=1}^n [A]_{ij}[x]_j \right\vert \leq n \Vert x \Vert_{\Lp{\infty}} \Vert A \Vert_{\infty,\infty}.
\end{equation}

\paragraph{Base case: $L=1$} The network we consider is $q(x) = W_1 \mathrm{ReLU}(x) + b_1$. Then, we have: \begin{align}
    \Vert q \Vert_{\Lp{\infty}} &\leq \Vert W_1  \mathrm{ReLU}(x) \Vert_{\Lp{\infty}} + \Vert b_1 \Vert_{\infty} \notag \\
    &\leq p \Vert W_1 \Vert_{\infty,\infty} \Vert \mathrm{ReLU}(x) \Vert_{\Lp{\infty}} + \Vert b_1 \Vert_{\infty} \label{eq:prop:covering:eq1} \\
    &\leq p \kappa \Vert x \Vert_{\Lp{\infty}} + \kappa  \label{eq:prop:covering:eq2}
\end{align}
where we used \eqref{eq:prop:covering:matrixBound} and the fact that the width of network (and hence the dimensions of the matrix $W_1$) is bounded by $p$ for \eqref{eq:prop:covering:eq1} as well as the facts that $\max\{\Vert W_1 \Vert_{\infty,\infty},\Vert b_1 \Vert_\infty\} \leq \kappa$ and $\mathrm{ReLU}(x) \leq x$ for \eqref{eq:prop:covering:eq2}.

\paragraph{Induction step: $L=k+1$} Suppose that $\Vert q \Vert_{\Lp{\infty}} \leq \kappa^k(p+1)^{k-1}(p\Vert x \Vert_{\Lp{\infty}} + 1)$ for $L = k$. For $L = k+1$, we write $q(x) = W_{L+1}\mathrm{ReLU}(\tilde{q}(x)) + b_{L+1}$ where $\tilde{q}$ is a feedforward ReLU network of depth $k$. We estimate as follows: \begin{align}
    \Vert q \Vert_{\Lp{\infty}} & \leq p \Vert W_{L+1} \Vert_{\infty} \Vert \mathrm{ReLU}(\tilde{q}(x)) \Vert_{\Lp{\infty}} + \Vert b_{L+1} \Vert_{\infty} \label{eq:prop:covering:eq3} \\
    &\leq p\kappa \Vert \tilde{q}(x) \Vert_{\Lp{\infty}} + \kappa \label{eq:prop:covering:eq4} \\
    &\leq p \kappa \kappa^k(p+1)^{k-1}(p\Vert x \Vert_{\Lp{\infty}} + 1) + \kappa \label{eq:prop:covering:eq5} \\
    &\leq p\kappa^{k+1} (p+1)^{k-1}(p\Vert x \Vert_{\Lp{\infty}} + 1) + \kappa^{k+1} (p+1)^{k-1}(p\Vert x \Vert_{\Lp{\infty}} + 1) \label{eq:prop:covering:eq6} \\
    &= \kappa^{k+1}(p+1)^k(p\Vert x \Vert_{\Lp{\infty}} + 1) \notag
\end{align}
where we used \eqref{eq:prop:covering:matrixBound} and the fact that the width of the network is bounded by $p$ for \eqref{eq:prop:covering:eq3}, the facts that $$\max\{\Vert W_1 \Vert_{\infty,\infty},\Vert b_1 \Vert_\infty\} \leq \kappa$$ and $\mathrm{ReLU}(x) \leq x$ for \eqref{eq:prop:covering:eq4}, the induction hypothesis for \eqref{eq:prop:covering:eq5} and the assumption that $\kappa \geq 1$ for \eqref{eq:prop:covering:eq6}.

\item Next, we proceed by induction on the network depth to bound the $\Lp{\infty}$-norm of the difference between the output of two networks $q_1,q_2 \in \cF_{\nn}(d_1,1,L,p,K,\kappa,R)$. 
Specifically, we prove that \begin{equation} \notag
   \Vert q_1 - q_2 \Vert_{\Lp{\infty}} \leq k \kappa^{k-1}(p+1)^{k-1}(p\Vert x \Vert_{\Lp{\infty}}+1)d(q_1,q_2) 
\end{equation}
when $L = k$.

\paragraph{Base case: $L =1$} The networks we consider are $q_1(x) = W_1^{(1)}\mathrm{ReLU}(x) + b_1^{(1)}$ and $q_2(x) = W_1^{(2)}\mathrm{ReLU}(x) + b_1^{(2)}$. We have \begin{align}
    \Vert q_1 - q_2 \Vert_{\Lp{\infty}} &\leq \Vert (W_1^{(1)} - W_1^{(2)})\mathrm{ReLU}(x) \Vert_{\Lp{\infty}} + \Vert b_1^{(1)} - b_1^{(2)} \Vert_{\infty} \notag \\
    &\leq p \Vert W_1^{(1)} - W_1^{(2)} \Vert_{\infty,\infty} \Vert x \Vert_{\Lp{\infty}} + d(q_1,q_2) \label{eq:prop:covering:eq7} \\
    &\leq d(q_1,q_2)(p\Vert x \Vert_{\Lp{\infty}}) + 1) \notag
\end{align}
where we used \eqref{eq:prop:covering:matrixBound}, the fact that the width of the two network is bounded by $p$ and that $\mathrm{ReLU}(x) \leq x$  for \eqref{eq:prop:covering:eq7}.

\paragraph{Induction step: $L=k+1$} Suppose that $\Vert q_1 - q_2 \Vert_{\Lp{\infty}} \leq k \kappa^{k-1}(p+1)^{k-1}(p\Vert x \Vert_{\Lp{\infty}}+1)d(q_1,q_2)$ for $L = k$. For $L = k+1$, we write $q_i(x) = W_{L+1}^{(i)}\mathrm{ReLU}\tilde{q}_i(x) + b_{L+1}^{(i)}$ where $\tilde{q}_i$ is a feedforward ReLU network of depth $k$. We estimate as follows: \begin{align}
    &\Vert q_1 - q_2 \Vert_{\Lp{\infty}} \notag \\
    &\leq \Vert W_{L+1}^{(1)}\mathrm{ReLU}(\tilde{q}_1(x)) - W_{L+1}^{(2)}\mathrm{ReLU}(\tilde{q}_1(x)) - W_{L+1}^{(2)}\mathrm{ReLU}(\tilde{q}_2(x)) + W_{L+1}^{(2)}\mathrm{ReLU}(\tilde{q}_1(x))  \Vert_{\Lp{\infty}} \notag \\
    &+ \Vert b_{L+1}^{(1)} - b_{L+1}^{(2)} \Vert_\infty \notag \\
    &\leq p \Vert W_{L+1}^{(1)} - W_{L+1} \Vert_{\infty,\infty} \Vert \tilde{q}_1(x) \Vert_{\Lp{\infty}} + p\Vert W_{L+1}^{(2)} \Vert_{\infty,\infty} \Vert \mathrm{ReLU}(\tilde{q}_1(x)) - \mathrm{ReLU}(\tilde{q}_2(x)) \Vert_{\Lp{\infty}} + d(q_1,q_2) \label{eq:prop:covering:eq8} \\
    &\leq p d(q_1,q_2) \kappa^k (p+1)^{k-1}(p\Vert x \Vert_{\Lp{\infty}} + 1) + p \kappa \Vert \tilde{q}_1(x) - \tilde{q}_2(x) \Vert_{\Lp{\infty}} + d(q_1,q_2) \label{eq:prop:covering:eq9} \\
    &\leq p d(q_1,q_2) \kappa^k (p+1)^{k-1}(p\Vert x \Vert_{\Lp{\infty}} + 1) + p k \kappa^{k}(p+1)^{k-1}(p\Vert x \Vert_{\Lp{\infty}}+1)d(q_1,q_2) + d(q_1,q_2) \label{eq:prop:covering:eq10} \\
    &\leq p d(q_1,q_2) \kappa^k (p+1)^{k-1}(p\Vert x \Vert_{\Lp{\infty}} + 1) + p k \kappa^{k}(p+1)^{k-1}(p\Vert x \Vert_{\Lp{\infty}}+1)d(q_1,q_2) \notag \\
    &+ (k+1) d(q_1,q_2) \kappa^{k}(p+1)^{k-1}(p\Vert x \Vert_{\Lp{\infty}}+1)\label{eq:prop:covering:eq11} \\
    &= (k+1) \kappa^{k}(p+1)^{k}(p\Vert x \Vert_{\Lp{\infty}}+1)d(q_1,q_2) \label{eq:prop:covering:eq12}
\end{align}
where we used \eqref{eq:prop:covering:matrixBound} and the fact that the width of the two network is bounded by $p$ for \eqref{eq:prop:covering:eq8}, \eqref{eq:prop:covering:boundOutput} and the fact that ReLU is $1$-Lipschitz for \eqref{eq:prop:covering:eq9}, the induction hypothesis for \eqref{eq:prop:covering:eq10}, the fact that $\kappa \geq 1$ for \eqref{eq:prop:covering:eq11} and the identity $(k+1)(p+1) = pk + k + p + 1$ for \eqref{eq:prop:covering:eq12}.

\item Finally, we derive an upper bound on the covering number of $\cF_{\nn}(d_1,1,L,p,K,\kappa,R)$. For any \( q \in \cF_{\nn}(d_1,1,L,p,K,\kappa,R) \), the weight matrices and biases satisfy
\(
\max_{1 \leq \ell \leq L} \left\{ \Vert W_\ell \Vert_{\infty,\infty},\ \Vert b_\ell \Vert_\infty \right\} \leq \kappa,
\)
so each parameter lies in the interval \([-\kappa, \kappa]\).

For a fixed \( h > 0 \), we discretize the latter interval into \( 2\kappa/h \) subintervals of length \( h \), yielding \( \lfloor 2\kappa/h \rfloor + 1 \) grid points. Then, for any parameter value \( c \in [-\kappa, \kappa] \), there exists a grid point \( c^* \) such that \( |c - c^*| \leq h/2 \).

Now, given any \( q \in \cF_{\nn}(d_1,1,L,p,K,\kappa,R) \), let \( q^* \) denote the network where each nonzero parameter of \( q \) is replaced by its nearest grid point. By construction, this implies that \( d(q, q^*) \leq h/2 \), and it follows from \eqref{eq:prop:covering:boundDifferenceOutput} that
\[
\Vert q - q^* \Vert_{\Lp{\infty}} \leq L\, \kappa^{L-1}(p+1)^{L-1} \left( p\Vert x \Vert_{\Lp{\infty}} + 1 \right)  \frac{h}{2}.
\]

Thus, setting
\(
h = \frac{2\eta}{L \kappa^{L-1}(p+1)^{L-1} (p\Vert x \Vert_{\Lp{\infty}} + 1)},
\)
we conclude that the set of networks of the form $$W_L \cdot \mathrm{ReLU}\left(W_{L-1} \cdots \mathrm{ReLU}(W_1 x + b_1) + \cdots + b_{L-1} \right) + b_L,$$ i.e. a feedforward ReLU network with \( L \) layers and width \( p \), whose nonzero parameters are constrained to grid points forms a \( \eta \)-cover of \( \cF_{\nn}(d_1,1,L,p,K,\kappa,R) \) in the \( \Lp{\infty} \)-norm.

It remains to estimate the number of such networks. Since each \( q \in \cF_{\nn}(d_1,1,L,p,K,\kappa,R) \) has at most \( K \) nonzero parameters, it suffices to consider networks with parameters restricted to grid points and at most \( K \) nonzero parameters.

The total number of parameters of a feedforward ReLU network with \( L \) layers and width \( p \), is at most \( L(p^2 + p) \), where \( p^2 \) corresponds to weights and \( p \) to biases per layer. Therefore, the number of possible sparsity patterns is bounded by
\[
\binom{L(p^2 + p)}{K},
\]
and for each such pattern, there are \( \left\lfloor 2\kappa/h \right\rfloor + 1 \) choices per nonzero coordinate. Hence, the total number of distinct networks is at most
\[
\binom{L(p^2 + p)}{K} \left( \left\lfloor \frac{2\kappa}{h} \right\rfloor + 1 \right)^K = \binom{L(p^2 + p)}{K} \left( \left\lfloor \frac{L \kappa^{L}(p+1)^{L-1} (p\Vert x \Vert_{\Lp{\infty}} + 1)}{\eta} \right\rfloor + 1 \right)^K.
\]

\end{enumerate}

\end{proof}

\subsection{Proof of the Main Results}

\subsubsection{Covering numbers for product network classes}

The proof follows the same discretization strategy as Proposition~\ref{prop:back:covering}, now applied to the separable
multiple operator architecture, which is a linear combination of products of three subnetworks. In particular, in addition
to quantizing the parameters of the subnetworks \(l_p\), \(b_k\), and \(\tau_\ell\), we must also quantize the coefficient
array \(\{\theta_{pk\ell}\}\subset[-I,I]\) that weights the separable expansion.

\begin{proof}[Proof of Proposition \ref{prop:back:coveringClippledMultipleOperator}]
    Let $$\nn_1[\alpha][u](x) :=  \operatorname{Clip}_a\l \sum_{p=1}^{P} \sum_{k=1}^H \sum_{\ell=1}^{N} \theta_{pk\ell} l_p(\bm{\alpha}) b_k(\mathbf{u}) \tau_\ell(x)\r$$ and $$ \nn_2[\alpha][u](x) := \operatorname{Clip}_a\l \sum_{p=1}^{P} \sum_{k=1}^H \sum_{\ell=1}^{N} \tilde{\theta}_{pk\ell} \tilde{l}_p(\bm{\alpha}) \tilde{b}_k(\mathbf{u}) \tilde{\tau}_\ell(x)\r$$ be two neural networks in \[
    \mathrm{Cl}_a(I,\cF_1,\cF_2,\cF_3,\{y_s\},\{c_s\},P,H,N).
    \]
    We proceed as in Proposition \ref{prop:back:covering} and start by estimating as follows: \begin{align}
        &\Vert \nn_1[\alpha][u](x) - \nn_2[\alpha][u](x)\Vert_{\Lp{\infty}(W\times U \times \Omega_V)} \notag \\
        &\leq \sum_{p=1}^{P} \sum_{k=1}^H \sum_{\ell=1}^{N} \Vert \theta_{pk\ell} l_p(\bm{\alpha}) b_k(\mathbf{u}) \tau_\ell(x) - \tilde{\theta}_{pk\ell} \tilde{l}_p(\bm{\alpha}) \tilde{b}_k(\mathbf{u}) \tilde{\tau}_\ell(x) \Vert_{\Lp{\infty}(W\times U \times \Omega_V)} \label{eq:prop:coveringClippledMultipleOperator:eq1} \\
        &\leq \sum_{p=1}^{P} \sum_{k=1}^H \sum_{\ell=1}^{N} \bigg[ \vert \theta_{pk\ell} \vert \cdot \Vert l_p(\bm{\alpha}) b_k(\mathbf{u}) \tau_\ell(x) - \tilde{l}_p(\bm{\alpha}) \tilde{b}_k(\mathbf{u}) \tilde{\tau}_\ell(x) \Vert_{\Lp{\infty}(W\times U \times \Omega_V)} \notag \\
        &+ \vert \theta_{pk\ell} - \tilde{\theta}_{pk\ell} \vert \cdot \Vert \tilde{l}_p(\bm{\alpha}) \tilde{b}_k(\mathbf{u}) \tilde{\tau}_\ell(x) \Vert_{\Lp{\infty}(W\times U \times \Omega_V)}  \Bigg] \notag \\
        &\leq \sum_{p=1}^{P} \sum_{k=1}^H \sum_{\ell=1}^{N} \Bigg[ \vert \theta_{pk\ell} \vert \cdot \bigg( \Vert l_p(\bm{\alpha}) \Vert_{\Lp{\infty}(W)} \cdot \Vert b_k(\bm{u}) \tau_\ell(x) - \tilde{b}_k(\bm{u}) \tilde{\tau}_\ell(x)  \Vert_{\Lp{\infty}(U \times \Omega_V)} \notag \\
        &+ \Vert l_p(\bm{\alpha}) - \tilde{l}_p(\bm{\alpha})  \Vert_{\Lp{\infty}(W)} \cdot \Vert \tilde{b}_k(\mathbf{u}) \tilde{\tau}_\ell(x) \Vert_{\Lp{\infty}(U \times \Omega_V)} \bigg) + \vert \theta_{pk\ell} - \tilde{\theta}_{pk\ell} \vert \cdot \Vert \tilde{l}_p(\bm{\alpha}) \tilde{b}_k(\mathbf{u}) \tilde{\tau}_\ell(x) \Vert_{\Lp{\infty}(W\times U \times \Omega_V)}  \Bigg] \notag \\
        &\leq \sum_{p=1}^{P} \sum_{k=1}^H \sum_{\ell=1}^{N} \Bigg[ \vert \theta_{pk\ell} \vert \cdot \bigg( \Vert l_p(\bm{\alpha}) \Vert_{\Lp{\infty}(W)} \cdot \Big[ \Vert b_k(\bm{u}) \Vert_{\Lp{\infty}(U)} \cdot \Vert \tau_\ell(x) - \tilde{\tau}_\ell(x)  \Vert_{\Lp{\infty}(\Omega_V)} \notag \\ 
        &+ \Vert b_k(\bm{u}) - \tilde{b}_k(\bm{u}) \Vert_{\Lp{\infty}(U)} \cdot \Vert \tilde{\tau}_\ell(x) \Vert_{\Lp{\infty}(\Omega_V)} \Big] + \Vert l_p(\bm{\alpha}) - \tilde{l}_p(\bm{\alpha})  \Vert_{\Lp{\infty}(W)} \Vert \tilde{b}_k(\mathbf{u}) \tilde{\tau}_\ell(x) \Vert_{\Lp{\infty}(U \times \Omega_V)} \bigg) \notag \\
        &+ \vert \theta_{pk\ell} - \tilde{\theta}_{pk\ell} \vert \cdot \Vert \tilde{l}_p(\bm{\alpha}) \tilde{b}_k(\mathbf{u}) \tilde{\tau}_\ell(x) \Vert_{\Lp{\infty}(W\times U \times \Omega_V)}  \Bigg] \notag \\
        &\leq \sum_{p=1}^{P} \sum_{k=1}^H \sum_{\ell=1}^{N} \Bigg[ I \bigg( R_3 \Big[ R_2 \Vert \tau_\ell(x) - \tilde{\tau}_\ell(x)  \Vert_{\Lp{\infty}(\Omega_V)} + \Vert b_k(\bm{u}) - \tilde{b}_k(\bm{u}) \Vert_{\Lp{\infty}(U)} R_1 \Big] \notag \\
        &+ \Vert l_p(\bm{\alpha}) - \tilde{l}_p(\bm{\alpha})  \Vert_{\Lp{\infty}(W)} R_2 R_1 \bigg) + \vert \theta_{pk\ell} - \tilde{\theta}_{pk\ell} \vert \cdot R_3 R_2 R_1 \Bigg] \label{eq:prop:coveringClippledMultipleOperator:eq2} \\
        &\leq \sum_{p=1}^{P} \sum_{k=1}^H \sum_{\ell=1}^{N} \Bigg[ I \bigg( R_3 \Big[R_2 L_1\kappa_1^{L_1-1}(p_1+1)^{L_1-1}(p_1 \Vert x \Vert_{\Lp{\infty}} +1)d(\tau_{\ell},\tilde{\tau}_\ell) \notag \\
        &+ R_1 L_2\kappa_2^{L_2-1}(p_2+1)^{L_2-1}(p_2 \Vert \bm{u} \Vert_{\Lp{\infty}} +1)d(b_{k},\tilde{b}_k)  \Big] \notag \\
        &+ R_1R_2 L_3\kappa_3^{L_3-1}(p_3+1)^{L_3-1}(p_3 \Vert \bm{\alpha} \Vert_{\Lp{\infty}} +1)d(l_{p},\tilde{l}_p)  \bigg)  + \vert \theta_{pk\ell} - \tilde{\theta}_{pk\ell} \vert \cdot R_3 R_2 R_1 \Bigg] \label{eq:prop:coveringClippledMultipleOperator:eq3} \\
        &= \sum_{p=1}^{P} \sum_{k=1}^H \sum_{\ell=1}^{N} \Bigg[ I R_2 R_3 L_1\kappa_1^{L_1-1}(p_1+1)^{L_1-1}(p_1 \Vert x \Vert_{\Lp{\infty}} +1)d(\tau_{\ell},\tilde{\tau}_\ell) \notag \\
        &+ I R_1 R_3 L_2\kappa_2^{L_2-1}(p_2+1)^{L_2-1}(p_2 \Vert \bm{u} \Vert_{\Lp{\infty}} +1)d(b_{k},\tilde{b}_k) \notag \\
        &+ I R_1R_2 L_3\kappa_3^{L_3-1}(p_3+1)^{L_3-1}(p_3 \Vert \bm{\alpha} \Vert_{\Lp{\infty}} +1)d(l_{p},\tilde{l}_p)  + R_1 R_2 R_3 \vert \theta_{pk\ell} - \tilde{\theta}_{pk\ell} \vert \Bigg] \label{eq:prop:coveringClippledMultipleOperator:eq4}
    \end{align}
where we used the fact that the clipping operator $\mathrm{Clip}$ is $1$-Lipschitz for \eqref{eq:prop:coveringClippledMultipleOperator:eq1}, Definitions \ref{def:networkClass} and \ref{def:clippedClass} for \eqref{eq:prop:coveringClippledMultipleOperator:eq2} as well as Proposition \ref{prop:back:covering} for \eqref{eq:prop:coveringClippledMultipleOperator:eq3}.

By the definition of \(
\mathrm{Cl}_a(I,\cF_1,\cF_2,\cF_3,\{y_s\},\{c_s\},P,H,N),
\)
we have $\theta_{pk\ell},\tilde{\theta}_{pk\ell} \in [-I,I]$ and each parameter of $l_p$, $b_k$ and $\tau_\ell$ is contained in $[-\kappa_3,\kappa_3]$, $[-\kappa_2,\kappa_2]$ and $[-\kappa_1,\kappa_1]$, respectively. For a fixed \( h > 0 \), we discretize the latter intervals into \( 2I/h \), \( 2\kappa_1/h \), \( 2\kappa_2/h \) and \( 2\kappa_3/h \)   subintervals of length \( h \), yielding \( \lfloor 2I/h \rfloor + 1 \), \( \lfloor 2\kappa_1/h \rfloor + 1 \), \( \lfloor 2\kappa_2/h \rfloor + 1 \) and \( \lfloor 2\kappa_3/h \rfloor + 1 \)  grid points. Then, for any coefficient $\theta_{pk \ell}$ or parameter value $c$ of $l_p$, $b_k$ and $\tau_\ell$, there exists grid points $\theta^*$ or \( c^* \) such that \( |\theta_{pk\ell} - \theta^*| \leq h/2 \) or \( |c - c^*| \leq h/2 \).

Now, for any $\nn \in \mathrm{Cl}_a(I,\cF_1,\cF_2,\cF_3,\{y_s\},\{c_s\},P,H,N)$, let $\nn^*$ denote its grid-constrained approximation. Specifically,  \[
\nn^*[\alpha][u](x) =  \operatorname{Clip}_a\l \sum_{p=1}^{P} \sum_{k=1}^H \sum_{\ell=1}^{N} \theta^*_{pk\ell} l_p^*(\bm{\alpha}) b_k^*(\mathbf{u}) \tau^*_\ell(x)\r. 
\] 
where $\theta_{pk\ell}^*$ is the nearest grid point to $\theta_{pk\ell}$ and $l_p^*$, $b_k^*$, and $\tau^*_\ell$ are the grid-constrained versions of $l_p$, $b_k$, and $\tau_\ell$, obtained by replacing each nonzero parameter with its nearest grid point. By construction, this implies that $\vert \theta_{pk\ell} - \theta^*_{pk\ell} \vert \leq h/2$, $d(l_p,l_p^*) \leq h/2$, $d(b_k,b_k^*)\leq h/2$ and $d(\tau_\ell,\tau_\ell^*) \leq h/2$. Inserting this into \eqref{eq:prop:coveringClippledMultipleOperator:eq4}, this yields \begin{align}
    & \Vert \nn[\alpha][u](x) - \nn^*[\alpha][u](x)\Vert_{\Lp{\infty}(W\times U \times \Omega_V)} \notag \\
    &\leq \frac{P \cdot H \cdot N \cdot h}{2}  \Bigg[ I R_2 R_3 L_1\kappa_1^{L_1-1}(p_1+1)^{L_1-1}(p_1 \Vert x \Vert_{\Lp{\infty}} +1) \notag \\
        &+ I R_1 R_3 L_2\kappa_2^{L_2-1}(p_2+1)^{L_2-1}(p_2 \Vert \bm{u} \Vert_{\Lp{\infty}} +1) \notag \\
        &+ I R_1R_2 L_3\kappa_3^{L_3-1}(p_3+1)^{L_3-1}(p_3 \Vert \bm{\alpha} \Vert_{\Lp{\infty}} +1)  + R_1 R_2 R_3  \Bigg] \label{eq:prop:coveringClippledMultipleOperator:eq5} \\
        &=: \frac{h}{2} T \notag
\end{align}
By picking \(
h = \frac{2\eta}{T},
\)
we conclude that the set of the networks of the form \begin{equation*} 
  \operatorname{Clip}_a\l \sum_{p=1}^{P} \sum_{k=1}^H \sum_{\ell=1}^{N} \theta_{pk\ell} l_p(\bm{\alpha}) b_k(\mathbf{u}) \tau_\ell(x)\r  
\end{equation*} whose nonzero coefficients and parameters are constrained to grid points form a $\eta$-cover of \[
\mathrm{Cl}_a(I,\cF_1,\cF_2,\cF_3,\{y_s\},\{c_s\},P,H,N).
\]

For each tuple $(p,k,\ell)$, since $l_p$, $b_k$ and $\tau_\ell$ have at most $K_3$, $K_2$ and $K_1$ nonzero parameters, it suffices to consider networks restricted to grid points and at most $K_3$, $K_2$ and $K_1$ nonzero parameters. As argued in Proposition \ref{prop:back:covering}, for each tuple $(p,k,\ell)$, there therefore are \begin{itemize}
    \item  $F(L_3,p_3,K_3,\kappa_3,h)$ possible grid-constrained networks $l^*_p$,
    \item $F(L_2,p_2,K_2,\kappa_2,h)$  possible grid-constrained networks $b^*_k$
    \item and $F(L_1,p_1,K_1,\kappa_1,h)$ possible grid-constrained networks $\tau^*_\ell$. 
\end{itemize}  
Furthermore, there are \( \lfloor 2I/h \rfloor + 1 \) choices for $\theta_{pk\ell}^*$ and thus, for each tuple $(p,k,\ell)$, this yield a total of \[
\l \lfloor 2I/h \rfloor + 1\r  F(L_3,p_3,K_3,\kappa_3,h) F(L_2,p_2,K_2,\kappa_2,h) F(L_1,p_1,K_1,h)
\]
grid-constrained networks. Since for each of the $P \cdot H \cdot N$ tuples $(p,k,\ell)$, the associated grid-constrained networks can be selected independently, we conclude that \begin{align}
    &\cN\l\eta,\mathrm{Cl}_a(I,\cF_1,\cF_2,\cF_3,\{y_s\},\{c_s\},P,H,N),\Vert \cdot \Vert_{\Lp{\infty}(W \times U \times \Omega_V)}\r \notag \\
    &\leq \ls \l \lfloor 2I/h \rfloor + 1\r  F(L_3,p_3,K_3,\kappa_3,h) F(L_2,p_2,K_2,\kappa_2,h) F(L_1,p_1,K_1,\kappa_1,h) \rs^{ P \cdot H \cdot N} . \notag
\end{align} 

For the second claim of the Proposition, we continue from \eqref{eq:prop:coveringClippledMultipleOperator:eq5} and estimate it using Assumption \ref{assumption:Main:assumptions:S4}: \begin{align}
    & \Vert \nn[\alpha][u](x) - \nn^*[\alpha][u](x)\Vert_{\Lp{\infty}(W\times U \times \Omega_V)} \notag \\
    &\leq \frac{P \cdot H \cdot N \cdot h}{2}  \Bigg[ I R_2 R_3 L_1\kappa_1^{L_1-1}(p_1+1)^{L_1-1}(p_1 \Vert x \Vert_{\Lp{\infty}} +1) \notag \\
        &+ I R_1 R_3 L_2\kappa_2^{L_2-1}(p_2+1)^{L_2-1}(p_2 \Vert \bm{u} \Vert_{\Lp{\infty}} +1) \notag \\
        &+ I R_1R_2 L_3\kappa_3^{L_3-1}(p_3+1)^{L_3-1}(p_3 \Vert \bm{\alpha} \Vert_{\Lp{\infty}} +1)  + R_1 R_2 R_3  \Bigg] \notag \\
        &\leq \frac{P \cdot H \cdot N \cdot h}{2}  \Bigg[ I R_2 R_3 L_1\kappa_1^{L_1-1}(p_1+1)^{L_1-1}(p_1 \gamma_V +1) \notag \\
        &+ I R_1 R_3 L_2\kappa_2^{L_2-1}(p_2+1)^{L_2-1}(p_2 \beta_U +1) \notag \\
        &+ I R_1R_2 L_3\kappa_3^{L_3-1}(p_3+1)^{L_3-1}(p_3 \beta_W +1)  + R_1 R_2 R_3  \Bigg]. \notag
\end{align}
The rest of the proof is analogous to the first part. 

\end{proof}

\subsubsection{Generalization bounds}

We start this section with several intermediate results used in the proof of Theorem \ref{thm:scalingLawsGeneralizationError}. Specifically, the next result bounds a quantity appearing in the estimation of the expected training-error. 

\begin{lemma}[Noise-error cross term bound] \label{lem:main:T5Bound}
    Assume the same setting as in Theorem \ref{thm:scalingLawsGeneralizationError}. 
    Then, \begin{align*}
        &\bbE_{S_{G, \{y_s\}, \{c_s\}}} \ls \frac{1}{n_\alpha n_u n_x} \sum_{\ell=1}^{n_\alpha} \sum_{i=1}^{n_u} \sum_{j=1}^{n_x} \left( G_{a,I,\cF_1,\cF_2,\cF_3,S}[\bm{\alpha}_\ell][\ub_{\ell i}](x_{\ell ij}) - G[\alpha_\ell][u_{\ell i}](x_{\ell ij})  \right) \zeta_{\ell ij} \rs \notag \\
        &\leq \eta \sigma +  \frac{\sigma}{\sqrt{n_\alpha n_u n_x}} \l \eta +  \sqrt{\bbE_{S_{G, \{y_s\}, \{c_s\}}} \ls \empiricalEvaluation(G_{a,I,\cF_1,\cF_2,\cF_3,S} -G)^2 \rs} \r \notag \\
    &\times \sqrt{\log \l \cN\l\eta,\mathrm{Cl}_a(I,\cF_1,\cF_2,\cF_3,\{y_s\},\{c_s\},P,H,N),\Vert \cdot \Vert_{\Lp{\infty}(W \times U \times \Omega_V)}\r \r + \log(2)} \notag
    \end{align*}
    where $\empiricalEvaluation(\nn)^2 = \frac{1}{n_\alpha n_u n_x} \sum_{\ell=1}^{n_\alpha} \sum_{i=1}^{n_u} \sum_{j=1}^{n_x} \nn[\bm{\alpha}_\ell][\bm{u}_{\ell i}][x_{\ell ij}]^2$. 
\end{lemma}

\begin{proof}
Let $$\coverName =  \{\cN\cN^*_{k}\}_{k=1}^{\cN\l\eta,\mathrm{Cl}_a(I,\cF_1,\cF_2,\cF_3,\{y_s\},\{c_s\},P,H,N),\Vert \cdot \Vert_{\Lp{\infty}(W \times U \times \Omega_V)}\r}$$ be the $\eta$-covering constructed in the proof of Proposition \ref{prop:back:coveringClippledMultipleOperator}. We have that \[
G_{a,I,\cF_1,\cF_2,\cF_3,S} \in \mathrm{Cl}_a(I,\cF_1,\cF_2,\cF_3,\{y_s\},\{c_s\},P^{n_{c_W}},H^{n_{c_U}},N^{d_V})
\] 
and there therefore exists $G_{a,I,\cF_1,\cF_2,\cF_3,S}^* \in \coverName$ such that \begin{equation} \label{eq:T5bound:eq1}
    \Vert G_{a,I,\cF_1,\cF_2,\cF_3,S} - G_{a,I,\cF_1,\cF_2,\cF_3,S}^* \Vert_{\Lp{\infty}(W \times U \times \Omega_V)} \leq \eta.
\end{equation}

\paragraph{Step 1: Decomposition}

We first decompose our quantity of interest as follows: \begin{align}
    &T_1 := \bbE_{S_{G, \{y_s\}, \{c_s\}}} \ls \frac{1}{n_\alpha n_u n_x} \sum_{\ell=1}^{n_\alpha} \sum_{i=1}^{n_u} \sum_{j=1}^{n_x} \left( G_{a,I,\cF_1,\cF_2,\cF_3,S}[\bm{\alpha}_\ell][\ub_{\ell i}](x_{\ell ij}) - G[\alpha_\ell][u_{\ell i}](x_{\ell ij})  \right) \zeta_{\ell ij} \rs \notag \\
    &= \bbE_{S_{G, \{y_s\}, \{c_s\}}} \ls \frac{1}{n_\alpha n_u n_x} \sum_{\ell=1}^{n_\alpha} \sum_{i=1}^{n_u} \sum_{j=1}^{n_x} \left( G_{a,I,\cF_1,\cF_2,\cF_3,S}[\bm{\alpha}_\ell][\ub_{\ell i}](x_{\ell ij}) - G_{a,I,\cF_1,\cF_2,\cF_3,S}^*[\bm{\alpha}_\ell][\ub_{\ell i}](x_{\ell ij})  \right) \zeta_{\ell ij} \rs \notag \\
    &+ \bbE_{S_{G, \{y_s\}, \{c_s\}}} \ls \frac{1}{n_\alpha n_u n_x} \sum_{\ell=1}^{n_\alpha} \sum_{i=1}^{n_u} \sum_{j=1}^{n_x} \left( G_{a,I,\cF_1,\cF_2,\cF_3,S}^*[\bm{\alpha}_\ell][\ub_{\ell i}](x_{\ell ij}) - G[\alpha_\ell][u_{\ell i}](x_{\ell ij})  \right) \zeta_{\ell ij} \rs \notag \\
    &=: T_2 + T_3. \label{eq:T5bound:eq17}
\end{align}
We first upper bound $T_2$ as follows: \begin{align}
    T_2^2 &\leq \bbE_{S_{G, \{y_s\}, \{c_s\}}} \Bigg[ \Bigg( \frac{1}{n_\alpha n_u n_x} \sum_{\ell=1}^{n_\alpha} \sum_{i=1}^{n_u} \sum_{j=1}^{n_x} \bigg( G_{a,I,\cF_1,\cF_2,\cF_3,S}[\bm{\alpha}_\ell][\ub_{\ell i}](x_{\ell ij}) \notag \\
    &- G_{a,I,\cF_1,\cF_2,\cF_3,S}^*[\bm{\alpha}_\ell][\ub_{\ell i}](x_{\ell ij})  \bigg) \zeta_{\ell ij} \Bigg)^2 \Bigg] \label{eq:T5bound:eq2} \\
    &\leq \bbE_{S_{G, \{y_s\}, \{c_s\}}} \Bigg[ \Bigg( \frac{1}{n_\alpha n_u n_x} \sum_{\ell=1}^{n_\alpha} \sum_{i=1}^{n_u} \sum_{j=1}^{n_x} \bigg( G_{a,I,\cF_1,\cF_2,\cF_3,S}[\bm{\alpha}_\ell][\ub_{\ell i}](x_{\ell ij}) \notag \\
    &- G_{a,I,\cF_1,\cF_2,\cF_3,S}^*[\bm{\alpha}_\ell][\ub_{\ell i}](x_{\ell ij})  \bigg)^2 \Bigg) \Bigg( \frac{1}{n_\alpha n_u n_x} \sum_{\ell=1}^{n_\alpha} \sum_{i=1}^{n_u} \sum_{j=1}^{n_x} \zeta_{\ell ij}^2 \Bigg) \Bigg] \label{eq:T5bound:eq3} \\
    &\leq \eta^2 \bbE_{S_{G, \{y_s\}, \{c_s\}}} \Bigg[ \frac{1}{n_\alpha n_u n_x} \sum_{\ell=1}^{n_\alpha} \sum_{i=1}^{n_u} \sum_{j=1}^{n_x} \zeta_{\ell ij}^2 \Bigg] \label{eq:T5bound:eq4} \\
    &\leq \eta^2 \sigma^2 \label{eq:T5bound:eq5}
\end{align}
where we used Jensen's inequality for \eqref{eq:T5bound:eq2}, Cauchy-Schwarz on the inner sum for \eqref{eq:T5bound:eq3}, \eqref{eq:T5bound:eq1} for \eqref{eq:T5bound:eq4} and Definition \ref{def:trainingSet} for \eqref{eq:T5bound:eq5}.

\paragraph{Step 2: Moment generating function estimation}

We now turn to the estimation of $T_3$. In particular,  we want to express the latter through moments of sub-Gaussian random variables. We first recall the squared empirical evaluation of a network \[
\empiricalEvaluation(\nn)^2 = \frac{1}{n_\alpha n_u n_x} \sum_{\ell=1}^{n_\alpha} \sum_{i=1}^{n_u} \sum_{j=1}^{n_x} \nn[\bm{\alpha}_\ell][\bm{u}_{\ell i}][x_{\ell ij}]^2. 
\]
With an abuse of notation, for our map $G$, we also write \[
\empiricalEvaluation(G)^2 = \frac{1}{n_\alpha n_u n_x} \sum_{\ell=1}^{n_\alpha} \sum_{i=1}^{n_u} \sum_{j=1}^{n_x} G[\alpha_\ell][u_{\ell i}][x_{\ell ij}]^2.
\]
Then, \begin{align}
    &\empiricalEvaluation(G_{a,I,\cF_1,\cF_2,\cF_3,S}^* - G) \notag \\
    &= \empiricalEvaluation(G_{a,I,\cF_1,\cF_2,\cF_3,S}^* - G_{a,I,\cF_1,\cF_2,\cF_3,S} + G_{a,I,\cF_1,\cF_2,\cF_3,S} - G) \notag \\
    &\leq \Bigg[ \frac{2}{n_\alpha n_u n_x} \sum_{\ell=1}^{n_\alpha} \sum_{i=1}^{n_u} \sum_{j=1}^{n_x} \l G_{a,I,\cF_1,\cF_2,\cF_3,S}^*[\bm{\alpha}_\ell][\bm{u}_{\ell i}][x_{\ell ij}] - G_{a,I,\cF_1,\cF_2,\cF_3,S}[\bm{\alpha}_\ell][\bm{u}_{\ell i}][x_{\ell ij}] \r^2 \notag \\
    &+  \frac{2}{n_\alpha n_u n_x} \sum_{\ell=1}^{n_\alpha} \sum_{i=1}^{n_u} \sum_{j=1}^{n_x} \l G_{a,I,\cF_1,\cF_2,\cF_3,S}[\bm{\alpha}_\ell][\bm{u}_{\ell i}][x_{\ell ij}] - G[\alpha_\ell][u_{\ell i}][x_{\ell ij}] \r^2   \Bigg]^{1/2} \notag \\
    &\leq \ls 2 \eta^2 + 2\empiricalEvaluation(G_{a,I,\cF_1,\cF_2,\cF_3,S} - G)^2 \rs^{1/2} \label{eq:T5bound:eq6} \\
    &\leq \sqrt{2}\ls \eta + \empiricalEvaluation(G_{a,I,\cF_1,\cF_2,\cF_3,S} - G)  \r \label{eq:T5bound:eq7}
\end{align}
where we used \eqref{eq:T5bound:eq1} for \eqref{eq:T5bound:eq6}. Next, we define \[
z_k = \frac{1}{\sqrt{n_\alpha n_u n_x} \empiricalEvaluation(\nn^*_k - G)} \sum_{\ell=1}^{n_\alpha} \sum_{i=1}^{n_u} \sum_{j=1}^{n_x} \l \nn^*_k [\bm{\alpha}_\ell][\bm{u}_{\ell i}][x_{\ell ij}] - G[\alpha_\ell][u_{\ell i}][x_{\ell ij}] \r \zeta_{\ell ij}
\]
for $1 \leq k \leq \cN\l\eta,\mathrm{Cl}_a(I,\cF_1,\cF_2,\cF_3,\{y_s\},\{c_s\},P,H,N),\Vert \cdot \Vert_{\Lp{\infty}(W \times U \times \Omega_V)}\r$ and note that, \begin{align}
    z_k \mid \left\{ \alpha_\ell, \left\{ u_{\ell i}, \left\{ x_{\ell ij}\right\} \right\} \right\} &\sim \mathrm{subG}\l \sigma^2 \frac{\sum_{\ell=1}^{n_\alpha} \sum_{i=1}^{n_u} \sum_{j=1}^{n_x} \l \nn^*_k [\bm{\alpha}_\ell][\bm{u}_{\ell i}][x_{\ell ij}] - G[\alpha_\ell][u_{\ell i}][x_{\ell ij}] \r^2}{\empiricalEvaluation(\nn^*_k - G)^2 n_\alpha n_u n_x} \r \label{eq:T5bound:eq8} \\
    &\sim \mathrm{subG}(\sigma^2) \notag
\end{align}
where we used Definition \ref{def:trainingSet} and \cite[p.24]{boucheron}
for \eqref{eq:T5bound:eq8}. We estimate as follows: \begin{align}
    T_3 &= \bbE_{S_{G, \{y_s\}, \{c_s\}}} \Bigg[ \frac{\empiricalEvaluation(G_{a,I,\cF_1,\cF_2,\cF_3,S}^* - G)}{n_\alpha n_u n_x} \notag \\
    &\times \frac{\sum_{\ell=1}^{n_\alpha} \sum_{i=1}^{n_u} \sum_{j=1}^{n_x} \left( G_{a,I,\cF_1,\cF_2,\cF_3,S}^*[\bm{\alpha}_\ell][\ub_{\ell i}](x_{\ell ij}) - G[\alpha_\ell][u_{\ell i}](x_{\ell ij})  \right) \zeta_{\ell ij}}{\empiricalEvaluation(G_{a,I,\cF_1,\cF_2,\cF_3,S}^* - G)} \Bigg] \notag \\
    &\leq \sqrt{2}\bbE_{S_{G, \{y_s\}, \{c_s\}}} \Bigg[ \frac{\eta + \empiricalEvaluation(G_{a,I,\cF_1,\cF_2,\cF_3,S} - G)}{\sqrt{n_\alpha n_u n_x}} \notag \\
    &\times \frac{\sum_{\ell=1}^{n_\alpha} \sum_{i=1}^{n_u} \sum_{j=1}^{n_x} \left( G_{a,I,\cF_1,\cF_2,\cF_3,S}^*[\bm{\alpha}_\ell][\ub_{\ell i}](x_{\ell ij}) - G[\alpha_\ell][u_{\ell i}](x_{\ell ij})  \right) \zeta_{\ell ij}}{\sqrt{n_\alpha n_u n_x} \empiricalEvaluation(G_{a,I,\cF_1,\cF_2,\cF_3,S}^* - G)} \Bigg] \label{eq:T5bound:eq9} \\
    &\leq \sqrt{2}\bbE_{S_{G, \{y_s\}, \{c_s\}}} \ls \frac{\eta + \empiricalEvaluation(G_{a,I,\cF_1,\cF_2,\cF_3,S} - G)}{\sqrt{n_\alpha n_u n_x}} \max_{k} \vert z_k \vert \rs \notag \\
    &\leq \frac{\sqrt{2}}{\sqrt{n_\alpha n_u n_x}} \sqrt{\bbE_{S_{G, \{y_s\}, \{c_s\}}} \ls \l \eta + \empiricalEvaluation(G_{a,I,\cF_1,\cF_2,\cF_3,S} -G) \r^2 \rs} \sqrt{\bbE_{S_{G, \{y_s\}, \{c_s\}}} \ls \max_{k} \vert z_k^2 \vert \rs} \label{eq:T5bound:eq10} \\
    &\leq \frac{2}{\sqrt{n_\alpha n_u n_x}} \sqrt{\bbE_{S_{G, \{y_s\}, \{c_s\}}} \ls \eta^2 + \empiricalEvaluation(G_{a,I,\cF_1,\cF_2,\cF_3,S} -G)^2 \rs} \notag \\
    &\times \sqrt{\bbE_{\left\{ \alpha_\ell, \left\{ u_{\ell i}, \left\{ x_{\ell ij}\right\} \right\} \right\}} \ls \bbE_{\{\zeta_{\ell ij}\}} \ls \max_{k} \vert z_k^2 \vert \mid \left\{ \alpha_\ell, \left\{ u_{\ell i}, \left\{ x_{\ell ij}\right\} \right\} \right\} \rs \rs} \label{eq:T5bound:eq11} \\
    &\leq \frac{2}{\sqrt{n_\alpha n_u n_x}} \l \eta +  \sqrt{\bbE_{S_{G, \{y_s\}, \{c_s\}}} \ls \empiricalEvaluation(G_{a,I,\cF_1,\cF_2,\cF_3,S} -G)^2 \rs} \r \notag \\
    &\times \sqrt{\bbE_{\left\{ \alpha_\ell, \left\{ u_{\ell i}, \left\{ x_{\ell ij}\right\} \right\} \right\}} \ls \bbE_{\{\zeta_{\ell ij}\}} \ls \max_{k} \vert z_k^2 \vert \mid \left\{ \alpha_\ell, \left\{ u_{\ell i}, \left\{ x_{\ell ij}\right\} \right\} \right\} \rs \rs} \label{eq:T5bound:eq12} \\
    &=: \frac{2}{\sqrt{n_\alpha n_u n_x}} \l \eta +  \sqrt{\bbE_{S_{G, \{y_s\}, \{c_s\}}} \ls \empiricalEvaluation(G_{a,I,\cF_1,\cF_2,\cF_3,S} -G)^2 \rs} \r \sqrt{\bbE_{\left\{ \alpha_\ell, \left\{ u_{\ell i}, \left\{ x_{\ell ij}\right\} \right\} \right\}} \ls T_4 \rs} \label{eq:T5bound:eq13}
\end{align}
where we used \eqref{eq:T5bound:eq7} for \eqref{eq:T5bound:eq9}, Cauchy-Schwarz for \eqref{eq:T5bound:eq10}, we split the expectation 
$$\bbE_{S_{G, \{y_s\}, \{c_s\}}} = \bbE_{\left\{ \alpha_\ell, \left\{ u_{\ell i}, \left\{ (x_{\ell ij},\zeta_{\ell ij})\right\}_j \right\}_i \right\}_\ell} = \bbE_{\left\{ \alpha_\ell, \left\{ u_{\ell i}, \left\{ x_{\ell ij}\right\}_j \right\}_i \right\}_\ell, \{\zeta_{\ell ij}\}_{\ell ij}} $$
using the law of iterated expectations $E_{X,Y}[f(X,Y)] = E_Y[E_X[f(X,Y) \mid Y]]$ for \eqref{eq:T5bound:eq11} and the inequality $\sqrt{a^2 + b} \leq a + \sqrt{b}$ for \eqref{eq:T5bound:eq12}. 

Our aim is now show that $T_4 \leq C$ for $C >0$ independent of all other random variables. This will follow from standard bounds on moment generating functions of the sub-Gaussian variables $z_k \mid \left\{ \alpha_\ell, \left\{ u_{\ell i}, \left\{ x_{\ell ij}\right\} \right\} \right\}$. For $2/\sigma^2 \leq t \leq 4/\sigma^2$, we proceed as follows: \begin{align}
    T_4 &= \frac{1}{t} \log\l \exp \l \bbE_{\{\zeta_{\ell ij}\}} \ls \max_{k} t\vert z_k^2 \vert \mid \left\{ \alpha_\ell, \left\{ u_{\ell i}, \left\{ x_{\ell ij}\right\} \right\} \right\} \rs \r \r \notag \\
    &\leq \frac{1}{t} \log\l \bbE_{\{\zeta_{\ell ij}\}} \ls \max_{k} e^{ tz_k^2} \mid \left\{ \alpha_\ell, \left\{ u_{\ell i}, \left\{ x_{\ell ij}\right\} \right\} \right\} \rs  \r \label{eq:T5bound:eq14} \\
    &\leq \frac{1}{t} \log\l \bbE_{\{\zeta_{\ell ij}\}} \ls \sum_{k=1}^{\cN\l\eta,\mathrm{Cl}_a(I,\cF_1,\cF_2,\cF_3,\{y_s\},\{c_s\},P,H,N),\Vert \cdot \Vert_{\Lp{\infty}(W \times U \times \Omega_V)}\r} e^{ tz_k^2} \mid \left\{ \alpha_\ell, \left\{ u_{\ell i}, \left\{ x_{\ell ij}\right\} \right\} \right\} \rs  \r \notag \\
    &\leq \frac{1}{t} \log \l \cN\l\eta,\mathrm{Cl}_a(I,\cF_1,\cF_2,\cF_3,\{y_s\},\{c_s\},P,H,N),\Vert \cdot \Vert_{\Lp{\infty}(W \times U \times \Omega_V)}\r \r \notag \\
    &+ \frac{1}{t} \log \l \bbE_{\{\zeta_{\ell ij}\}} \ls e^{tz_1^2} \mid \left\{ \alpha_\ell, \left\{ u_{\ell i}, \left\{ x_{\ell ij}\right\} \right\} \right\} \rs \r \label{eq:T5bound:eq15} \\
    &\leq \frac{1}{t} \log \l \cN\l\eta,\mathrm{Cl}_a(I,\cF_1,\cF_2,\cF_3,\{y_s\},\{c_s\},P,H,N),\Vert \cdot \Vert_{\Lp{\infty}(W \times U \times \Omega_V)}\r \r + \frac{1}{t} \log(2) \label{eq:T5bound:eq16} 
\end{align}
where we used Jensen's inequality for \eqref{eq:T5bound:eq14}, the fact that $z_k$ is identically distributed for all $k$ by \eqref{eq:T5bound:eq8} for \eqref{eq:T5bound:eq15} and the following equivalence from \cite[p. 26]{boucheron} for \eqref{eq:T5bound:eq16}: $X \sim \mathrm{subG}(v)$ if and only if $\bbE\ls e^{t X^2} \rs \leq 2$ for all $2/v \leq t \leq 4/v$. Picking $t = 4/\sigma^2$, we obtain that \[
T_4 \leq \frac{\sigma^2}{4}    \log \l \cN\l\eta,\mathrm{Cl}_a(I,\cF_1,\cF_2,\cF_3,\{y_s\},\{c_s\},P^{n_{c_W}},H^{n_{c_U}},N^{d_V}),\Vert \cdot \Vert_{\Lp{\infty}(W \times U \times \Omega_V)}\r \r + \frac{\sigma^2}{4} \log(2) 
\]
which is non-random, since uniform in $\alpha \in W$, $u \in U$ and $x \in \Omega_V$ by the second claim of Proposition \ref{prop:back:coveringClippledMultipleOperator}. 
Therefore, substituting the latter expression in \eqref{eq:T5bound:eq13} yields: \begin{align}
    T_3 &\leq \frac{\sigma}{\sqrt{n_\alpha n_u n_x}} \l \eta +  \sqrt{\bbE_{S_{G, \{y_s\}, \{c_s\}}} \ls \empiricalEvaluation(G_{a,I,\cF_1,\cF_2,\cF_3,S} -G)^2 \rs} \r \notag \\
    &\times \sqrt{\log \l \cN\l\eta,\mathrm{Cl}_a(I,\cF_1,\cF_2,\cF_3,\{y_s\},\{c_s\},P^{n_{c_W}},H^{n_{c_U}},N^{d_V}),\Vert \cdot \Vert_{\Lp{\infty}(W \times U \times \Omega_V)}\r \r + \log(2)}. \notag 
\end{align}
Continuing from \eqref{eq:T5bound:eq17} and using \eqref{eq:T5bound:eq5}, we conclude that \begin{align}
   T_1 &\leq  \eta \sigma +  \frac{\sigma}{\sqrt{n_\alpha n_u n_x}} \l \eta +  \sqrt{\bbE_{S_{G, \{y_s\}, \{c_s\}}} \ls \empiricalEvaluation(G_{a,I,\cF_1,\cF_2,\cF_3,S} -G)^2 \rs} \r \notag \\
    &\times \sqrt{\log \l \cN\l\eta,\mathrm{Cl}_a(I,\cF_1,\cF_2,\cF_3,\{y_s\},\{c_s\},P^{n_{c_W}},H^{n_{c_U}},N^{d_V}),\Vert \cdot \Vert_{\Lp{\infty}(W \times U \times \Omega_V)}\r \r + \log(2)}. \notag \qedhere
\end{align}
\end{proof}

We next establish the covering number of a shifted network class used for bounding the generalization-gap.

\begin{lemma}[Covering number of the shifted clipped multiple operator class] \label{lem:coveringShifted}
Let $$\mathrm{Cl}_a(I,\cF_1,\cF_2,\cF_3,\{y_s\},\{c_s\},P,H,N)$$ be the clipped multiple operator class defined in Definition \ref{def:clippedClass} where, for each $\nn$ in that class, we assume the following input domains: $\nn: W \times U \times \Omega_V \mapsto \bbR$. Let $G:W \mapsto \{G[\alpha]:U \mapsto V\}$ be a map such that $\Vert G \Vert_{\Lp{\infty}(W \times U \times \Omega_V)} \leq a$ and define the shifted clipped multiple operator class \begin{align}
   &\mathrm{SCl}_a(I,\cF_1,\cF_2,\cF_3,\{y_s\},\{c_s\},P,H,N) \notag \\
   &= \left\{ \l \nn[\bm{\alpha}][\bm{u}](x) - G[\alpha][u](x) \r^2 \mid \nn \in \mathrm{Cl}_a(I,\cF_1,\cF_2,\cF_3,\{y_s\},\{c_s\},P,H,N) \right\}. \notag
\end{align}
Then, \[\{(\cN\cN^*_{k} - G)^2\}_{k=1}^{\cN\l\eta/(4a),\mathrm{Cl}_a(I,\cF_1,\cF_2,\cF_3,\{y_s\},\{c_s\},P,H,N),\Vert \cdot \Vert_{\Lp{\infty}(W \times U \times \Omega_V)}\r}
    \]
is a $\eta$-cover of $\mathrm{SCl}_a(I,\cF_1,\cF_2,\cF_3,\{y_s\},\{c_s\},P,H,N)$ where $$\{\cN\cN^*_{k}\}_{k=1}^{\cN\l\eta/(4a),\mathrm{Cl}_a(I,\cF_1,\cF_2,\cF_3,\{y_s\},\{c_s\},P,H,N),\Vert \cdot \Vert_{\Lp{\infty}(W \times U \times \Omega_V)}\r}$$ is the $\eta/(4a)$-covering constructed in the proof of Proposition \ref{prop:back:coveringClippledMultipleOperator}. In particular, we also have that \begin{align*}
&\cN\l\eta,\mathrm{SCl}_a(I,\cF_1,\cF_2,\cF_3,\{y_s\},\{c_s\},P,H,N),\Vert \cdot \Vert_{\Lp{\infty}(W \times U \times \Omega_V)}\r \notag \\
&\leq \cN\l\eta/(4a),\mathrm{Cl}_a(I,\cF_1,\cF_2,\cF_3,\{y_s\},\{c_s\},P,H,N),\Vert \cdot \Vert_{\Lp{\infty}(W \times U \times \Omega_V)}\r
\end{align*}
and $$\Vert (\cN\cN^*_{k} - G)^2 \Vert_{\Lp{\infty}(W \times U \times \Omega_V)} \leq 4a^2$$ for all $1 \leq k \leq \cN\l\eta/(4a),\mathrm{Cl}_a(I,\cF_1,\cF_2,\cF_3,\{y_s\},\{c_s\},P,H,N),\Vert \cdot \Vert_{\Lp{\infty}(W \times U \times \Omega_V)}\r$.

\end{lemma}

\begin{proof}
    Let $g = (\nn - G)^2$ and $\tilde{g} = (\widetilde{\nn} - G)^2$ be two networks in $\mathrm{SCl}_a(I,\cF_1,\cF_2,\cF_3,\{y_s\},\{c_s\},P,H,N)$. Then, we have: \begin{align}
        \Vert g - \tilde{g} \Vert_{\Lp{\infty}(W \times U \times \Omega_V)} &= \left\Vert \l \nn - \widetilde{\nn}\r \l \nn + \widetilde{\nn} - 2G \r \right\Vert_{\Lp{\infty}(W \times U \times \Omega_V)} \label{eq:coveringShifted:eq1} \\
        &\leq \left\Vert \nn - \widetilde{\nn} \right\Vert_{\Lp{\infty}(W \times U \times \Omega_V)} \left\Vert \nn + \widetilde{\nn} - 2G \right\Vert_{\Lp{\infty}(W \times U \times \Omega_V)} \notag \\
        &\leq 4a \left\Vert \nn - \widetilde{\nn} \right\Vert_{\Lp{\infty}(W \times U \times \Omega_V)} \label{eq:coveringShifted:eq2}
    \end{align} 
    where we used the identity $(b-c)^2 - (d-c)^2 = (b-d)(b+d-2c)$ for \eqref{eq:coveringShifted:eq1} as well as the facts that $\nn, \widetilde{\nn} \in \mathrm{Cl}_a(I,\cF_1,\cF_2,\cF_3,\{y_s\},\{c_s\},P,H,N)$ and $\Vert G \Vert_{\Lp{\infty}(W \times U \times \Omega_V)} \leq a$ for \eqref{eq:coveringShifted:eq2}. 
    
    Let $\{\cN\cN^*_{k}\}_{k=1}^{\cN\l\eta/(4a),\mathrm{Cl}_a(I,\cF_1,\cF_2,\cF_3,\{y_s\},\{c_s\},P,H,N),\Vert \cdot \Vert_{\Lp{\infty}(W \times U \times \Omega_V)}\r}$ be the $\eta/(4a)$-covering constructed in the proof of Proposition \ref{prop:back:coveringClippledMultipleOperator}.
    Equation \eqref{eq:coveringShifted:eq2} implies that \[\{(\cN\cN^*_{k} - G)^2\}_{k=1}^{\cN\l\eta/(4a),\mathrm{Cl}_a(I,\cF_1,\cF_2,\cF_3,\{y_s\},\{c_s\},P,H,N),\Vert \cdot \Vert_{\Lp{\infty}(W \times U \times \Omega_V)}\r}
    \]
    forms a $\eta$-covering of $\mathrm{SCl}_a(I,\cF_1,\cF_2,\cF_3,\{y_s\},\{c_s\},P,H,N)$. In particular, by the construction in Proposition \ref{prop:back:coveringClippledMultipleOperator}, we know that $\Vert \nn^*_k \Vert_{\Lp{\infty}(W \times U \times \Omega_V)} \leq a$ for all $k$ and therefore, \[
    \Vert (\cN\cN^*_{k} - G)^2 \Vert_{\Lp{\infty}(W \times U \times \Omega_V)} \leq 4 a^2. \qedhere
    \]

\end{proof}

\begin{lemma}[Bound on generalization gap/estimation part] \label{lem:T2bound}
    Assume the same setting as in Theorem \ref{thm:scalingLawsGeneralizationError}.
Then, \begin{align}
    &\mathbb{E}_{S_{G, \{y_s\}, \{c_s\}}} \mathbb{E}_{\alpha \sim \mu_\alpha} \, \mathbb{E}_{u \sim \mu_u} \, \mathbb{E}_{\{x_j\}_{j=1}^{n_x} \sim \mu_x^{\otimes n_x}} \left[
\frac{1}{n_x} \sum_{j=1}^{n_x} \left( G_{a,I,\cF_1,\cF_2,\cF_3,S}[\bm{\alpha}][\ub](x_j) - G[\alpha][u](x_j) \right)^2
\right] \notag  \\
& - 2 \bbE_{S_{G, \{y_s\}, \{c_s\}}} \ls \frac{1}{n_\alpha n_u n_x} \sum_{\ell=1}^{n_\alpha} \sum_{i=1}^{n_u} \sum_{j=1}^{n_x} \left( G_{a,I,\cF_1,\cF_2,\cF_3,S}[\bm{\alpha}_\ell][\ub_{\ell i}](x_{\ell ij}) - G[\alpha_\ell][u_{\ell i}](x_{\ell ij}) \right)^2 \rs \notag \\ 
&\leq 6 \eta + \frac{112\beta_V^2}{3n_\alpha} \log\l   \cN\l\eta,\mathrm{SCl}_a(I,\cF_1,\cF_2,\cF_3,\{y_s\},\{c_s\},P^{n_{c_W}},H^{n_{c_U}},N^{d_V}),\Vert \cdot \Vert_{\Lp{\infty}(W \times U \times \Omega_V)}\r \r \notag
\end{align}
where $\mathrm{SCl}_a(I,\cF_1,\cF_2,\cF_3,\{y_s\},\{c_s\},P^{n_{c_W}},H^{n_{c_U}},N^{d_V})$ is defined in Lemma \ref{lem:coveringShifted}.
    
\end{lemma}

\begin{proof}
    For ease of notation, we define $$\hat{g}[\alpha][u][x] =  \left( G_{a,I,\cF_1,\cF_2,\cF_3,S}[\bm{\alpha}][\ub](x) - G[\alpha][u](x) \right)^2.$$

\paragraph{Step 1: Symmetrization}
    
We estimate as follows: \begin{align}
        &T_1 := \mathbb{E}_{S_{G, \{y_s\}, \{c_s\}}} \mathbb{E}_{\alpha \sim \mu_\alpha} \, \mathbb{E}_{u \sim \mu_u} \, \mathbb{E}_{\{x_j\}_{j=1}^{n_x} \sim \mu_x^{\otimes n_x}} \left[
\frac{1}{n_x} \sum_{j=1}^{n_x} \left( G_{a,I,\cF_1,\cF_2,\cF_3,S}[\bm{\alpha}][\ub](x_j) - G[\alpha][u](x_j) \right)^2
\right] \notag  \\
& - 2 \bbE_{S_{G, \{y_s\}, \{c_s\}}} \ls \frac{1}{n_\alpha n_u n_x} \sum_{\ell=1}^{n_\alpha} \sum_{i=1}^{n_u} \sum_{j=1}^{n_x} \left( G_{a,I,\cF_1,\cF_2,\cF_3,S}[\bm{\alpha}_\ell][\ub_{\ell i}](x_{\ell ij}) - G[\alpha_\ell][u_{\ell i}](x_{\ell ij}) \right)^2 \rs \notag \\ 
&= \mathbb{E}_{S_{G, \{y_s\}, \{c_s\}}} \mathbb{E}_{\alpha \sim \mu_\alpha} \, \mathbb{E}_{u \sim \mu_u} \, \mathbb{E}_{\{x_j\}_{j=1}^{n_x} \sim \mu_x^{\otimes n_x}} \ls
\frac{1}{n_x} \sum_{j=1}^{n_x} \hat{g}[\alpha][u](x_j) \rs \notag \\
&- 2 \bbE_{S_{G, \{y_s\}, \{c_s\}}} \ls \frac{1}{n_\alpha n_u n_x} \sum_{\ell=1}^{n_\alpha} \sum_{i=1}^{n_u} \sum_{j=1}^{n_x} \hat{g}[\alpha_\ell][u_{\ell i}](x_{\ell i j}) \rs \notag \\
&= 2 \mathbb{E}_{S_{G, \{y_s\}, \{c_s\}}} \Bigg[ \mathbb{E}_{\alpha \sim \mu_\alpha} \, \mathbb{E}_{u \sim \mu_u} \, \mathbb{E}_{\{x_j\}_{j=1}^{n_x} \sim \mu_x^{\otimes n_x}} \ls
\frac{1}{n_x} \sum_{j=1}^{n_x} \hat{g}[\alpha][u](x_j) \rs \notag \\
&- \frac{1}{n_\alpha n_u n_x} \sum_{\ell=1}^{n_\alpha} \sum_{i=1}^{n_u} \sum_{j=1}^{n_x} \hat{g}[\alpha_\ell][u_{\ell i}](x_{\ell i j}) - \frac{1}{2} \mathbb{E}_{\alpha \sim \mu_\alpha} \, \mathbb{E}_{u \sim \mu_u} \, \mathbb{E}_{\{x_j\}_{j=1}^{n_x} \sim \mu_x^{\otimes n_x}} \ls
\frac{1}{n_x} \sum_{j=1}^{n_x} \hat{g}[\alpha][u](x_j) \rs \Bigg] \notag \\
&\leq 2 \mathbb{E}_{S_{G, \{y_s\}, \{c_s\}}} \Bigg[ \mathbb{E}_{\alpha \sim \mu_\alpha} \, \mathbb{E}_{u \sim \mu_u} \, \mathbb{E}_{\{x_j\}_{j=1}^{n_x} \sim \mu_x^{\otimes n_x}} \ls
\frac{1}{n_x} \sum_{j=1}^{n_x} \hat{g}[\alpha][u](x_j) \rs \notag \\
&- \frac{1}{n_\alpha n_u n_x} \sum_{\ell=1}^{n_\alpha} \sum_{i=1}^{n_u} \sum_{j=1}^{n_x} \hat{g}[\alpha_\ell][u_{\ell i}](x_{\ell i j}) - \frac{1}{8\beta_V^2} \mathbb{E}_{\alpha \sim \mu_\alpha} \, \mathbb{E}_{u \sim \mu_u} \, \mathbb{E}_{\{x_j\}_{j=1}^{n_x} \sim \mu_x^{\otimes n_x}} \ls
\frac{1}{n_x} \sum_{j=1}^{n_x} \hat{g}^2[\alpha][u](x_j) \rs \Bigg] \label{eq:T2bound:eq1} 
    \end{align}
where we used the fact that $$G_{a,I,\cF_1,\cF_2,\cF_3,S}[\bm{\alpha}][\ub](x) \in \mathrm{Cl}_a(I,\cF_1,\cF_2,\cF_3,\{y_s\},\{c_s\},P^{n_{c_W}},H^{n_{c_U}},N^{d_V})$$  
and Assumption \ref{assumption:Main:assumptions:S4} on $V$ imply that \begin{equation} \label{eq:T2bound:eq8}
  \Vert \hat{g}[\alpha][u](x) \Vert_{\Lp{\infty}(W \times U \times \Omega_V)} \leq 4 \beta_V^2,  
\end{equation}
thus $\hat{g}^2/(4\beta_V^2) \leq  \hat{g} \cdot \hat{g}/(4\beta_V^2) \leq \hat{g}$, for \eqref{eq:T2bound:eq1}.

Let $S'$ be an independent copy of $S$. We introduce the latter to symmetrize \eqref{eq:T2bound:eq1}. In particular, \begin{align}
    \mathbb{E}_{S'_{G, \{y_s\}, \{c_s\}}} \ls \frac{1}{n_\alpha n_u n_x} \sum_{\ell = 1}^{n_\alpha} \sum_{i = 1}^{n_u} \sum_{j=1}^{n_x} \hat{g}[\alpha'_\ell][u'_{\ell i}](x'_{\ell ij}) \rs &= \frac{1}{n_\alpha n_u} \sum_{\ell = 1}^{n_\alpha} \sum_{i = 1}^{n_u}  \mathbb{E}_{S'_{G, \{y_s\}, \{c_s\}}} \ls \frac{1}{n_x}\sum_{j=1}^{n_x} \hat{g}[\alpha'_\ell][u'_{\ell i}](x'_{\ell ij}) \rs \notag \\
    &= \mathbb{E}_{\alpha \sim \mu_\alpha} \, \mathbb{E}_{u \sim \mu_u} \, \mathbb{E}_{\{x_j\}_{j=1}^{n_x} \sim \mu_x^{\otimes n_x}} \ls \frac{1}{n_x} \sum_{j=1}^{n_x} \hat{g}[\alpha][u](x_j) \rs \label{eq:T2bound:eq2} 
\end{align}
where we used Definition \ref{def:trainingSet} for \eqref{eq:T2bound:eq2}. Inserting \eqref{eq:T2bound:eq2} in \eqref{eq:T2bound:eq1}, we obtain \begin{align}
    T_1 &\leq 2 \mathbb{E}_{S_{G, \{y_s\}, \{c_s\}}} \Bigg[ \mathbb{E}_{S'_{G, \{y_s\}, \{c_s\}}} \ls \frac{1}{n_\alpha n_u n_x} \sum_{\ell = 1}^{n_\alpha} \sum_{i = 1}^{n_u} \sum_{j=1}^{n_x} \hat{g}[\alpha'_\ell][u'_{\ell i}](x'_{\ell ij}) \rs \notag \\
&- \frac{1}{n_\alpha n_u n_x} \sum_{\ell=1}^{n_\alpha} \sum_{i=1}^{n_u} \sum_{j=1}^{n_x} \hat{g}[\alpha_\ell][u_{\ell i}](x_{\ell i j}) - \frac{1}{8\beta_V^2} \mathbb{E}_{S'_{G, \{y_s\}, \{c_s\}}} \ls \frac{1}{n_\alpha n_u n_x} \sum_{\ell = 1}^{n_\alpha} \sum_{i = 1}^{n_u} \sum_{j=1}^{n_x} \hat{g}^2[\alpha'_\ell][u'_{\ell i}](x'_{\ell ij}) \rs \Bigg] \notag \\
&\leq 2 \mathbb{E}_{S_{G, \{y_s\}, \{c_s\}}} \Bigg[ \sup_{g \in \shiftedClippedClass} \Bigg( \mathbb{E}_{S'_{G, \{y_s\}, \{c_s\}}} \ls \frac{1}{n_\alpha n_u n_x} \sum_{\ell = 1}^{n_\alpha} \sum_{i = 1}^{n_u} \sum_{j=1}^{n_x} g[\alpha'_\ell][u'_{\ell i}])(x'_{\ell ij}) \rs \notag \\
&- \frac{1}{n_\alpha n_u n_x} \sum_{\ell=1}^{n_\alpha} \sum_{i=1}^{n_u} \sum_{j=1}^{n_x} g[\alpha_\ell][u_{\ell i}](x_{\ell i j}) - \frac{1}{8\beta_V^2} \mathbb{E}_{S'_{G, \{y_s\}, \{c_s\}}} \ls \frac{1}{n_\alpha n_u n_x} \sum_{\ell = 1}^{n_\alpha} \sum_{i = 1}^{n_u} \sum_{j=1}^{n_x} g^2[\alpha'_\ell][u'_{\ell i}](x'_{\ell ij}) \rs \Bigg) \Bigg] \label{eq:T2bound:eq3} \\
&= 2 \mathbb{E}_{S_{G, \{y_s\}, \{c_s\}}} \Bigg[ \sup_{g \in \shiftedClippedClass} \Bigg( \mathbb{E}_{S'_{G, \{y_s\}, \{c_s\}}} \ls \frac{1}{n_\alpha n_u n_x} \sum_{\ell = 1}^{n_\alpha} \sum_{i = 1}^{n_u} \sum_{j=1}^{n_x} g[\alpha'_\ell][u'_{\ell i}](x'_{\ell ij}) -  g[\alpha_\ell][u_{\ell i}](x_{\ell i j}) \rs \notag \\ 
&- \frac{1}{16\beta_V^2} \mathbb{E}_{S'_{G, \{y_s\}, \{c_s\}}} \ls \frac{1}{n_\alpha n_u n_x} \sum_{\ell = 1}^{n_\alpha} \sum_{i = 1}^{n_u} \sum_{j=1}^{n_x} g^2[\alpha'_\ell][u'_{\ell i}](x'_{\ell ij}) \rs \notag \\
&- \frac{1}{16\beta_V^2} \mathbb{E}_{S_{G, \{y_s\}, \{c_s\}}} \ls \frac{1}{n_\alpha n_u n_x} \sum_{\ell = 1}^{n_\alpha} \sum_{i = 1}^{n_u} \sum_{j=1}^{n_x} g^2[\alpha_\ell][u_{\ell i}](x_{\ell ij}) \rs \Bigg) \Bigg] \label{eq:T2bound:eq4} \\
&= 2 \mathbb{E}_{S_{G, \{y_s\}, \{c_s\}}} \Bigg[ \sup_{g \in \shiftedClippedClass} \Bigg( \mathbb{E}_{S'_{G, \{y_s\}, \{c_s\}}} \ls \frac{1}{n_\alpha n_u n_x} \sum_{\ell = 1}^{n_\alpha} \sum_{i = 1}^{n_u} \sum_{j=1}^{n_x} g[\alpha'_\ell][u'_{\ell i}](x'_{\ell ij}) -  g[\alpha_\ell][u_{\ell i}](x_{\ell i j}) \rs  \notag \\ 
&- \frac{1}{16\beta_V^2} \mathbb{E}_{S_{G, \{y_s\}, \{c_s\}}, S'_{G, \{y_s\}, \{c_s\}}} \ls \frac{1}{n_\alpha n_u n_x} \sum_{\ell = 1}^{n_\alpha} \sum_{i = 1}^{n_u} \sum_{j=1}^{n_x} g^2[\alpha'_\ell][u'_{\ell i}](x'_{\ell ij}) + g^2[\alpha_\ell][u_{\ell i}](x_{\ell ij}) \rs \Bigg) \Bigg] \notag \\
&= 2 \mathbb{E}_{S_{G, \{y_s\}, \{c_s\}}} \Bigg[ \sup_{g \in \shiftedClippedClass} \Bigg( \mathbb{E}_{S'_{G, \{y_s\}, \{c_s\}}} \Bigg[ \frac{1}{n_\alpha n_u n_x} \sum_{\ell = 1}^{n_\alpha} \sum_{i = 1}^{n_u} \sum_{j=1}^{n_x} g[\alpha'_\ell][u'_{\ell i}](x'_{\ell ij}) -  g[\alpha_\ell][u_{\ell i}](x_{\ell i j}) \notag \\ 
&- \frac{1}{16\beta_V^2} \mathbb{E}_{S_{G, \{y_s\}, \{c_s\}}, S'_{G, \{y_s\}, \{c_s\}}} \ls \frac{1}{n_\alpha n_u n_x} \sum_{\ell = 1}^{n_\alpha} \sum_{i = 1}^{n_u} \sum_{j=1}^{n_x} g^2[\alpha'_\ell][u'_{\ell i}](x'_{\ell ij}) + g^2[\alpha_\ell][u_{\ell i}](x_{\ell ij}) \rs \Bigg] \Bigg) \Bigg] \notag \\
&\leq
2 \mathbb{E}_{S_{G, \{y_s\}, \{c_s\}},S'_{G, \{y_s\}, \{c_s\}}} \Bigg[ \sup_{g \in \shiftedClippedClass} \Bigg( \frac{1}{n_\alpha} \sum_{ \ell=1}^{n_\alpha} \Bigg( \frac{1}{n_u n_x} \sum_{i = 1}^{n_u} \sum_{j=1}^{n_x} g[\alpha'_\ell][u'_{\ell i}](x'_{\ell ij}) -  g[\alpha_\ell][u_{\ell i}](x_{\ell i j}) \notag \\ 
&- \frac{1}{16\beta_V^2} \mathbb{E}_{S_{G, \{y_s\}, \{c_s\}}, S'_{G, \{y_s\}, \{c_s\}}} \ls \frac{1}{n_u n_x} \sum_{i = 1}^{n_u} \sum_{j=1}^{n_x} g^2[\alpha'_\ell][u'_{\ell i}](x'_{\ell ij}) + g^2[\alpha_\ell][u_{\ell i}](x_{\ell ij}) \rs \Bigg)  \Bigg) \Bigg] \label{eq:T2bound:eq5}
\end{align}
where we used the facts that $G_{a,I,\cF_1,\cF_2,\cF_3,S}[\bm{\alpha}][\ub](x) \in \mathrm{Cl}_a(I,\cF_1,\cF_2,\cF_3,\{y_s\},\{c_s\},P^{n_{c_W}},H^{n_{c_U}},N^{d_V})$ and $\shiftedClippedClass = \mathrm{SCl}_a(I,\cF_1,\cF_2,\cF_3,\{y_s\},\{c_s\},P^{n_{c_W}},H^{n_{c_U}},N^{d_V})$ is the shifted clipped multiple operator network class defined in Lemma \ref{lem:coveringShifted} in \eqref{eq:T2bound:eq3}, the fact that $S'$ be an independent copy of $S$ for \eqref{eq:T2bound:eq4} and Jensen's inequality for \eqref{eq:T2bound:eq5}.

\paragraph{Step 2: Moment generating function estimation}

Let $$\coverName = \{g_k^*\}_{k=1}^{\cN\l\eta,\mathrm{SCl}_a(I,\cF_1,\cF_2,\cF_3,\{y_s\},\{c_s\},P^{n_{c_W}},H^{n_{c_U}},N^{d_V}),\Vert \cdot \Vert_{\Lp{\infty}(W \times U \times \Omega_V)}\r}$$ be the $\eta$-cover of $\shiftedClippedClass$ constructed in Lemma \ref{lem:coveringShifted}. Then, for every $g \in \shiftedClippedClass$, there exists $g^* \in \coverName$ such that $\Vert g - g^* \Vert_{\Lp{\infty}(W \times U \times \Omega_V)} \leq \eta$. Using this, we estimate as follows: \begin{align}
    g[\alpha'_\ell][u'_{\ell i}](x'_{\ell ij}) -  g[\alpha_\ell][u_{\ell i}](x_{\ell i j}) &= g[\alpha'_\ell][u'_{\ell i}][x'_{\ell ij}] -  g^*[\alpha'_\ell][u'_{\ell i}](x'_{\ell i j}) + g^*[\alpha'_\ell][u'_{\ell i}](x'_{\ell i j}) - g^*[\alpha_\ell][u_{\ell i}](x_{\ell i j}) \notag \\ 
    &+ g^*[\alpha_\ell][u_{\ell i}](x_{\ell i j}) - g[\alpha_\ell][u_{\ell i}](x_{\ell i j}) \notag \\
    &\leq 2 \Vert g - g^* \Vert_{\Lp{\infty}(W \times U \times \Omega_V)} + g^*[\alpha'_\ell][u'_{\ell i}](x'_{\ell i j}) - g^*[\alpha_\ell][u_{\ell i}](x_{\ell i j}) \notag \\
    &\leq 2 \eta + g^*[\alpha'_\ell][u'_{\ell i}](x'_{\ell i j}) - g^*[\alpha_\ell][u_{\ell i}](x_{\ell i j}) \label{eq:T2bound:eq6}
\end{align}
and \begin{align}
    &g^2[\alpha'_\ell][u'_{\ell i}](x'_{\ell ij}) + g^2[\alpha_\ell][u_{\ell i}](x_{\ell ij}) \notag \\
    &= g^2[\alpha'_\ell][u'_{\ell i}](x'_{\ell ij}) - (g^*)^2[\alpha'_\ell][u'_{\ell i}](x'_{\ell ij}) + g^2[\alpha_\ell][u_{\ell i}](x_{\ell ij}) - (g^*)^2[\alpha_\ell][u_{\ell i}](x_{\ell ij})\notag \\
    &  + (g^*)^2[\alpha_\ell][u_{\ell i}](x_{\ell ij}) + (g^*)^2[\alpha'_\ell][u'_{\ell i}](x'_{\ell ij}) \notag \\
    &\geq (g^*)^2[\alpha_\ell][u_{\ell i}](x_{\ell ij}) + (g^*)^2[\alpha'_\ell][u'_{\ell i}](x'_{\ell ij}) \notag \\
    &-\vert g[\alpha'_\ell][u'_{\ell i}](x'_{\ell ij}) - g^*[\alpha'_\ell][u'_{\ell i}](x'_{\ell ij}) \vert \cdot  \vert g[\alpha'_\ell][u'_{\ell i}](x'_{\ell ij}) + g^*[\alpha'_\ell][u'_{\ell i}](x'_{\ell ij})\vert \notag \\
    &- \vert g[\alpha_\ell][u_{\ell i}](x_{\ell ij}) - g^*[\alpha_\ell][u_{\ell i}](x_{\ell ij}) \vert \cdot \vert g[\alpha_\ell][u_{\ell i}](x_{\ell ij}) + g^*[\alpha_\ell][u_{\ell i}](x_{\ell ij}) \vert \label{eq:T2bound:eq7} \\
    &\geq (g^*)^2[\alpha_\ell][u_{\ell i}](x_{\ell ij}) + (g^*)^2[\alpha'_\ell][u'_{\ell i}](x'_{\ell ij}) - 2 \eta \l \Vert g \Vert_{\Lp{\infty}(W \times U \times \Omega_V)} + \Vert g^* \Vert_{\Lp{\infty}(W \times U \times \Omega_V)}   \r \notag \\
    &\geq (g^*)^2[\alpha_\ell][u_{\ell i}](x_{\ell ij}) + (g^*)^2[\alpha'_\ell][u'_{\ell i}](x'_{\ell ij}) - 16 \eta \beta_V^2 \label{eq:T2bound:eq9}
\end{align}
where we used the inequality $(a-b)(a+b) + \vert a-b \vert \cdot \vert a + b \vert = a^2-b^2 + \vert a-b \vert \cdot \vert a + b \vert \geq 0$ twice for \eqref{eq:T2bound:eq7} as well as \eqref{eq:T2bound:eq8} and Lemma \ref{lem:coveringShifted} for \eqref{eq:T2bound:eq9}. Inserting \eqref{eq:T2bound:eq6} and \eqref{eq:T2bound:eq9} into \eqref{eq:T2bound:eq5}, we continue estimating as follows: \begin{align}
    &T_1 \leq 2 \mathbb{E}_{S_{G, \{y_s\}, \{c_s\}},S'_{G, \{y_s\}, \{c_s\}}} \Bigg[ 3\eta +  \sup_{g \in \shiftedClippedClass} \Bigg( \frac{1}{n_\alpha} \sum_{ \ell=1}^{n_\alpha} \Bigg( \frac{1}{n_u n_x} \sum_{i = 1}^{n_u} \sum_{j=1}^{n_x} g^*[\alpha'_\ell][u'_{\ell i}](x'_{\ell ij}) -  g^*[\alpha_\ell][u_{\ell i}](x_{\ell i j}) \notag \\ 
&- \frac{1}{16\beta_V^2} \mathbb{E}_{S_{G, \{y_s\}, \{c_s\}}, S'_{G, \{y_s\}, \{c_s\}}} \ls \frac{1}{n_u n_x} \sum_{i = 1}^{n_u} \sum_{j=1}^{n_x} (g^*)^2[\alpha'_\ell][u'_{\ell i}](x'_{\ell ij}) + (g^*)^2[\alpha_\ell][u_{\ell i}](x_{\ell ij}) \rs \Bigg)  \Bigg) \Bigg] \notag\\ 
&= 6 \eta + 2 \mathbb{E}_{S_{G, \{y_s\}, \{c_s\}},S'_{G, \{y_s\}, \{c_s\}}} \Bigg[ \max_{k} \Bigg( \frac{1}{n_\alpha} \sum_{ \ell=1}^{n_\alpha} \Bigg( \frac{1}{n_u n_x} \sum_{i = 1}^{n_u} \sum_{j=1}^{n_x} g_k^*[\alpha'_\ell][u'_{\ell i}](x'_{\ell ij}) -  g_k^*[\alpha_\ell][u_{\ell i}](x_{\ell i j}) \notag \\ 
&- \frac{1}{16\beta_V^2} \mathbb{E}_{S_{G, \{y_s\}, \{c_s\}}, S'_{G, \{y_s\}, \{c_s\}}} \ls \frac{1}{n_u n_x} \sum_{i = 1}^{n_u} \sum_{j=1}^{n_x} (g_k^*)^2[\alpha'_\ell][u'_{\ell i}](x'_{\ell ij}) + (g_k^*)^2[\alpha_\ell][u_{\ell i}](x_{\ell ij}) \rs \Bigg)  \Bigg) \Bigg] \notag\\ 
&=: 6 \eta + T_2.  \label{eq:T2bound:eq10}
\end{align}

For ease of notation, we write $\mathbb{E}_{S,S'}$ for $\mathbb{E}_{S_{G, \{y_s\}, \{c_s\}},S'_{G, \{y_s\}, \{c_s\}}}$ and define \begin{align*}
    r_k(\alpha'_{\ell},\{u'_{\ell i}\}_{i=1}^{n_u},\{x'_{\ell ij}\}_{i,j=1}^{n_u,n_x},\alpha_{\ell},\{u_{\ell i}\}_{i=1}^{n_u},\{x_{\ell ij}\}_{i,j=1}^{n_u,n_x}) &= \frac{1}{n_u n_x} \sum_{i=1}^{n_u} \sum_{j=1}^{n_x} \l g_k^*[\alpha'_\ell][u'_{\ell i}](x'_{\ell i j}) - g^*_k[\alpha_\ell][u_{\ell i}](x_{\ell i j}) \r \\
    &=: r_{k\ell} 
\end{align*}
and note that, since $S$ and $S'$ are identical copies, 
\begin{equation} \label{eq:T2bound:eq17}
    \mathbb{E}_{S,S'} \ls r_{k\ell} \rs = 0.
\end{equation}
For fixed $k$ and $\ell$, we also define
\[
Z^{(k\ell)}_{ij}
:=
g_k^*[\alpha'_\ell][u'_{\ell i}](x'_{\ell ij}) - g_k^*[\alpha_\ell][u_{\ell i}](x_{\ell ij}),
\]
such that 
\(
r_{k\ell}
=
\frac{1}{n_u n_x}\sum_{i=1}^{n_u}\sum_{j=1}^{n_x} Z^{(k\ell)}_{ij}.
\)
By \eqref{eq:T2bound:eq17}, we have
\begin{align}
\Var(r_{k\ell}) &=\mathbb{E}_{S,S'}[r_{k\ell}^2]  \notag \\
&= \mathbb{E}_{S,S'}\left[
\frac{1}{(n_u n_x)^2}\sum_{i=1}^{n_u}\sum_{j=1}^{n_x} \sum_{s=1}^{n_u}\sum_{q=1}^{n_x} Z^{(k\ell)}_{ij} Z^{(k\ell)}_{sq}
\right] \notag \\
&= \mathbb{E}_{S,S'}\left[
\frac{1}{(n_u n_x)^2}\sum_{(i,j) = (s,q)} \l Z^{(k\ell)}_{ij} \r^2
\right] + \mathbb{E}_{S,S'}\left[
\frac{1}{(n_u n_x)^2}\sum_{(i,j) \neq (s,q)} Z^{(k\ell)}_{ij} Z^{(k\ell)}_{sq}
\right] \notag \\
&= \mathbb{E}_{S,S'}\left[
\frac{1}{(n_u n_x)^2}\sum_{(i,j) = (s,q)} \l Z^{(k\ell)}_{ij} \r^2
\right] + \mathbb{E}_{S,S'}\left[
\frac{1}{(n_u n_x)^2}\sum_{i=s, \, j \neq q} Z^{(k\ell)}_{ij} Z^{(k\ell)}_{iq}
\right] \notag \\
&+ \mathbb{E}_{S,S'}\left[
\frac{1}{(n_u n_x)^2}\sum_{i\neq s} Z^{(k\ell)}_{ij} Z^{(k\ell)}_{sq}
\right]. \label{eq:T2bound:newEq1}
\end{align}
For $j \neq q$, we estimate as follows: \begin{align}
    &\mathbb{E}_{S,S'}\ls Z^{(k\ell)}_{ij} Z^{(k\ell)}_{iq} \rs = \mathbb{E}_{S,S'} \ls \l g_k^*[\alpha'_\ell][u'_{\ell i}](x'_{\ell ij}) - g_k^*[\alpha_\ell][u_{\ell i}](x_{\ell ij}) \r \l g_k^*[\alpha'_\ell][u'_{\ell i}](x'_{\ell iq}) - g_k^*[\alpha_\ell][u_{\ell i}](x_{\ell iq})  \r \rs \notag \\
    &= \mathbb{E}_{\alpha_\ell,\alpha_\ell' \iid \mu_\alpha, u_{\ell i}, u_{\ell i}' \iid \mu_u} \Bigg[ \mathbb{E}_{ x_{\ell ij}, x_{\ell ij}', x_{\ell iq}, x_{\ell iq}' \iid \mu_x} \Bigg[ \l g_k^*[\alpha'_\ell][u'_{\ell i}](x'_{\ell ij}) - g_k^*[\alpha_\ell][u_{\ell i}](x_{\ell ij}) \r \notag \\
    &\times  \l g_k^*[\alpha'_\ell][u'_{\ell i}](x'_{\ell iq}) - g_k^*[\alpha_\ell][u_{\ell i}](x_{\ell iq}) \r \mid \alpha_\ell, \alpha_\ell', u_{\ell i}, u_{\ell i}' \Bigg] \Bigg] \notag \\
    &= \mathbb{E}_{\alpha_\ell,\alpha_\ell' \iid \mu_\alpha, u_{\ell i}, u_{\ell i}' \iid \mu_u} \Bigg[ \mathbb{E}_{ x_{\ell ij}, x_{\ell ij}' \iid \mu_x} \Bigg[ \l g_k^*[\alpha'_\ell][u'_{\ell i}](x'_{\ell ij}) - g_k^*[\alpha_\ell][u_{\ell i}](x_{\ell ij}) \r \mid \alpha_\ell, \alpha_\ell', u_{\ell i}, u_{\ell i}' \Bigg]  \notag \\
    &\times \mathbb{E}_{ x_{\ell iq}, x_{\ell iq}' \iid \mu_x} \Bigg[ \l g_k^*[\alpha'_\ell][u'_{\ell i}](x'_{\ell iq}) - g_k^*[\alpha_\ell][u_{\ell i}](x_{\ell iq}) \r \mid \alpha_\ell, \alpha_\ell', u_{\ell i}, u_{\ell i}' \Bigg] \Bigg] \label{eq:T2bound:newEq2} \\
    &=\mathbb{E}_{\alpha_\ell,\alpha_\ell' \iid \mu_\alpha, u_{\ell i}, u_{\ell i}' \iid \mu_u} \Bigg[ \mathbb{E}_{ x_{\ell ij}, x_{\ell ij}' \iid \mu_x} \bigg[ \l g_k^*[\alpha'_\ell][u'_{\ell i}](x'_{\ell ij}) - g_k^*[\alpha_\ell][u_{\ell i}](x_{\ell ij}) \r \mid \alpha_\ell, \alpha_\ell', u_{\ell i}, u_{\ell i}' \bigg]^2 \Bigg] \label{eq:T2bound:newEq3} \\
    &\leq \mathbb{E}_{\alpha_\ell,\alpha_\ell' \iid \mu_\alpha, u_{\ell i}, u_{\ell i}' \iid \mu_u} \Bigg[ \mathbb{E}_{ x_{\ell ij}, x_{\ell ij}' \iid \mu_x} \bigg[ \l g_k^*[\alpha'_\ell][u'_{\ell i}](x'_{\ell ij}) - g_k^*[\alpha_\ell][u_{\ell i}](x_{\ell ij}) \r^2 \mid \alpha_\ell, \alpha_\ell', u_{\ell i}, u_{\ell i}' \bigg] \Bigg] \label{eq:T2bound:newEq4} \\
    &= \mathbb{E}_{S,S'} \ls \l Z_{ij}^{(k \ell)} \r^2 \rs \label{eq:T2bound:newEq9}
\end{align}
where we used the conditional independence of $x_{\ell ij}, x_{\ell ij}'$ and $x_{\ell iq}, x_{\ell iq}'$ for \eqref{eq:T2bound:newEq2}, the fact that $x_{\ell ij}, x_{\ell ij}'$ and $x_{\ell iq}, x_{\ell iq}'$ are identically distributed for \eqref{eq:T2bound:newEq3} and Jensen's inequality for \eqref{eq:T2bound:newEq4}. Similarly, for $i \neq s$, we have \begin{align}
    &\mathbb{E}_{S,S'}\ls Z^{(k\ell)}_{ij} Z^{(k\ell)}_{sq} \rs  = \mathbb{E}_{\alpha_\ell,\alpha_\ell' \iid \mu_\alpha} \Bigg[ \mathbb{E}_{ u_{\ell i}, u_{\ell i}', u_{\ell s}, u_{\ell s}' \iid \mu_u, x_{\ell ij}, x_{\ell ij}', x_{\ell sq}, x_{\ell sq}' \iid \mu_x} \Bigg[ \l g_k^*[\alpha'_\ell][u'_{\ell i}](x'_{\ell ij}) - g_k^*[\alpha_\ell][u_{\ell i}](x_{\ell ij}) \r \notag \\
    &\times  \l g_k^*[\alpha'_\ell][u'_{\ell s}](x'_{\ell sq}) - g_k^*[\alpha_\ell][u_{\ell s}](x_{\ell sq}) \r \mid \alpha_\ell, \alpha_\ell' \Bigg] \Bigg] \notag \\
    &= \mathbb{E}_{\alpha_\ell,\alpha_\ell' \iid \mu_\alpha} \Bigg[ \mathbb{E}_{ u_{\ell i}, u_{\ell i}', u_{\ell s}, u_{\ell s}' \iid \mu_u, x_{\ell ij}, x_{\ell ij}', x_{\ell sq}, x_{\ell sq}' \iid \mu_x} \Bigg[ \l g_k^*[\alpha'_\ell][u'_{\ell i}](x'_{\ell ij}) - g_k^*[\alpha_\ell][u_{\ell i}](x_{\ell ij}) \r \mid \alpha_\ell, \alpha_\ell' \Bigg] \notag \\
    &\times  \mathbb{E}_{ u_{\ell i}, u_{\ell i}', u_{\ell s}, u_{\ell s}' \iid \mu_u, x_{\ell ij}, x_{\ell ij}', x_{\ell sq}, x_{\ell sq}' \iid \mu_x} \Bigg[ \l g_k^*[\alpha'_\ell][u'_{\ell s}](x'_{\ell sq}) - g_k^*[\alpha_\ell][u_{\ell s}](x_{\ell sq}) \r \mid \alpha_\ell, \alpha_\ell' \Bigg] \Bigg] \label{eq:T2bound:newEq5} \\
    &= \mathbb{E}_{\alpha_\ell,\alpha_\ell' \iid \mu_\alpha} \Bigg[ \mathbb{E}_{ u_{\ell i}, u_{\ell i}', u_{\ell s}, u_{\ell s}' \iid \mu_u, x_{\ell ij}, x_{\ell ij}', x_{\ell sq}, x_{\ell sq}' \iid \mu_x} \bigg[ \l g_k^*[\alpha'_\ell][u'_{\ell i}](x'_{\ell ij}) - g_k^*[\alpha_\ell][u_{\ell i}](x_{\ell ij}) \r \mid \alpha_\ell, \alpha_\ell' \bigg]^2 \Bigg] \label{eq:T2bound:newEq6} \\
    &\leq \mathbb{E}_{\alpha_\ell,\alpha_\ell' \iid \mu_\alpha} \Bigg[ \mathbb{E}_{ u_{\ell i}, u_{\ell i}', u_{\ell s}, u_{\ell s}' \iid \mu_u, x_{\ell ij}, x_{\ell ij}', x_{\ell sq}, x_{\ell sq}' \iid \mu_x} \bigg[ \l g_k^*[\alpha'_\ell][u'_{\ell i}](x'_{\ell ij}) - g_k^*[\alpha_\ell][u_{\ell i}](x_{\ell ij}) \r^2 \mid \alpha_\ell, \alpha_\ell' \bigg] \Bigg] \label{eq:T2bound:newEq7} \\
    &= \mathbb{E}_{S,S'}\ls \l Z_{ij}^{(k \ell)} \r^2 \rs \label{eq:T2bound:newEq8}
\end{align}
where we used the conditional independence of $u_{\ell i}, u_{\ell i}', x_{\ell ij}, x_{\ell ij}'$ and $u_{\ell s}, u_{\ell s}', x_{\ell sq}, x_{\ell sq}'$ for \eqref{eq:T2bound:newEq5}, the fact that $u_{\ell i}, u_{\ell i}', x_{\ell ij}, x_{\ell ij}'$ and $u_{\ell s}, u_{\ell s}', x_{\ell sq}, x_{\ell sq}'$ are identically distributed for \eqref{eq:T2bound:newEq6} and Jensen's inequality for \eqref{eq:T2bound:newEq7}. Inserting \eqref{eq:T2bound:newEq9} and \eqref{eq:T2bound:newEq8} into \eqref{eq:T2bound:newEq1}, we obtain: \begin{align}
    \Var(r_{k \ell}) &\leq \mathbb{E}_{S,S'}\left[
\frac{1}{(n_u n_x)^2}\sum_{i=1}^{n_u}\sum_{j=1}^{n_x} \sum_{s=1}^{n_u}\sum_{q=1}^{n_x} \l Z^{(k\ell)}_{ij} \r^2
\right] \notag \\
&=\mathbb{E}_{S,S'}\left[
\frac{1}{n_u n_x}\sum_{i=1}^{n_u}\sum_{j=1}^{n_x} \l Z^{(k\ell)}_{ij} \r^2
\right] \notag \\
&\leq 2 \mathbb{E}_{S,S'} \ls \frac{1}{n_u n_x} \sum_{i=1}^{n_u} \sum_{j=1}^{n_x} \l g_k^*[\alpha'_\ell][u'_{\ell i}](x'_{\ell ij})^2 + g_k^*[\alpha_\ell][u_{\ell i}](x_{\ell ij})^2 \r \rs \label{eq:T2bound:newEq10}.
\end{align}

Our aim is now to estimate $T_2$ through the moment generating function of $r_{k \ell}$. To this purpose, first re-write $T_2$ in \eqref{eq:T2bound:eq10}:  \begin{align}
    T_2 &\leq 2 \mathbb{E}_{S_{G, \{y_s\}, \{c_s\}},S'_{G, \{y_s\}, \{c_s\}}} \Bigg[ \max_{k} \Bigg( \frac{1}{n_\alpha} \sum_{\ell = 1}^{n_\alpha} \l r_{k \ell} - \frac{1}{32\beta_V^2} \mathrm{Var}(r_{k \ell}) \r  \Bigg) \Bigg] \label{eq:T2bound:eq12}
\end{align}
where we used \eqref{eq:T2bound:newEq10} for \eqref{eq:T2bound:eq12}. Then, we proceed as follows for some $t > 0$: \begin{align}
    e^{tT_2/2} &\leq \exp \l t \mathbb{E}_{S_{G, \{y_s\}, \{c_s\}},S'_{G, \{y_s\}, \{c_s\}}} \Bigg[ \max_{k} \Bigg( \frac{1}{n_\alpha} \sum_{\ell = 1}^{n_\alpha} \l r_{k \ell} - \frac{1}{32\beta_V^2} \mathrm{Var}(r_{k \ell}) \r  \Bigg) \Bigg] \r \notag \\
&\leq  \mathbb{E}_{S_{G, \{y_s\}, \{c_s\}},S'_{G, \{y_s\}, \{c_s\}}} \Bigg[ \exp \l t \max_{k} \Bigg( \frac{1}{n_\alpha} \sum_{\ell = 1}^{n_\alpha} \l r_{k \ell} - \frac{1}{32\beta_V^2} \mathrm{Var}(r_{k \ell}) \r  \Bigg)  \r \Bigg] \label{eq:T2bound:eq20} \\
&\leq \sum_k \mathbb{E}_{S_{G, \{y_s\}, \{c_s\}},S'_{G, \{y_s\}, \{c_s\}}} \Bigg[ \exp \l \frac{t}{n_\alpha} \sum_{\ell = 1}^{n_\alpha} \l r_{k \ell} - \frac{1}{32\beta_V^2} \mathrm{Var}(r_{k \ell}) \r   \r \Bigg] \notag \\
&= \sum_k \prod_{\ell = 1}^{n_\alpha} \mathbb{E}_{S_{G, \{y_s\}, \{c_s\}},S'_{G, \{y_s\}, \{c_s\}}} \ls \exp \l \frac{t}{n_{\alpha}} r_{k \ell} \r \rs \times \exp \l - \frac{t}{n_{\alpha}} \frac{1}{32\beta_V^2} \Var(r_{k \ell}) \r \label{eq:T2bound:eq13} \\
&=: \sum_k \prod_{\ell = 1}^{n_\alpha} T_3 \exp \l - \frac{t}{n_{\alpha}} \frac{1}{32\beta_V^2} \Var(r_{k \ell}) \r \notag
\end{align}
where we used Jensen's inequality for \eqref{eq:T2bound:eq20}, independence between $S$ and $S'$ as well as Definition \ref{def:trainingSet} for \eqref{eq:T2bound:eq13}. The $T_3$ term is the moment generating function we want to estimate for some $\lambda > 0$ (and fixed $\ell$): \begin{align}
\mathbb{E}_{S,S'} \ls \exp \l \frac{t}{n_{\alpha}} r_{k \ell} \r \rs &= \mathbb{E}_{S,S'} \ls 1 + \lambda r_{k \ell}  + \sum_{s=2}^{\infty} \frac{\lambda^s r_{k\ell}^s}{s!} \rs \notag\\
     &\leq \mathbb{E}_{S,S'} \Bigg[ 1 + \lambda r_{k\ell} + \lambda^2 r_{k \ell}^2 \sum_{s=2}^{\infty} \frac{\lambda^{s-2} r_{k\ell}^{s-2}}{2 \cdot 3^{s-2} } \Bigg] \label{eq:T2bound:eq14} \\
     &\leq \mathbb{E}_{S,S'} \Bigg[ 1 + \lambda r_{k\ell} + \lambda^2 r_{k\ell}^2 \sum_{s=2}^{\infty} \frac{\lambda^{s-2} (8\beta_V^2)^{s-2}}{2 \cdot 3^{s-2} } \Bigg] \label{eq:T2bound:eq15} \\ 
     &= \mathbb{E}_{S,S'} \Bigg[ 1 + \lambda r_{k\ell} + \frac{\lambda^2 r_{k\ell}^2}{2} \sum_{s=2}^{\infty} \l \frac{8 \lambda \beta_V^2}{3}\r^{s-2} \Bigg] \label{eq:T2bound:eq16}
\end{align}
where used the fact that $s! \geq 2 \cdot 3^{s-2}$ for $s \geq 2$ for \eqref{eq:T2bound:eq14} (which is simply shown through induction) and the fact that \begin{align}
    \Vert r_k(\alpha'_{\ell},u'_{\ell i},x'_{\ell ij},\alpha_{\ell},u_{\ell i},x_{\ell ij}) \Vert_{\Lp{\infty}((W \times U \times \Omega_V) \times (W \times U \times \Omega_V))} &\leq 2 \Vert g^*_k \Vert_{\Lp{\infty}(W \times U \times \Omega_V)}  \leq 8 \beta_V^2 \notag 
\end{align}
by Lemma \ref{lem:coveringShifted} for \eqref{eq:T2bound:eq15}. We pick $\lambda < 3/(8\beta_V^2)$ and continue from \eqref{eq:T2bound:eq16}: \begin{align}
     \mathbb{E}_{S,S'} \Bigg[ \exp \Bigg( \lambda  r_{k\ell} \Bigg) \Bigg] &\leq \mathbb{E}_{S,S'} \Bigg[ 1 + \lambda r_{k\ell} + \frac{\lambda^2 r_{k\ell}^2}{2} \frac{1}{1-\frac{8\lambda\beta_V^2}{3}} \Bigg] \notag \\
     &= 1 + \frac{3 \lambda^2}{6-16\lambda\beta_V^2} \mathrm{Var}( r_{k\ell}) \label{eq:T2bound:eq18} \\
     &\leq \exp \l \frac{3 \lambda^2}{6-16\lambda\beta_V^2} \mathrm{Var}( r_{k\ell}) \r \label{eq:T2bound:eq19}
\end{align} 
where we used \eqref{eq:T2bound:eq17} for \eqref{eq:T2bound:eq18} and the inequality $1 + x \leq e^x$ for $x \geq 0$ for \eqref{eq:T2bound:eq19}. In order to insert \eqref{eq:T2bound:eq19} into \eqref{eq:T2bound:eq13}, we need to ensure that \[
\frac{t}{n_\alpha} \leq \frac{3}{8\beta_V^2}
\]
or, equivalently,
\begin{equation} \label{eq:T2bound:eq21}
    t \leq \frac{3n_\alpha}{8\beta_V^2}.
\end{equation}
For such $t$, using \eqref{eq:T2bound:eq19}, we continue from \eqref{eq:T2bound:eq13} to obtain 
\begin{align}
    e^{tT_2/2} \leq \sum_k  \prod_{\ell = 1}^{n_\alpha} \exp \l \mathrm{Var}( r_{k\ell}) \ls \frac{3 \l \frac{t}{n_\alpha} \r^2}{6-\frac{16t\beta_V^2}{n_\alpha }} - \frac{t}{n_\alpha 32\beta_V^2} \rs  \r. \label{eq:T2bound:eq22} 
\end{align}
Now, we pick $t$  such that \[
\frac{3 \l \frac{t}{n_\alpha} \r^2}{6-\frac{16t\beta_V^2}{n_\alpha}} - \frac{t}{n_\alpha 32\beta_V^2} = 0,
\]
or, equivalently, \[
t = \frac{3n_\alpha }{56\beta_V^2}.
\]
Since $3/8 \geq 3/56$, we note that this choice of $t$ is compatible with the requirement \eqref{eq:T2bound:eq21}. From \eqref{eq:T2bound:eq22}, we deduce that $e^{tT_2/2} \leq \sum_{k} 1 = \cN\l\eta,\mathrm{SCl}_a(I,\cF_1,\cF_2,\cF_3,\{y_s\},\{c_s\},P^{n_{c_W}},H^{n_{c_U}},N^{d_V}),\Vert \cdot \Vert_{\Lp{\infty}(W \times U \times \Omega_V)}\r$ where $\mathrm{SCl}_a(I,\cF_1,\cF_2,\cF_3,\{y_s\},\{c_s\},P^{n_{c_W}},H^{n_{c_U}},N^{d_V})$ is defined in Lemma \ref{lem:coveringShifted}.  In turn, this yields \begin{align}
    T_2 &\leq \frac{2}{t} \log\l   \cN\l\eta,\mathrm{SCl}_a(I,\cF_1,\cF_2,\cF_3,\{y_s\},\{c_s\},P^{n_{c_W}},H^{n_{c_U}},N^{d_V}),\Vert \cdot \Vert_{\Lp{\infty}(W \times U \times \Omega_V)}\r\r \notag \\
    &= \frac{112 \beta_V^2}{3n_\alpha} \log\l   \cN\l\eta,\mathrm{SCl}_a(I,\cF_1,\cF_2,\cF_3,\{y_s\},\{c_s\},P^{n_{c_W}},H^{n_{c_U}},N^{d_V}),\Vert \cdot \Vert_{\Lp{\infty}(W \times U \times \Omega_V)}\r \r. \label{eq:T2bound:eq23}
\end{align}
Substituting \eqref{eq:T2bound:eq23} in \eqref{eq:T2bound:eq10} yields the claim of the lemma.

\end{proof}

\begin{proof}[Proof of Theorem \ref{thm:scalingLawsGeneralizationError}]
    We start by decomposing the expected generalization error into a bias and a variance term. Specifically, we write: \begin{align}
        T_0 &:= \mathbb{E}_{S_{G, \{y_s\}, \{c_s\}}} \mathbb{E}_{\alpha \sim \mu_\alpha} \, \mathbb{E}_{u \sim \mu_u} \, \mathbb{E}_{\{x_j\}_{j=1}^{n_x} \sim \mu_x^{\otimes n_x}} \left[
\frac{1}{n_x} \sum_{j=1}^{n_x} \left( G_{a,I,\cF_1,\cF_2,\cF_3,S}[\bm{\alpha}][\ub](x_j) - G[\alpha][u](x_j) \right)^2
\right] \notag  \\
&= 2 \bbE_{S_{G, \{y_s\}, \{c_s\}}} \ls \frac{1}{n_\alpha n_u n_x} \sum_{\ell=1}^{n_\alpha} \sum_{i=1}^{n_u} \sum_{j=1}^{n_x} \left( G_{a,I,\cF_1,\cF_2,\cF_3,S}[\bm{\alpha}_\ell][\ub_{\ell i}](x_{\ell ij}) - G[\alpha_\ell][u_{\ell i}](x_{\ell ij}) \right)^2 \rs \notag \\
&+T_0 - 2 \bbE_{S_{G, \{y_s\}, \{c_s\}}} \ls \frac{1}{n_\alpha n_u n_x} \sum_{\ell=1}^{n_\alpha} \sum_{i=1}^{n_u} \sum_{j=1}^{n_x} \left( G_{a,I,\cF_1,\cF_2,\cF_3,S}[\bm{\alpha}_\ell][\ub_{\ell i}](x_{\ell ij}) - G[\alpha_\ell][u_{\ell i}](x_{\ell ij}) \right)^2 \rs \notag \\
&=: T_1 + T_2 \label{eq:scalingLawsGeneralizationError:eq1}.
    \end{align}
    
\paragraph{Step 1: Bound on $T_1$}
We start by bounding the $T_1$ term which corresponds to the expected empirical risk, that is the average squared error of the learned operator evaluated on the training inputs, measured against the noise-free outputs. For ease of notation, we will write $ \mathrm{Cl}_a(I,\cF_1,\cF_2,\cF_3,\{y_s\},\{c_s\},P^{n_{c_W}},H^{n_{c_U}},N^{d_V}) = \clippedClass$. We estimate as follows: \begin{align}
    T_1 &= 2 \bbE_{S_{G, \{y_s\}, \{c_s\}}} \ls \frac{1}{n_\alpha n_u n_x} \sum_{\ell=1}^{n_\alpha} \sum_{i=1}^{n_u} \sum_{j=1}^{n_x} \left( G_{a,I,\cF_1,\cF_2,\cF_3,S}[\bm{\alpha}_\ell][\ub_{\ell i}](x_{\ell ij}) - w_{\ell ij} + \zeta_{\ell ij} \right)^2 \rs \label{eq:scalingLawsGeneralizationError:eq2} \\
    &= 2 \bbE_{S_{G, \{y_s\}, \{c_s\}}} \ls \frac{1}{n_\alpha n_u n_x} \sum_{\ell=1}^{n_\alpha} \sum_{i=1}^{n_u} \sum_{j=1}^{n_x} \left( G_{a,I,\cF_1,\cF_2,\cF_3,S}[\bm{\alpha}_\ell][\ub_{\ell i}](x_{\ell ij}) - w_{\ell ij} \right)^2 \rs \notag \\
    &+ 4 \bbE_{S_{G, \{y_s\}, \{c_s\}}} \ls \frac{1}{n_\alpha n_u n_x} \sum_{\ell=1}^{n_\alpha} \sum_{i=1}^{n_u} \sum_{j=1}^{n_x} \left( G_{a,I,\cF_1,\cF_2,\cF_3,S}[\bm{\alpha}_\ell][\ub_{\ell i}](x_{\ell ij}) - w_{\ell ij} \right) \zeta_{\ell ij} \rs \notag \\
    &+ 2 \bbE_{S_{G, \{y_s\}, \{c_s\}}} \ls \frac{1}{n_\alpha n_u n_x} \sum_{\ell=1}^{n_\alpha} \sum_{i=1}^{n_u} \sum_{j=1}^{n_x} \zeta_{\ell ij}^2 \rs \notag \\
    &=: 2 \bbE_{S_{G, \{y_s\}, \{c_s\}}} \ls \min_{\nn \in \clippedClass} \frac{1}{n_\alpha n_u n_x} \sum_{\ell=1}^{n_\alpha} \sum_{i=1}^{n_u} \sum_{j=1}^{n_x} \left( \nn [\bm{\alpha}_\ell][\ub_{\ell i}](x_{\ell ij}) - w_{\ell ij}  \right)^2 \rs \label{eq:scalingLawsGeneralizationError:eq3} \\
    &+ 4 \bbE_{S_{G, \{y_s\}, \{c_s\}}} \ls \frac{1}{n_\alpha n_u n_x} \sum_{\ell=1}^{n_\alpha} \sum_{i=1}^{n_u} \sum_{j=1}^{n_x} \left( G_{a,I,\cF_1,\cF_2,\cF_3,S}[\bm{\alpha}_\ell][\ub_{\ell i}](x_{\ell ij}) - w_{\ell ij}  \right) \zeta_{\ell ij} \rs + 2 T_3 \notag \\
    &\leq 2 \min_{\nn \in \clippedClass} \bbE_{S_{G, \{y_s\}, \{c_s\}}} \ls \frac{1}{n_\alpha n_u n_x} \sum_{\ell=1}^{n_\alpha} \sum_{i=1}^{n_u} \sum_{j=1}^{n_x} \l \nn [\bm{\alpha}_\ell][\ub_{\ell i}](x_{\ell ij}) - G[\alpha_\ell][u_{\ell i}](x_{\ell ij}) -\zeta_{\ell ij}  \r^2 \rs \label{eq:scalingLawsGeneralizationError:eq4} \\
    &+ 4 \bbE_{S_{G, \{y_s\}, \{c_s\}}} \ls \frac{1}{n_\alpha n_u n_x} \sum_{\ell=1}^{n_\alpha} \sum_{i=1}^{n_u} \sum_{j=1}^{n_x} \left( G_{a,I,\cF_1,\cF_2,\cF_3,S}[\bm{\alpha}_\ell][\ub_{\ell i}](x_{\ell ij}) - w_{\ell ij}  \right) \zeta_{\ell ij} \rs + 2 T_3 \notag\\
    &= 2 \min_{\nn \in \clippedClass} \bbE_{S_{G, \{y_s\}, \{c_s\}}} \ls \frac{1}{n_\alpha n_u n_x} \sum_{\ell=1}^{n_\alpha} \sum_{i=1}^{n_u} \sum_{j=1}^{n_x} \l \nn [\bm{\alpha}_\ell][\ub_{\ell i}](x_{\ell ij}) - G[\alpha_\ell][u_{\ell i}](x_{\ell ij}) \r^2 \rs \notag \\
    &-4 \min_{\nn \in \clippedClass} \bbE_{S_{G, \{y_s\}, \{c_s\}}} \ls \frac{1}{n_\alpha n_u n_x} \sum_{\ell=1}^{n_\alpha} \sum_{i=1}^{n_u} \sum_{j=1}^{n_x} \l \nn [\bm{\alpha}_\ell][\ub_{\ell i}](x_{\ell ij}) - G[\alpha_\ell][u_{\ell i}](x_{\ell ij}) \r \zeta_{\ell ij} \rs \notag \\
    &+ 4 \bbE_{S_{G, \{y_s\}, \{c_s\}}} \ls \frac{1}{n_\alpha n_u n_x} \sum_{\ell=1}^{n_\alpha} \sum_{i=1}^{n_u} \sum_{j=1}^{n_x} \left( G_{a,I,\cF_1,\cF_2,\cF_3,S}[\bm{\alpha}_\ell][\ub_{\ell i}](x_{\ell ij}) - G[\alpha_\ell][u_{\ell i}](x_{\ell ij}) - \zeta_{\ell ij}  \right) \zeta_{\ell ij} \rs + 4 T_3 \label{eq:scalingLawsGeneralizationError:eq5}\\ 
    &= 2 \min_{\nn \in \clippedClass} \bbE_{S_{G, \{y_s\}, \{c_s\}}} \ls \frac{1}{n_\alpha n_u n_x} \sum_{\ell=1}^{n_\alpha} \sum_{i=1}^{n_u} \sum_{j=1}^{n_x} \l \nn [\bm{\alpha}_\ell][\ub_{\ell i}](x_{\ell ij}) - G[\alpha_\ell][u_{\ell i}](x_{\ell ij}) \r^2 \rs \notag \\
    &-4 \min_{\nn \in \clippedClass} \bbE_{S_{G, \{y_s\}, \{c_s\}}} \ls \frac{1}{n_\alpha n_u n_x} \sum_{\ell=1}^{n_\alpha} \sum_{i=1}^{n_u} \sum_{j=1}^{n_x} \l \nn [\bm{\alpha}_\ell][\ub_{\ell i}](x_{\ell ij}) - G[\alpha_\ell][u_{\ell i}](x_{\ell ij}) \r \zeta_{\ell ij} \rs \notag \\
    &+ 4 \bbE_{S_{G, \{y_s\}, \{c_s\}}} \ls \frac{1}{n_\alpha n_u n_x} \sum_{\ell=1}^{n_\alpha} \sum_{i=1}^{n_u} \sum_{j=1}^{n_x} \left( G_{a,I,\cF_1,\cF_2,\cF_3,S}[\bm{\alpha}_\ell][\ub_{\ell i}](x_{\ell ij}) - G[\alpha_\ell][u_{\ell i}](x_{\ell ij})  \right) \zeta_{\ell ij} \rs \label{eq:scalingLawsGeneralizationError:eq6}
\end{align}
where we used the definition of $w_{\ell ij}$ for \eqref{eq:scalingLawsGeneralizationError:eq2}, Definition \ref{def:trainedOperator} for \eqref{eq:scalingLawsGeneralizationError:eq3}, Jensen's inequality and the definition of $w_{\ell ij}$ for \eqref{eq:scalingLawsGeneralizationError:eq4} as well as the definition of $w_{\ell ij}$ for \eqref{eq:scalingLawsGeneralizationError:eq5}. We continue by noting that $\nn [\bm{\alpha}_\ell][\ub_{\ell i}](x_{\ell ij})$ is independent of $\zeta_{\ell ij}$ by Definition \ref{def:trainingSet} since, for a generic $\nn \in \clippedClass$, $\nn$ is not trained (this is in contrast with the trained operator $G_{a,\cF_1,\cF_2,\cF_3,S} \in \clippedClass$ which certainly depends on $w_{\ell ij}$ and thus $\zeta_{\ell ij}$). Analogously, $G[\alpha_\ell][u_{\ell i}](x_{\ell ij})$ is also independent of $\zeta_{\ell ij}$ by Definition \ref{def:trainingSet}. Therefore, we have: \begin{align}
    &\min_{\nn \in \clippedClass} \bbE_{S_{G, \{y_s\}, \{c_s\}}} \ls \frac{1}{n_\alpha n_u n_x} \sum_{\ell=1}^{n_\alpha} \sum_{i=1}^{n_u} \sum_{j=1}^{n_x} \l \nn [\bm{\alpha}_\ell][\ub_{\ell i}](x_{\ell ij}) - G[\alpha_\ell][u_{\ell i}](x_{\ell ij}) \r \zeta_{\ell ij} \rs \notag \\
    &= \min_{\nn \in \clippedClass} \frac{1}{n_\alpha n_u n_x} \sum_{\ell=1}^{n_\alpha} \sum_{i=1}^{n_u} \sum_{j=1}^{n_x} \bbE_{S_{G, \{y_s\}, \{c_s\}}} \ls \l \nn [\bm{\alpha}_\ell][\ub_{\ell i}](x_{\ell ij}) - G[\alpha_\ell][u_{\ell i}](x_{\ell ij}) \r \rs \bbE_{S_{G, \{y_s\}, \{c_s\}}} \ls \zeta_{\ell ij} \rs \label{eq:scalingLawsGeneralizationError:eq7} \\
    &= 0 \label{eq:scalingLawsGeneralizationError:eq8}
\end{align}  
where we used the above-noted independence in \eqref{eq:scalingLawsGeneralizationError:eq7} and the fact that $\bbE_{S_{G, \{y_s\}, \{c_s\}}} \ls \zeta_{\ell ij} \rs = 0$ by Definition \ref{def:trainingSet} for \eqref{eq:scalingLawsGeneralizationError:eq8}. Inserting \eqref{eq:scalingLawsGeneralizationError:eq8} in \eqref{eq:scalingLawsGeneralizationError:eq6}, we obtain: \begin{align}
    T_1 &\leq 2 \min_{\nn \in \clippedClass} \bbE_{S_{G, \{y_s\}, \{c_s\}}} \ls \frac{1}{n_\alpha n_u n_x} \sum_{\ell=1}^{n_\alpha} \sum_{i=1}^{n_u} \sum_{j=1}^{n_x} \l \nn [\bm{\alpha}_\ell][\ub_{\ell i}](x_{\ell ij}) - G[\alpha_\ell][u_{\ell i}](x_{\ell ij}) \r^2 \rs \notag \\
    &+ 4 \bbE_{S_{G, \{y_s\}, \{c_s\}}} \ls \frac{1}{n_\alpha n_u n_x} \sum_{\ell=1}^{n_\alpha} \sum_{i=1}^{n_u} \sum_{j=1}^{n_x} \left( G_{a,I,\cF_1,\cF_2,\cF_3,S}[\bm{\alpha}_\ell][\ub_{\ell i}](x_{\ell ij}) - G[\alpha_\ell][u_{\ell i}](x_{\ell ij})  \right) \zeta_{\ell ij} \rs \notag \\
    &= 2 \min_{\nn \in \clippedClass}  \frac{1}{n_\alpha n_u} \sum_{\ell=1}^{n_\alpha} \sum_{i=1}^{n_u} \bbE_{S_{G, \{y_s\}, \{c_s\}}} \ls \frac{1}{n_x} \sum_{j=1}^{n_x} \l \nn [\bm{\alpha}_\ell][\ub_{\ell i}](x_{\ell ij}) - G[\alpha_\ell][u_{\ell i}](x_{\ell ij}) \r^2 \rs \notag \\
    &+ 4 \bbE_{S_{G, \{y_s\}, \{c_s\}}} \ls \frac{1}{n_\alpha n_u n_x} \sum_{\ell=1}^{n_\alpha} \sum_{i=1}^{n_u} \sum_{j=1}^{n_x} \left( G_{a,I,\cF_1,\cF_2,\cF_3,S}[\bm{\alpha}_\ell][\ub_{\ell i}](x_{\ell ij}) - G[\alpha_\ell][u_{\ell i}](x_{\ell ij})  \right) \zeta_{\ell ij} \rs \notag \\ 
    &= 2 \min_{\nn \in \clippedClass}  \frac{1}{n_\alpha n_u} \sum_{\ell=1}^{n_\alpha} \sum_{i=1}^{n_u} \mathbb{E}_{\alpha \sim \mu_\alpha} \, \mathbb{E}_{u \sim \mu_u} \, \mathbb{E}_{\{x_j\}_{j=1}^{n_x} \sim \mu_x^{\otimes n_x}} \ls \frac{1}{n_x} \sum_{j=1}^{n_x} \l \nn [\bm{\alpha}][\ub](x_{j}) - G[\alpha][u](x_{j}) \r^2 \rs \label{eq:scalingLawsGeneralizationError:eq9} \\
    &+ 4 \bbE_{S_{G, \{y_s\}, \{c_s\}}} \ls \frac{1}{n_\alpha n_u n_x} \sum_{\ell=1}^{n_\alpha} \sum_{i=1}^{n_u} \sum_{j=1}^{n_x} \left( G_{a,I,\cF_1,\cF_2,\cF_3,S}[\bm{\alpha}_\ell][\ub_{\ell i}](x_{\ell ij}) - G[\alpha_\ell][u_{\ell i}](x_{\ell ij})  \right) \zeta_{\ell ij} \rs \notag \\
    &= 2 \min_{\nn \in \clippedClass}  \mathbb{E}_{\alpha \sim \mu_\alpha} \, \mathbb{E}_{u \sim \mu_u} \, \mathbb{E}_{\{x_j\}_{j=1}^{n_x} \sim \mu_x^{\otimes n_x}} \ls \frac{1}{n_x} \sum_{j=1}^{n_x} \l \nn [\bm{\alpha}][\ub](x_{j}) - G[\alpha][u](x_{j}) \r^2 \rs \notag \\
    &+ 4 \bbE_{S_{G, \{y_s\}, \{c_s\}}} \ls \frac{1}{n_\alpha n_u n_x} \sum_{\ell=1}^{n_\alpha} \sum_{i=1}^{n_u} \sum_{j=1}^{n_x} \left( G_{a,I,\cF_1,\cF_2,\cF_3,S}[\bm{\alpha}_\ell][\ub_{\ell i}](x_{\ell ij}) - G[\alpha_\ell][u_{\ell i}](x_{\ell ij})  \right) \zeta_{\ell ij} \rs \notag \\
    &=: T_4 + T_5 \label{eq:scalingLawsGeneralizationError:T4T5}
\end{align}
where we used the facts that $\alpha_\ell \iid \mu_\alpha$, $u_{\ell i} \iid \mu_u$ and $x_{\ell ij} \iid \mu_x$ by Definition \ref{def:trainingSet} for \eqref{eq:scalingLawsGeneralizationError:eq9}. For $T_4$, we estimate as follows: \begin{align}
    T_4 &\leq 2 \sup_{\alpha \in W} \sup_{u \in U} \sup_{x \in \Omega_V} \vert \nn[\bm{\alpha}][\bm{u}](x) - G[\alpha][u](x) \vert^2 \label{eq:scalingLawsGeneralizationError:eq10} \\
    &\leq 2 \eps^2 \label{eq:scalingLawsGeneralizationError:eq11}
\end{align}
where $\nn \in \clippedClass$ in \eqref{eq:scalingLawsGeneralizationError:eq10} is the network such that \eqref{eq:cor:back:clippedScalingLaws} holds -- the latter exists by Corollary \ref{cor:back:clippedScalingLaws} and Remark \ref{rem:eps2}. 
Combining \eqref{eq:scalingLawsGeneralizationError:eq11} with Lemma \ref{lem:main:T5Bound} (and using the notation from the latter) for $T_5$, we obtain that \begin{align}
    T_1 &= \bbE_{S_{G, \{y_s\}, \{c_s\}}} \ls \empiricalEvaluation(G_{a,I,\cF_1,\cF_2,\cF_3,S} - G)^2 \rs \notag \\
    &\leq 2 \eps^2 + 4\eta \sigma +  \frac{4\sigma}{\sqrt{n_\alpha n_u n_x}} \l \eta +  \sqrt{\bbE_{S_{G, \{y_s\}, \{c_s\}}} \ls \empiricalEvaluation(G_{a,I,\cF_1,\cF_2,\cF_3,S} -G)^2 \rs} \r \notag \\
    &\times \sqrt{\log \l \cN\l\eta,\mathrm{Cl}_a(I,\cF_1,\cF_2,\cF_3,\{y_s\},\{c_s\},P^{n_{c_W}},H^{n_{c_U}},N^{d_V}),\Vert \cdot \Vert_{\Lp{\infty}(W \times U \times \Omega_V)}\r \r + \log(2)}. \notag
\end{align} 
As in the proof of \cite[Theorem 2]{liu2024neuralscalinglawsdeep}, the latter can be equivalently written as
\(
\rho^2 \leq c + 2b\,\rho
\)
with $$\rho = \sqrt{\bbE_{S_{G, \{y_s\}, \{c_s\}}} \ls \empiricalEvaluation(G_{a,I,\cF_1,\cF_2,\cF_3,S} - G)^2\rs},$$ \begin{align}
    c &= 2 \eps^2 + 4 \eta \sigma \notag \\
    &+ \frac{4\sigma \eta}{\sqrt{n_\alpha n_u n_x}} \sqrt{\log \l \cN\l\eta,\mathrm{Cl}_a(I,\cF_1,\cF_2,\cF_3,\{y_s\},\{c_s\},P^{n_{c_W}},H^{n_{c_U}},N^{d_V}),\Vert \cdot \Vert_{\Lp{\infty}(W \times U \times \Omega_V)}\r \r + \log(2)} \notag
\end{align}
and $$b = \frac{2\sigma}{\sqrt{n_\alpha n_u n_x}} \sqrt{\log \l \cN\l\eta,\mathrm{Cl}_a(I,\cF_1,\cF_2,\cF_3,\{y_s\},\{c_s\},P^{n_{c_W}},H^{n_{c_U}},N^{d_V}),\Vert \cdot \Vert_{\Lp{\infty}(W \times U \times \Omega_V)}\r \r + \log(2)}.$$
Rearranging terms yields
\(
(\rho - b)^2 \leq c + b^2,
\)
which implies
\(
\rho \leq \sqrt{c + b^2} + b.
\)
Consequently, we obtain the bound
\(
\rho^2 \leq b^2 + c + b^2 + 2b\sqrt{c+b^2} \leq 2b^2 + c + b^2 + c + b^2 = 2 c + 4b^2
\) (where the last inequality follows from $2pq \leq p^2 + q^2$)
or, equivalently, \begin{align}
    T_1 &\leq  4 \eps^2 + 8 \eta \sigma \notag \\
    &+ \frac{8\sigma \eta}{\sqrt{n_\alpha n_u n_x}} \sqrt{\log \l \cN\l\eta,\mathrm{Cl}_a(I,\cF_1,\cF_2,\cF_3,\{y_s\},\{c_s\},P^{n_{c_W}},H^{n_{c_U}},N^{d_V}),\Vert \cdot \Vert_{\Lp{\infty}(W \times U \times \Omega_V)}\r \r + \log(2)} \notag \\
    &+ \frac{16\sigma^2}{n_\alpha n_u n_x} \l \log \l \cN\l\eta,\mathrm{Cl}_a(I,\cF_1,\cF_2,\cF_3,\{y_s\},\{c_s\},P^{n_{c_W}},H^{n_{c_U}},N^{d_V}),\Vert \cdot \Vert_{\Lp{\infty}(W \times U \times \Omega_V)}\r \r + \log(2) \r. \label{eq:scalingLawsGeneralizationError:T1}
\end{align}

\paragraph{Step 2: Bound on $T_2$} By Lemma \ref{lem:T2bound}, we obtain that \begin{align} 
    T_2 \leq 6 \eta + \frac{112\beta_V^2}{3n_\alpha}  \log\l   \cN\l\eta,\mathrm{SCl}_a(I,\cF_1,\cF_2,\cF_3,\{y_s\},\{c_s\},P^{n_{c_W}},H^{n_{c_U}},N^{d_V}),\Vert \cdot \Vert_{\Lp{\infty}(W \times U \times \Omega_V)}\r \r, \notag
\end{align}
so combining the latter with \eqref{eq:scalingLawsGeneralizationError:eq1} and \eqref{eq:scalingLawsGeneralizationError:T1} yields: \begin{align}
    &T_0 \leq  4 \eps^2 + 8 \eta \sigma \notag \\
    &+ \frac{8\sigma \eta}{\sqrt{n_\alpha n_u n_x}} \sqrt{\log \l \cN\l\eta,\mathrm{Cl}_a(I,\cF_1,\cF_2,\cF_3,\{y_s\},\{c_s\},P^{n_{c_W}},H^{n_{c_U}},N^{d_V}),\Vert \cdot \Vert_{\Lp{\infty}(W \times U \times \Omega_V)}\r \r + \log(2)} \notag \\
    &+ \frac{16\sigma^2}{n_\alpha n_u n_x} \l \log \l \cN\l\eta,\mathrm{Cl}_a(I,\cF_1,\cF_2,\cF_3,\{y_s\},\{c_s\},P^{n_{c_W}},H^{n_{c_U}},N^{d_V}),\Vert \cdot \Vert_{\Lp{\infty}(W \times U \times \Omega_V)}\r \r + \log(2) \r \notag \\
    &+ 6 \eta + \frac{112 \beta_V^2}{3n_\alpha} \log\l   \cN\l\eta,\mathrm{SCl}_a(I,\cF_1,\cF_2,\cF_3,\{y_s\},\{c_s\},P^{n_{c_W}},H^{n_{c_U}},N^{d_V}),\Vert \cdot \Vert_{\Lp{\infty}(W \times U \times \Omega_V)}\r \r \notag \\
    &\leq 4 \eps^2 + \eta (8 \sigma + 6) \notag \\
    &+ \frac{8\sigma \eta}{\sqrt{n_\alpha n_u n_x}} \sqrt{\log \l \cN\l\eta,\mathrm{Cl}_a(I,\cF_1,\cF_2,\cF_3,\{y_s\},\{c_s\},P^{n_{c_W}},H^{n_{c_U}},N^{d_V}),\Vert \cdot \Vert_{\Lp{\infty}(W \times U \times \Omega_V)}\r \r + \log(2)} \notag \\
    &+ \frac{16\sigma^2}{n_\alpha n_u n_x} \l \log \l \cN\l\eta,\mathrm{Cl}_a(I,\cF_1,\cF_2,\cF_3,\{y_s\},\{c_s\},P^{n_{c_W}},H^{n_{c_U}},N^{d_V}),\Vert \cdot \Vert_{\Lp{\infty}(W \times U \times \Omega_V)}\r \r + \log(2) \r \notag \\
    &+ \frac{112 \beta_V^2}{3n_\alpha} \log\l   \cN\l\eta/(4\beta_V),\mathrm{Cl}_a(I,\cF_1,\cF_2,\cF_3,\{y_s\},\{c_s\},P^{n_{c_W}},H^{n_{c_U}},N^{d_V}),\Vert \cdot \Vert_{\Lp{\infty}(W \times U \times \Omega_V)}\r \r \label{eq:scalingLawsGeneralizationError:eq12}
\end{align}
where $\mathrm{SCl}_a(I,\cF_1,\cF_2,\cF_3,\{y_s\},\{c_s\},P^{n_{c_W}},H^{n_{c_U}},N^{d_V})$ is defined in Lemma \ref{lem:coveringShifted} and we used the latter for \eqref{eq:scalingLawsGeneralizationError:eq12}.
    
\end{proof}

\begin{proof}[Proof of Corollary \ref{cor:covering:eps}]
We start by estimating: \begin{align}
    F(L,p,K,\kappa,h) &=  \binom{L(p^2 + p)}{K} \left( \left\lfloor \frac{2\kappa}{h} \right\rfloor + 1 \right)^K \notag \\
    &\leq (L(p^2+p))^K \l \frac{3\kappa}{h} \r^K \label{eq:cor:coveringEps:eq1} \\
    &\lesssim \l \frac{Lp^2\kappa T}{\eta} \r^K \label{eq:cor:coveringEps:eq2}
\end{align}
where we used the inequality $\binom{n}{k} \leq n^k$ for \eqref{eq:cor:coveringEps:eq1} and the definition of $h$ in Proposition \ref{prop:back:coveringClippledMultipleOperator} for \eqref{eq:cor:coveringEps:eq2}. We continue as follows: \begin{align}
    &\log \l \cN\l\eta,\mathrm{Cl}_a(I,\cF_1,\cF_2,\cF_3,\{y_s\},\{c_s\},P^{n_{c_W}},H^{n_{c_U}},N^{d_V}),\Vert \cdot \Vert_{\Lp{\infty}(W \times U \times \Omega_V)}\r \r \notag \\
    &\leq  P^{n_{c_W}} H^{n_{c_U}} N^{d_V} \ls \log \l 3I/h \r + K_3 \log \l L_3 p_3^2 \kappa_3 T \r + K_2 \log \l L_2 p_2^2 \kappa_2 T \r +K_1 \log \l L_1 p_1^2 \kappa_1 T \r \rs \label{eq:cor:coveringEps:eq3} \\
    &\lesssim   P^{n_{c_W}} H^{n_{c_U}} N^{d_V} \ls \log \l \frac{T}{\eta} \r + K_3 \log \l \frac{L_3 \kappa_3 T}{\eta} \r + K_2 \log \l \frac{L_2 \kappa_2 T}{\eta} \r +K_1 \log \l \frac{L_1 \kappa_1 T}{\eta} \r \rs \label{eq:cor:coveringEps:eq4}
\end{align}
where we used Proposition \ref{prop:back:coveringClippledMultipleOperator} and \eqref{eq:cor:coveringEps:eq2} for \eqref{eq:cor:coveringEps:eq3} as well as the fact that $p_i = \mathcal{O}(1)$ for $1 \leq i \leq 3$ by Theorem \ref{thm:scalingLawsGeneralizationError} for \eqref{eq:cor:coveringEps:eq4}.

We now compute the asymptotic scaling of the all the above constants as a function of $\eps$. In particular, we have \begin{itemize}
    \item $n_{c_W} \lesssim \eps^{-d_W} $;
    \item $L_3 \asymp K_3 \lesssim (1+d_W) \eps^{-2d_W} \log(\eps^{-1})$;
    \item for $\kappa_3$, we first consider the logarithm: \begin{align}
        \log(\kappa_3) &\lesssim \l \frac{n_{c_W}}{2} + 1 \r \log(n_{c_W}) + (1 + n_{c_W}) \log(\eps^{-1}) \notag \\
        &\lesssim \l 1 + \frac{d_W}{2} \r \eps^{-d_W} \log(\eps^{-1}) \notag
    \end{align}
    which yields $\kappa_3 \lesssim \eps^{-\eps^{-d_W}\l 1 + \frac{d_W}{2} \r}$;
    \item for $P^{n_{c_W}}$, we first consider the logarithm: \begin{align}
        \log( P^{n_{c_W}} ) &\lesssim \frac{n_{c_W}}{2} \log(n_{c_W}) + n_{c_W} \log(\eps^{-1}) \notag \\
        &\lesssim \l 1 + \frac{d_W}{2} \r \eps^{-d_W} \log(\eps^{-1}) \notag
    \end{align}
    which yields $P^{n_{c_W}} \lesssim \eps^{-\eps^{-d_W}\l 1 + \frac{d_W}{2} \r}$;
    \item $n_{c_U} \lesssim \eps^{-\eps^{-d_W}d_U(1 + d_V)\l 1 + \frac{d_W}{2}\r}$ (as in \cite[Remark 3.17]{weihs2025MOL})
    \item $L_2 \asymp K_2 \lesssim \eps^{-2\eps^{-d_W}d_U(1 + d_V)\l 1 + \frac{d_W}{2}\r} -d_W\log(\eps^{-1}) (1 + d_U)(1 + d_V)\l 1 + \frac{d_W}{2}\r$ (as in \cite[Remark 3.17]{weihs2025MOL});
    \item for $\kappa_2$, we first consider the logarithm: \begin{align}
        \log( \kappa_2 ) &\lesssim \l 1 + \frac{n_{c_U}}{2} \r \log(n_{c_U}) + (d_V+1)(n_{c_U} + 1)(n_{c_W}+1) \log(\eps^{-1}) \notag \\
        &\lesssim \frac{n_{c_U}}{2} \log(n_{c_U}) + n_{c_U} n_{c_W} \log(\eps^{-1}) \notag \\
        &\lesssim \eps^{-\eps^{-d_W}d_U(1 + d_V)\l 1 + \frac{d_W}{2}\r} \ls \eps^{-d_W}d_U(1 + d_V)\l 1 + \frac{d_W}{2}\r \log(\eps^{-1}) + \eps^{-d_W} \log(\eps^{-1}) \rs  \notag \\
        &=  \eps^{-\eps^{-d_W}d_U(1 + d_V)\l 1 + \frac{d_W}{2}\r} \eps^{-d_W} \log(\eps^{-1}) \ls d_U(1 + d_V)\l 1 + \frac{d_W}{2}\r + 1\rs \notag
    \end{align}
    which yields $\kappa_2 \lesssim \eps^{-\eps^{-\eps^{-d_W}d_U(1 + d_V)\l 1 + \frac{d_W}{2}\r} \eps^{-d_W} \ls d_U(1 + d_V)\l 1 + \frac{d_W}{2}\r + 1\rs}$;
    \item $H^{n_{c_U}} \lesssim \eps^{-\eps^{-\eps^{-d_W}d_U(1 + d_V)\l 1 + \frac{d_W}{2}\r -d_W} \ls  d_U(1 + d_V)\l 1 + \frac{d_W}{2}\r + d_W\frac{(d_V+1)}{2} + (d_V+1)  \rs }$ (as in \cite[Remark 3.17]{weihs2025MOL});
    \item $K_1 \asymp L_1 \lesssim \eps^{-d_W}\log(\eps^{-1})d_V^2 \l 1 + \frac{d_W}{2} \r$
    \item for $\kappa_1$, we first consider the logarithm: \begin{align}
        \log(\kappa_1) &\lesssim (1+d_V)(1+n_{c_W})\log(\eps^{-1}) + (1 + d_V)\frac{n_{c_W}}{2} \log(n_{c_W})\notag \\
        &\lesssim (1 + d_V) \eps^{-d_W} \log(\eps^{-1}) + (1+d_V)\frac{d_W}{2}\eps^{-d_W}\log(\eps^{-1}) \notag \\
        &\lesssim (1 + d_V)(1 + \frac{d_W}{2}) \eps^{-d_W} \log(\eps^{-1}) \notag
    \end{align}
    which yields $\kappa_1 \lesssim \eps^{-\eps^{-d_W}(1 + d_V)(1 + \frac{d_W}{2})}$;
    \item $N^{d_V} \lesssim \eps^{-\eps^{-d_W}d_V\l \frac{d_W}{2} +1 \r}$ (as in \cite[Remark 3.17]{weihs2025MOL}).
\end{itemize}
Using the above-derived formulas, we continue from \eqref{eq:cor:coveringEps:eq4} to derive that \begin{align}
   &\log \l \cN\l\eta,\mathrm{Cl}_a(I,\cF_1,\cF_2,\cF_3,\{y_s\},\{c_s\},P^{n_{c_W}},H^{n_{c_U}},N^{d_V}),\Vert \cdot \Vert_{\Lp{\infty}(W \times U \times \Omega_V)}\r \r \notag \\
   &\lesssim P^{n_{c_W}} H^{n_{c_U}} N^{d_V} \ls  K_2 \log \l \frac{L_2 \kappa_2 T}{\eta} \r  \rs. \notag
\end{align}
Taking logarithms on the latter, and relying on the formulas again, we have \begin{align}
    &\log \l \log \l \cN\l\eta,\mathrm{Cl}_a(I,\cF_1,\cF_2,\cF_3,\{y_s\},\{c_s\},P^{n_{c_W}},H^{n_{c_U}},N^{d_V}),\Vert \cdot \Vert_{\Lp{\infty}(W \times U \times \Omega_V)}\r \r \r \notag \\
    &\lesssim \log(H^{n_{c_U}}) + \log(K_2)  + \log \l \log \l \frac{L_2 \kappa_2 T}{\eta} \r \r \notag
\end{align}
from which we deduce that \begin{align}
    &\log \l \cN\l\eta,\mathrm{Cl}_a(I,\cF_1,\cF_2,\cF_3,\{y_s\},\{c_s\},P^{n_{c_W}},H^{n_{c_U}},N^{d_V}),\Vert \cdot \Vert_{\Lp{\infty}(W \times U \times \Omega_V)}\r \r \notag \\
    &\lesssim  H^{n_{c_U}} \log \l \frac{L_2 \kappa_2 T}{\eta} \r
    \label{eq:cor:coveringEps:eq5} \\
    &\lesssim H^{n_{c_U}} \log(L_2 \kappa_2 T ) + H^{n_{c_U}} \log(\eta^{-1}). \label{eq:cor:coveringEps:eq9}
\end{align}

We now consider $T$ and use the asymptotic formulas to estimate as follows: \begin{align}
    T &\lesssim P^{n_{c_W}} H^{n_{c_U}} N^{n_{c_U}}  \Bigg[  L_1\kappa_1^{L_1-1} + L_2\kappa_2^{L_2-1} +  L_3\kappa_3^{L_3-1} \Bigg] \label{eq:cor:coveringEps:eq6} \\
    &\lesssim  H^{n_{c_U}} L_2\kappa_2^{L_2-1} \label{eq:cor:coveringEps:eq7} \\
    &\lesssim H^{n_{c_U}} \kappa_2^{L_2} \label{eq:cor:coveringEps:eq8}
\end{align}
where we use the fact that $R_i = 1$ and $p_i = \mathcal{O}(1)$ for $1 \leq i \leq 3$ by Theorem \ref{thm:scalingLawsGeneralizationError} for \eqref{eq:cor:coveringEps:eq6} and proceed in analagous fashion to how we obtained \eqref{eq:cor:coveringEps:eq5} for \eqref{eq:cor:coveringEps:eq7} and \eqref{eq:cor:coveringEps:eq8}.
Writing \[
H^{n_{c_U}} \lesssim \eps^{-\eps^{-\eps^{-d_W}d_U(1 + d_V)\l 1 + \frac{d_W}{2}\r -d_W} \ls  d_U(1 + d_V)\l 1 + \frac{d_W}{2}\r + d_W\frac{(d_V+1)}{2} + (d_V+1)  \rs } \lesssim: \eps^{-\delta_1 \eps^{-\delta_2\eps^{-d_W}}},
\] 
\begin{align*}
   L_2 &\lesssim \eps^{-2\eps^{-d_W}d_U(1 + d_V)\l 1 + \frac{d_W}{2}\r} -d_W\log(\eps^{-1}) (1 + d_U)(1 + d_V)\l 1 + \frac{d_W}{2}\r \\
   &\lesssim: \delta_3 \eps^{-\delta_4 \eps^{-d_W}} \log(\eps^{-1})
\end{align*}
and
\[
\kappa_2 \lesssim \eps^{-\eps^{-\eps^{-d_W}d_U(1 + d_V)\l 1 + \frac{d_W}{2}\r} \eps^{-d_W} \ls d_U(1 + d_V)\l 1 + \frac{d_W}{2}\r + 1\rs} \lesssim: \eps^{-\delta_5 \eps^{-\delta_6\eps^{-d_W}}},
\]
from \eqref{eq:cor:coveringEps:eq8}, we obtain that 
\begin{align}
    T &\lesssim \eps^{-\delta_1 \eps^{-\delta_2\eps^{-d_W}}} \l \eps^{-\delta_5 \eps^{-\delta_6\eps^{-d_W}}} \r^{\delta_3 \eps^{-\delta_4 \eps^{-d_W}} \log(\eps^{-1})} \notag \\
    &\lesssim \eps^{-\delta_1 \eps^{-\delta_2\eps^{-d_W}}}  \eps^{-\delta_5 \delta_3 \eps^{-(\delta_6+\delta_4)\eps^{-d_W}} \log(\eps^{-1})} \notag \\
    &\lesssim \eps^{-\delta_1 \eps^{-\delta_2\eps^{-d_W}} -\delta_5 \delta_3 \eps^{-(\delta_6+\delta_4)\eps^{-d_W}} \log(\eps^{-1})} \notag
\end{align}
and therefore \begin{align}
    \log( L_2 \kappa_2 T) &\lesssim \log(\kappa_2 T)  \notag \\
    &\lesssim \log \l \eps^{-\delta_1 \eps^{-\delta_2\eps^{-d_W}} -\delta_5 \delta_3 \eps^{-(\delta_6+\delta_4)\eps^{-d_W}} \log(\eps^{-1}) -\delta_5 \eps^{-\delta_6\eps^{-d_W}}} \r \notag \\
    &\lesssim \l\delta_1 \eps^{-\delta_2\eps^{-d_W}} +\delta_5 \delta_3 \eps^{-(\delta_6+\delta_4)\eps^{-d_W}} \log(\eps^{-1}) +\delta_5 \eps^{-\delta_6\eps^{-d_W}} \r \log(\eps^{-1}). \label{eq:cor:coveringEps:eq10}
\end{align}
Equation \eqref{eq:cor:coveringEps:eq10} shows that $\log( L_2 \kappa_2 T)$ grows much more slowly than $H^{n_{c_U}}$ 
so, continuing from \eqref{eq:cor:coveringEps:eq9}, we have \begin{align}
    &\log \l \cN\l\eta,\mathrm{Cl}_a(I,\cF_1,\cF_2,\cF_3,\{y_s\},\{c_s\},P^{n_{c_W}},H^{n_{c_U}},N^{d_V}),\Vert \cdot \Vert_{\Lp{\infty}(W \times U \times \Omega_V)}\r \r \notag \\
    &\lesssim \eps^{-\delta_1 \eps^{-\delta_2\eps^{-d_W}}} \l 1 + \log(\eta^{-1}) \r. \notag
\end{align}

\end{proof}

\begin{proof}[Proof of Corollary \ref{cor:boundEps}]
In the proof $C>0$ will denote a constant that can be arbitrarily large, is independent of the $\eta, n_\alpha, n_u, n_x$, and that may change from line to line. 

    From Theorem \ref{thm:scalingLawsGeneralizationError}, we want to express \begin{align}
    &4 \eps^2 + \eta (8 \sigma + 6) \notag \\
    &+ \frac{8\sigma \eta}{\sqrt{n_\alpha n_u n_x}} \sqrt{\log \l \cN\l\eta,\mathrm{Cl}_a(I,\cF_1,\cF_2,\cF_3,\{y_s\},\{c_s\},P^{n_{c_W}},H^{n_{c_U}},N^{d_V}),\Vert \cdot \Vert_{\Lp{\infty}(W \times U \times \Omega_V)}\r \r + \log(2)} \notag \\
    &+ \frac{16\sigma^2}{n_\alpha n_u n_x} \l \log \l \cN\l\eta,\mathrm{Cl}_a(I,\cF_1,\cF_2,\cF_3,\{y_s\},\{c_s\},P^{n_{c_W}},H^{n_{c_U}},N^{d_V}),\Vert \cdot \Vert_{\Lp{\infty}(W \times U \times \Omega_V)}\r \r + \log(2) \r \notag \\
    &+ \frac{112 \beta_V^2}{3n_\alpha} \log\l   \cN\l\eta/(4\beta_V),\mathrm{Cl}_a(I,\cF_1,\cF_2,\cF_3,\{y_s\},\{c_s\},P^{n_{c_W}},H^{n_{c_U}},N^{d_V}),\Vert \cdot \Vert_{\Lp{\infty}(W \times U \times \Omega_V)}\r \r \notag \\
    &=:T \notag
    \end{align}
as a function of $\eps$. For ease of notation, we write 
\[
\cN(\eta) := \cN\l\eta,\mathrm{Cl}_a(I,\cF_1,\cF_2,\cF_3,\{y_s\},\{c_s\},P^{n_{c_W}},H^{n_{c_U}},N^{d_V}),\Vert \cdot \Vert_{\Lp{\infty}(W \times U \times \Omega_V)}\r.
\]
We estimate as follows: \begin{align}
    T&\leq 4 \eps^2 + \eta (8 \sigma + 6) + \frac{8\sigma \eta}{\sqrt{n_\alpha n_u n_x}} \sqrt{\log \l \cN(\eta/(4\beta_V)) \r + \log(2)} + \frac{16\sigma^2}{n_\alpha n_u n_x} \l \log \l \cN\l \eta/(4\beta_V)\r \r + \log(2) \r \notag \\
    &+ \frac{112 \beta_V^2}{3 n_\alpha} \log(\cN(\eta/(4\beta_V))) \label{eq:cor:boundEps:eq1} \\
    &\lesssim 4 \eps^2 + \eta (8 \sigma + 6) + \frac{8\sigma \eta}{\sqrt{n_\alpha n_u n_x}} \sqrt{\log \l \cN(\eta/(4\beta_V)) \r} + \frac{16\sigma^2}{n_\alpha n_u n_x}  \log \l \cN\l \eta/(4\beta_V)\r \r  \notag \\
    &+ \frac{112 \beta_V^2}{3 n_\alpha} \log(\cN(\eta/(4\beta_V))) \notag \\
    &\lesssim 4 \eps^2 + \eta (8 \sigma + 6) + \frac{8\sigma \eta}{\sqrt{n_\alpha n_u n_x}} \eps^{-(\delta_1/2) \eps^{-\delta_2\eps^{-d_W}}} \l 1 + \log(4\beta_V \eta^{-1}) \r \notag \\
    &+ \frac{16\sigma^2}{n_\alpha n_u n_x} \eps^{-\delta_1 \eps^{-\delta_2\eps^{-d_W}}} \l 1 + \log(4\beta_V \eta^{-1}) \r + \frac{112 \beta_V^2}{3n_\alpha} \eps^{-\delta_1 \eps^{-\delta_2\eps^{-d_W}}} \l 1 + \log(4\beta_V \eta^{-1}) \r  \label{eq:cor:boundEps:eq2}
\end{align}
where we used the fact that $\cN(\eta) \leq \cN(\tilde{\eta})$ if $\tilde{\eta} \leq \eta$ for \eqref{eq:cor:boundEps:eq1} and Corollary \ref{cor:covering:eps} for \eqref{eq:cor:boundEps:eq2}. By picking $\eta = 4\beta_V n_\alpha^{-1}$,  we ensure that the $\eta$-dependent terms are of comparable scale to the $n_\alpha^{-1}$-terms and, continuing from \eqref{eq:cor:boundEps:eq2}, we obtain: \begin{align}
    T&\lesssim 4 \eps^2 + \frac{C}{n_\alpha} + \frac{C}{n_\alpha^{3/2} (n_u n_x)^{1/2}} \eps^{-(\delta_1/2) \eps^{-\delta_2\eps^{-d_W}}}  \log(n_\alpha)  \notag \\
    &+ \frac{C}{n_\alpha n_u n_x} \eps^{-\delta_1 \eps^{-\delta_2\eps^{-d_W}}} \log(n_\alpha) + \frac{C}{n_\alpha} \eps^{-\delta_1 \eps^{-\delta_2\eps^{-d_W}}} \log(n_\alpha) \notag \\
    &\lesssim 4 \eps^2 + \frac{C}{n_\alpha} \eps^{-\delta_1 \eps^{-\delta_2\eps^{-d_W}}} \log(n_\alpha) \notag  \\
    &=: T_1 + T_2 \label{eq:cor:boundEps:eq7}.
\end{align}

We now balance the last two terms $T_1$ and $T_2$. 
Our goal is to choose \(\varepsilon = \varepsilon(n_\alpha)\) so that the second term is of (at most) the same order as the first.
Motivated by the structure of the exponent, we pick
\begin{equation}\label{eq:eps-choice}
    \varepsilon
    :=
    \left(
        \frac{d_W}{2\delta_2}
        \frac{\log\log n_\alpha}{\log\log\log n_\alpha}
    \right)^{-\frac{1}{d_W}}.
\end{equation}
We now compute the corresponding asymptotic scales. By definition \eqref{eq:eps-choice},
\begin{align}
    \log(\eps^{-1})
    &= \frac{1}{d_W}
       \log\left(
            \frac{d_W}{2\delta_2}
            \frac{\log\log n_\alpha}{\log\log\log n_\alpha}
       \right) \notag \\
    &= \frac{1}{d_W}
       \Bigl(
         \log\frac{d_W}{2\delta_2}
         + \log\log\log n_\alpha
         - \log\log\log\log n_\alpha
       \Bigr)  \notag \\
       & \lesssim
    \frac{1}{d_W}\,\log\log\log n_\alpha. \label{eq:cor:boundEps:eq3}
\end{align}
We next analyze the intermediate exponential
\begin{align}
    \varepsilon^{-\delta_2 \varepsilon^{-d_W}} &= \exp\l\delta_2 \varepsilon^{-d_W} \log\frac{1}{\varepsilon}\r \notag \\
    &\lesssim \exp \l \delta_2 \frac{d_W}{2\delta_2}
        \frac{\log\log n_\alpha}{\log\log\log n_\alpha}  \frac{\log \log \log n_\alpha}{d_W} \r \label{eq:cor:boundEps:eq4} \\
    &=\exp \l \frac{1}{2} \log\log n_\alpha \r \notag \\
    &= \log(n_\alpha)^{1/2} \label{eq:cor:boundEps:eq5}
\end{align}
where we used \eqref{eq:cor:boundEps:eq3} for \eqref{eq:cor:boundEps:eq4}. We can now estimate the complete exponential factor as follows:
\begin{align}
    \varepsilon^{-\delta_1 \varepsilon^{-\delta_2\varepsilon^{-d_W}}} &= \exp \Bigl(\delta_1 \varepsilon^{-\delta_2\varepsilon^{-d_W}} \log\tfrac{1}{\varepsilon}\Bigr) \notag \\
    &\lesssim \exp \l \frac{\delta_1}{d_W} \log(n_\alpha)^{1/2} \log \log \log n_\alpha  \r \label{eq:cor:boundEps:eq6}
\end{align}
where we used \eqref{eq:cor:boundEps:eq5} for \eqref{eq:cor:boundEps:eq6} and thus \begin{align}
    \log \l T_2 \r &\lesssim \log(C) + \log\l \frac{1}{n_\alpha} \r  + \frac{\delta_1}{d_W} \log(n_\alpha)^{1/2} \log \log \log n_\alpha + \log \log n_\alpha \notag. 
\end{align}
On the other hand, with the choice \eqref{eq:eps-choice},
\[
\log(T_1)
= \log(4) + \frac{-2}{d_W} \ls \log\l \frac{d_W}{2\delta_2} \r + \log\log \log n_\alpha - \log \log \log \log n_\alpha \rs.
\]
We note that $\lim_{n_\alpha \to \infty} \log(T_2) = -\infty$ due to the $\log(n_\alpha^{-1})$ term; we also have $\lim_{n_\alpha \to \infty} \log(T_1) = -\infty$ due to the term $- \log \log \log n_\alpha$. This implies that $T_2$ goes to 0 much faster than $T_1$, so $T_1$ dominates in \eqref{eq:cor:boundEps:eq7} and, using \eqref{eq:eps-choice}, we conclude \[
T \lesssim 4\eps^2 \lesssim  \left(
        \frac{d_W}{2\delta_2}
        \frac{\log\log n_\alpha}{\log\log\log n_\alpha}
    \right)^{-\frac{2}{d_W}}.
\]
\end{proof}

\begin{remark}[Exact error balancing]
In principle, one could choose $\varepsilon=\varepsilon(n_\alpha)$ by solving the implicit balancing relation
$T_1(\varepsilon)\asymp T_2(\varepsilon)$.
However, since $T_2(\varepsilon)$ contains the nested exponential factor
$\varepsilon^{-\delta_1 \varepsilon^{-\delta_2\varepsilon^{-d_W}}}$,
this equation does not admit a tractable closed-form solution.
We therefore select an explicit $\varepsilon(n_\alpha)$ for which $T_2$ is asymptotically negligible (and hence $T_1$ dominant),
yielding an explicit rate.
\end{remark}

\section{Conclusion} \label{sec:discussion}

In this work, we provided theoretical insights into statistical generalization for multiple operator learning, where the goal is to learn an operator family
\(\{G[\alpha]:U\to V\}_{\alpha\in W}\) from data collected hierarchically across sampled operator instances, input functions, and evaluation points.
To motivate and contextualize this setting, we presented several representative examples in which the multi-operator viewpoint is intrinsic, including parameterized integral operators, parameterized PDE solution operators, and task-conditioned operator families. For separable neural architectures, in particular the MNO model, we established an explicit generalization bound that makes the dependence on the sampling budgets
\((n_\alpha,n_u,n_x)\) and on the statistical complexity of the induced hypothesis class transparent.
The analysis combines two main ingredients: (i) an \(\varepsilon\)-accurate approximation guarantee for MNO (with an explicit \(\varepsilon\)-dependent architecture) from \cite{weihs2025MOL}, and
(ii) a covering-number estimate for the corresponding clipped separable network class, derived via a parameter-quantization argument.
Balancing the approximation scale \(\varepsilon\) with the covering scale \(\eta\) then yields an explicit learning-rate statement in the operator-sampling budget \(n_\alpha\),
highlighting how increased operator variability improves transfer to previously unseen operator instances.

Several extensions of the present theory are natural. First, our analysis treats the trained network as an exact empirical-risk minimizer over the prescribed hypothesis class. In practice, training is approximate and influenced by optimization dynamics, regularization, and early stopping; incorporating an explicit optimization/suboptimality term, or deriving algorithm-dependent generalization guarantees, would sharpen the connection between the theory and real training pipelines. Second, the rates obtained here rely on global \(\Lp{\infty}\) covering-number control for clipped hypothesis classes, which is a large class. Improved bounds may be possible by using localized or data-dependent complexity measures, and by studying alternative architectures or conditioning regimes that reduce the effective statistical complexity. In particular, it would be interesting to extend the analysis beyond separable MNO-type classes to attention-based operator architectures, in which the operator descriptor $\alpha$ and the input observations are encoded and combined in a shared latent representation. Finally, under our current global complexity control, the bound is asymptotically bottlenecked by the $n_\alpha$-dependent term, meaning that operator diversity cannot be compensated for by oversampling $(u,x)$ once $n_\alpha$ is fixed. This motivates investigating adaptive or active strategies for selecting operator instances $\alpha$ to accelerate transfer to previously unseen operators under a constrained operator-sampling budget.

\section*{Acknowledgments}
A. Weihs and H. Schaeffer were supported in part by NSF 2427558.
\vspace{0.5em}

\bibliographystyle{plain}
\bibliography{references}{}

\appendix

\end{document}